\PassOptionsToPackage{table}{xcolor}
\documentclass{article}
\usepackage{iclr2027_conference,times}
\usepackage{amsmath,amssymb,amsthm,booktabs,graphicx,tabularx,array,hyperref,url,placeins,longtable,multirow,wrapfig,needspace,float}
\usepackage{xcolor}
\usepackage{portal_layout}
\hypersetup{colorlinks=true,linkcolor=black,citecolor=black,urlcolor=black}
\newcommand{\PORTAL}{\textsc{Portal}}
\newcommand{\pp}{\,\mathrm{pp}}
\newtheorem{theorem}{Theorem}
\newtheorem{proposition}{Proposition}
\newtheorem{lemma}{Lemma}

\title{When Should We Intervene at Test Time?}
\iclrfinalcopy
\author{\begin{tabular}{@{}c@{\hspace{0.4em}}c@{}}
\parbox[t][4em][t]{0.47\textwidth}{\centering\bfseries Jiahao Zhang$^{*}$\\\normalfont\footnotesize Zhengzhou University of Aeronautics\\\texttt{marchematics@gmail.com}} &
\parbox[t][4em][t]{0.47\textwidth}{\centering\bfseries Xinling Wen\\\normalfont\footnotesize Zhengzhou University of Aeronautics\\\texttt{wenxinling@zua.edu.cn}}\\[0.8em]
\parbox[t][4em][t]{0.47\textwidth}{\centering\bfseries Liangliang Li\\\normalfont\footnotesize Xi'an University of Technology\\\texttt{1240311009@stu.xaut.edu.cn}} &
\parbox[t][4em][t]{0.47\textwidth}{\centering\bfseries Yu Chen$^{*}$\\\normalfont\footnotesize Zhengzhou University of Aeronautics\\\texttt{chenyu@zua.edu.cn}}
\end{tabular}\\[0.6em]
\normalfont\footnotesize $^{*}$Corresponding authors}
\begin{document}
\maketitle
\begin{abstract}
Test-time adaptation helps a deployed model respond to changing inputs, but the same intervention can rescue one batch and damage the next. Updating normalization statistics, for example, can improve corrupted images yet fail when a batch contains mostly one class. Average accuracy hides these failures and cannot tell a model when to use an intervention. Target labels can guide this decision, but they are scarce and some choices are already clear. We introduce PORTAL, a controller that chooses when to deploy an intervention, which inputs to label, and when to stop querying. Each available prediction rule is a candidate, and its predictions are fixed before current feedback arrives. Every acquired label therefore evaluates all candidates on the same input. PORTAL combines these shared comparisons and directs queries toward candidates that can still overtake the current leader. It stops when the remaining labels cannot change the choice or when the possible accuracy gain no longer justifies another query. Our analysis links the required feedback to accuracy gaps and shared mistakes, explaining why some choices need much less evidence. Experiments show that this control turns an intervention that loses accuracy overall into a useful choice on selected batches. It prevents severe failures caused by changes in arrival order and preserves gains across complete corruption streams. When a fixed rule already works well, PORTAL uses little feedback to retain its performance. Periodic label checks reveal when past evidence needs refreshing and help the controller recover as conditions change. The acquired labels also train stronger candidates for future batches, and selection brings nearly all of their accuracy gains into deployment. The same feedback therefore improves both the choice made now and the predictions available to later batches.
\end{abstract}
% BEGIN inlined sections/expanded_introduction.tex
\section{Introduction}
A deployed model encounters changes in both image appearance and class composition. Fog may replace clear weather, or one object class may dominate a run of inputs. Test-time adaptation responds by changing the model as data arrive. Yet an update that rescues one batch can break predictions on the next. Its average gain across a test set does not tell us when to use it. Consider batch normalization (BN), which normalizes activations using feature statistics. Recomputing these statistics on incoming images often helps under corruption. Averaging predictions from an image and its horizontally flipped view offers another intervention. We call each available prediction rule a \emph{candidate}. Existing adaptation methods provide many such candidates \citep{wang2021tent,niu2022eata,niu2023sar}. A fixed rule chooses one for the whole stream. Local selection chooses for each batch, retaining the unchanged model when an intervention is unhelpful.

The need for selection becomes clear on PACS, which contains Photo, Art, Cartoon and Sketch domains. Combining BN updates with original and flipped views gives BN+flip. This candidate works well when domains arrive in blocks or are mixed together. Grouping inputs by class, called \emph{class-block arrivals}, changes the batch statistics and reduces its accuracy to 22.58\%. We introduce \PORTAL{}, a controller that uses sparse target labels to choose among the available predictions.\footnote{Code, data, and manuscript: \url{https://anonymous.4open.science/r/PORTAL-8322/}.} It reaches 51.28\% under the same arrivals. Choosing well requires evidence, but labels are scarce. Active test-time adaptation acquires labels to guide updates and selection \citep{gui2024simatta,li2024hiltta,wang2025eatta}; adaptive labeling also supports evaluation under shift \citep{DBLP:conf/nips/MittalMJN24}. A close comparison may cost more feedback than its accuracy gain is worth. The controller must decide which candidate to deploy, which input to label, and when to stop. Appendix~\ref{app:related_work} places these decisions in the broader literature.

\PORTAL{} uses one label to evaluate every candidate on the same input. Agreement gives candidates the same outcome; disagreement reveals which predictions are correct. We group inputs by the classes that candidates predict and estimate class probabilities within each group. Every label updates all comparisons, including the overlap between errors. Queries target rivals that can overtake the leader within the remaining budget. Selection stops when no remaining labels can change the winner or when the possible accuracy gain cannot justify another query. Our analysis relates the feedback requirement to accuracy gaps and shared errors. It also shows why biased queries can make a wrong choice more certain.

Earlier labels reduce query costs while their class relationships remain stable; periodic checks reveal when comparisons need refreshing. After choosing among fixed predictions, we use the new labels to train stronger candidates for later batches. Our experiments follow this progression from intervention failure to future learning. Holding a candidate fixed shows that targeted feedback turns its aggregate accuracy loss into a gain. A more harmful transfer shows how selection prevents most of an intervention's damage. Complete ImageNet-C corruption streams test control at scale. CIFAR comparisons show the label savings when a strong fixed rule already works well. Finally, learned candidates raise PACS deployment accuracy from 59.68\% to 72.64\%. The same labels improve current deployment decisions and train stronger candidates for later batches.
\space
% END inlined sections/expanded_introduction.tex
% BEGIN inlined sections/expanded_method.tex
\section{Shared feedback for intervention selection}
\label{sec:expanded_method}
\subsection{The deployment decision}
At batch $t$, the candidate set $\mathcal A_t$ includes the unchanged rule $I$, called the \emph{incumbent}. A candidate compared against it is a \emph{challenger}. Other candidates update normalization, average image views, or use another classifier. A classifier produces class scores, or \emph{logits}. Averaging these scores gives an ensemble. All candidate predictions are fixed before querying.

Split the batch into a query set $Q_t$ and a separate deployment set $E_t$. Each input keeps its assigned role throughout the stream. The selector labels inputs from $Q_t$ and chooses one candidate for all of $E_t$. Deployment labels are opened only for evaluation. For input $x$ with label $y$, let $R_a(x,y)=\mathbf1\{a(x)=y\}$ record whether candidate $a$ is correct. We seek the highest expected accuracy on $E_t$ given the available evidence. For comparison, a \emph{batch oracle} uses every deployment label to choose the most accurate candidate.

\subsection{One label updates every comparison}
Merge candidates that agree on every current input. Each input's \emph{prediction pattern} lists the remaining candidates' predicted classes. Inputs with the same pattern form a \emph{stratum}. In stratum $s$, candidate $a$ predicts $p_{sa}$; the stratum's deployment share is $w_s=|E_t\cap s|/|E_t|$.

For $C$ classes, let $\theta_s=(\theta_{s1},\ldots,\theta_{sC})$ give the class probabilities in stratum $s$, and collect them as $\theta=(\theta_s)_s$. We assume query and deployment inputs within a stratum share these probabilities. This is the \emph{query-to-deployment transport assumption}. Candidate $a$ then has expected accuracy $A_a(\theta)=\sum_s w_s\theta_{s,p_{sa}}$.

A Dirichlet prior describes uncertainty about each stratum's class probabilities. The priors are independent across strata. After observing labels, let $\alpha_{sc}$ be the class counts including the prior. Write $n_s=\sum_c\alpha_{sc}$ for their total and $\pi_{sc}=\alpha_{sc}/n_s$ for the estimated class probabilities. Given the current evidence $\mathcal F_t$, candidate accuracy means and covariances are
\begin{align}
 \mu_a&=\sum_s w_s\pi_{s,p_{sa}},\label{eq:expanded_mu}\\
 \Sigma_{ab}&=\sum_s\frac{w_s^2}{n_s+1}
 \left[\mathbf1\{p_{sa}=p_{sb}\}\pi_{s,p_{sa}}
       -\pi_{s,p_{sa}}\pi_{s,p_{sb}}\right].\label{eq:expanded_cov}
\end{align}
The covariance records which candidates succeed or fail together. The weights $w_s$ account for the current mix of prediction patterns. A label therefore updates all comparisons without treating shared errors as independent. Reusing past counts also requires stable class probabilities within each stratum. Section~\ref{sec:transport_stress} tests what happens when this stability or representative sampling breaks.

\subsection{Query only while the decision can improve}
Let $j$ be the candidate with the highest estimated accuracy $\mu_a$. A designated fixed candidate breaks ties. A rival is \emph{reachable} if some sequence of affordable labels can put it ahead of $j$. To test this, give every queried disagreement the rival's predicted class. After $k_s$ such labels in stratum $s$, its best possible advantage is
\begin{equation}
 F_a(\boldsymbol k)=\sum_{s:p_{sa}\ne p_{sj}}
 w_s\frac{\alpha_{s,p_{sa}}-\alpha_{s,p_{sj}}+k_s}{n_s+k_s}.
 \label{eq:expanded_reachable}
\end{equation}
Maximize $F_a$ over available query counts and remaining budget. Decreasing marginal increments make greedy integer allocation exact. If every rival is unreachable, further queries have zero expected accuracy value for this batch.

For a reachable rival, let $\mu_a^{(s,c)}$ be the mean after observing class $c$ in stratum $s$. A query's expected decision value is
\begin{equation}
 V_s=\sum_c\pi_{sc}\max_a\mu_a^{(s,c)}-\max_a\mu_a.
 \label{eq:expanded_voi}
\end{equation}
Query the eligible stratum with the greatest $V_s$. A useful switch may need several labels before a rival overtakes the leader. If every $V_s$ is zero, use $T_s$, the largest variance reduction across reachable leader--rival comparisons (Appendix~\ref{app:expanded_algorithm}). In Figure~\ref{fig:expanded_method}, $s^*$ is the selected stratum and $b$ is the remaining label budget. F and BF denote horizontal flip and BN+flip.

\paragraph{Bounding the value of further information.}
The \emph{expected value of perfect information} (EVPI) is the gain from learning all candidate accuracies exactly. Bounding this gain bounds the benefit of further querying. Define the leader's gap $d_a=\mu_j-\mu_a$ and paired posterior variance $s_a^2=\Sigma_{aa}+\Sigma_{jj}-2\Sigma_{aj}$.

\begin{theorem}[Posterior value and local stopping]
\label{thm:expanded_stop}
Assume independent positive Dirichlet stratum priors, the categorical label likelihood, and fixed target weights. For a posterior-mean maximizer $j$, the expected accuracy improvement from any subsequent label policy is at most
\begin{equation}
 U=\min\left\{1-\mu_j,
 \sum_{a\ne j}\frac{\sqrt{d_a^2+s_a^2}-d_a}{2}\right\}.
 \label{eq:expanded_evpi}
\end{equation}
Let $c$ be the cost of another label in expected batch-accuracy units. If $c\ge U$, continuing cannot improve expected net utility. For a fixed remaining budget and available stratum counts, suppose Equation~\eqref{eq:expanded_reachable} is nonpositive for every feasible allocation and every rival. Those remaining labels then have zero expected accuracy value.
\end{theorem}
The bound adds the possible gains from rivals while accounting for their shared errors. Its proof appears in Appendix~\ref{app:posterior_joint}. It governs the current batch under the transport assumption. Reusing past counts also requires temporal stability.

\begin{figure}[!htb]
\centering
\includegraphics[width=\linewidth]{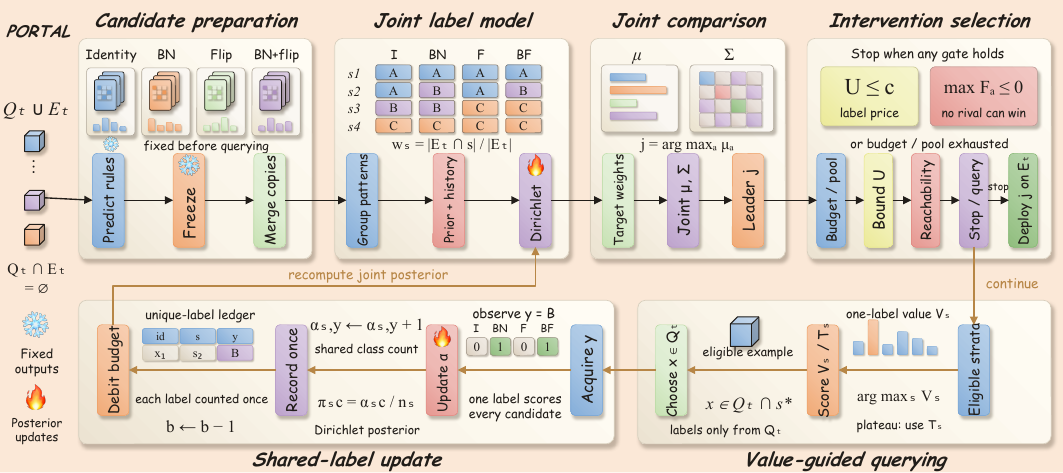}
\caption{\PORTAL{} compares fixed candidates through shared labels. It deploys the leader when no rival can overtake it or when the value bound falls below the label price. Otherwise, it queries to improve the decision. Snowflakes mark fixed predictions; flames mark updates to the class-probability estimates.}
\label{fig:expanded_method}
\end{figure}

\paragraph{Learning after selection.}
A label can help beyond the current decision. After comparing the fixed predictions, we use it to train the next batch's classifier. Uniform queries fill any unused training allowance, and each label is counted once. Paired experiments retain or undo an update to measure its effect on later predictions (Appendix~\ref{app:expanded_state}).
\space
% END inlined sections/expanded_method.tex
% BEGIN inlined sections/expanded_theory.tex
\section{The evidence needed to choose an intervention}
\label{sec:expanded_theory}
A larger accuracy gap makes a choice easier. Shared errors also help: labels need only distinguish the candidates where their predictions differ. We first quantify these effects for independent queries within strata, then examine a fixed query pool.

\subsection{Accuracy gaps and paired differences}
Fix the prediction map, class probabilities and positive weights $w_s$. Let $*$ be the unique best candidate after merging duplicates. For rival $a$, define $\Delta_a=A_*-A_a$ and $v_{s,a}=\operatorname{Var}(R_*-R_a\mid s)$. An allocation $q$ assigns query fractions $q_s\ge0$ with $\sum_s q_s=1$. The paired variance is
\begin{equation}
 V_a(q)=\sum_s\frac{w_s^2v_{s,a}}{q_s},\qquad
 v_{s,a}=\mathbf1\{p_{s*}\ne p_{sa}\}
 [\theta_{s,p_{s*}}+\theta_{s,p_{sa}}
 -(\theta_{s,p_{s*}}-\theta_{s,p_{sa}})^2].
 \label{eq:expanded_variance}
\end{equation}
Only disagreements contribute to this variance. Adding a duplicate candidate creates no new comparison. A close rival, by contrast, can require many labels to distinguish.

A \emph{uniformly $\delta$-correct} procedure selects the best candidate with error probability at most $\delta\in(0,1/2)$ for every allowed distribution. Let $\tau$ count its labels. Categorical Kullback--Leibler divergence $\operatorname{KL}$ and binary divergence $\operatorname{kl}$ measure how well observations distinguish distributions.

\begin{theorem}[Shared-label information and local allocation]

\label{thm:expanded_information}
Let $\operatorname{Alt}(\theta)$ be the set of alternative class distributions $\lambda$ whose best candidate differs from that under $\theta$. Every uniformly $\delta$-correct strategy that stops almost surely at each allowed distribution with a unique best candidate satisfies
\begin{equation}
 \mathbb E_\theta\tau\ge
 \frac{\operatorname{kl}(1-\delta,\delta)}{
 \sup_q\inf_{\lambda\in\operatorname{Alt}(\theta)}
 \sum_s q_s\operatorname{KL}(\theta_s\Vert\lambda_s)}.
 \label{eq:expanded_lower}
\end{equation}
For small gaps, with class probabilities and query fractions bounded away from zero, the local optimal allocation minimizes $\max_a V_a(q)/\Delta_a^2$. A pilot-then-confirm procedure approaches the leading label coefficient as error probability tends to zero before the gap. Appendix~\ref{app:local_allocation} gives the expansion, conditions and construction.
\end{theorem}
The proof adapts best-action identification \citep{garivier2016optimal} to shared class labels. The ratio $V_a(q)/\Delta_a^2$ gives a direct reading of the result. Small gaps increase the label requirement; shared errors reduce the variance that labels must resolve.

\subsection{Finite budgets and biased queries}
For a fixed pair, each query favors one candidate or ties their errors. Uniform sampling without replacement gives exact selection probabilities, including ties, and hence expected deployment accuracy. Appendix~\ref{app:finite_budget} derives the sampling law; Section~\ref{sec:real_pair_budget} tests it on real predictions.

Let $\Delta>0$ be the better candidate's deployment advantage. After $m$ queried labels, let $\widehat d_m\sim\mathcal N(\Delta+b,\sigma^2/m)$ estimate this advantage. Here $b$ is query bias and $\sigma^2/m$ is the variance, with $\sigma>0$. Write $\Phi$ for the standard normal cumulative distribution. Selecting by its sign gives expected regret, or extra deployment error,
\begin{equation}
 \mathcal R_m=\Delta\Phi\!\left(-\frac{\sqrt m(\Delta+b)}{\sigma}\right),
 \label{eq:expanded_gaussian}
\end{equation}
More labels reduce regret when $\Delta+b>0$. When bias reverses that sign, more labels increase regret: the evidence becomes more certain about the wrong choice. Appendix~\ref{app:finite_budget} extends this calculation to multiple candidates.

With four candidates, correlated errors strengthen the query ranking by canceling shared noise (Appendix~\ref{app:finite_budget}, Figure~\ref{fig:expanded_gaussian}). This helps representative queries and can reinforce biased choices.

\space
% END inlined sections/expanded_theory.tex
% BEGIN inlined sections/expanded_experiments.tex
\begingroup
% Keep evidence close to its discussion without changing the template text area.
\setlength{\intextsep}{6pt plus 1pt minus 1pt}
\setlength{\textfloatsep}{10pt plus 2pt minus 2pt}
\setlength{\floatsep}{8pt plus 1pt minus 1pt}
\captionsetup{skip=4pt}
\section{Experiments}
\label{sec:expanded_experiments}
We first expose intervention failures, then hold the candidates fixed to measure the effect of control. Full corruption streams test the scale of the effect. We next examine stopping and changing evidence, before using the same feedback to train stronger candidates. Accuracy differences are in percentage points (pp).

The main comparisons use ResNet-18 on CIFAR-10-C/100-C and a Photo-trained ResNet-50 on PACS. ImageNet-C uses a ResNet-50 updated by EATTA, an active test-time adaptation method \citep{wang2025eatta}. Natural-domain tests use contrastive language--image pretraining (CLIP), with a large vision transformer and 14-pixel patches (ViT-L/14). Its \emph{readout} maps features to class scores. CIFAR-10.1 supplies newly collected images \citep{recht2018cifar101}.

Fixed rules use the unchanged model, BN, horizontal flip, or BN+flip. We also compare their mean-logit ensemble. \emph{Empirical risk minimization} (ERM) selects the candidate with the fewest queried errors, using a fixed tie rule. Its queries are uniform, prioritize class disagreements, or target symmetric KL divergence between prediction probabilities (Top-KL). Shared-feedback comparisons keep query and evaluation inputs separate. Historical paired-admission tests use queried-error comparisons. Their protocols are described below; the shared-posterior experiments use Section~\ref{sec:expanded_method}.

\subsection{Arrival order can break a strong intervention}
\label{sec:intervention_failure}
On PACS, BN+flip works well when domains arrive in blocks or are mixed together. Grouping images by class changes its behavior sharply. Accuracy falls to 22.58\%, below the unchanged model's 50.88\%. \PORTAL{} chooses among the same predictions and reaches 51.28\%. Local selection prevents the fixed intervention from turning a change in arrival order into a large accuracy loss.

Figure~\ref{fig:expanded_capacity} also measures \emph{net batch harm}: positive batch error increases over the unchanged model, summed and divided by evaluation size. \emph{New errors} count previously correct predictions made wrong. These measures distinguish the damage an intervention causes from its net accuracy.

\begin{figure}[!htb]
\centering
\begin{tikzpicture}
\node[anchor=south west,inner sep=0] (arrivalplot) at (0,0)
  {\includegraphics[width=\linewidth]{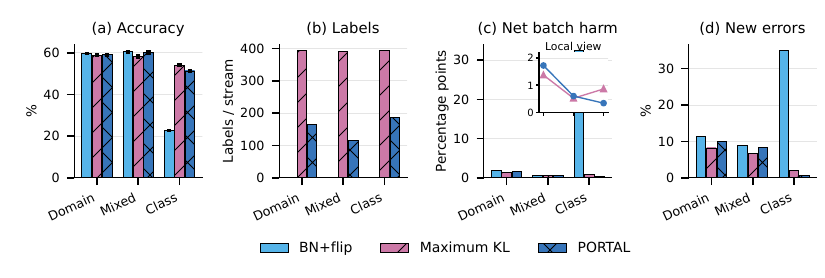}};
\begin{scope}[x={(arrivalplot.south east)},y={(arrivalplot.north west)}]
\node[font=\fontsize{6.4}{7}\selectfont\bfseries,fill=white,inner sep=.7pt] at (.205,.562) {22.58};
\node[font=\fontsize{6.4}{7}\selectfont\bfseries,fill=white,inner sep=.7pt] at (.231,.792) {51.28};
\end{scope}
\end{tikzpicture}
\caption{A strong fixed intervention fails when classes arrive together. On class-block PACS, BN+flip reaches 22.58\% accuracy and \PORTAL{} reaches 51.28\%. Net batch harm falls from 32.41 to 0.34 percentage points. Domain, Mixed and Class denote domain blocks, mixed domains and class blocks. Results use a Photo source and batch size 64. Bars average three orders; accuracy error bars show order standard deviation (SD). The inset enlarges the selectors' harm values.}
\label{fig:expanded_capacity}
\end{figure}

\paragraph{Adding candidates can help or hurt.}
A separate controlled experiment expands the candidate set, or \emph{portfolio}. Expansion raises error by 14.32 points under a HILTTA-style controller \citep{li2024hiltta}. Under \PORTAL{}, it lowers error by 1.24 points. This yields a 15.56-point controller--portfolio interaction across ten seeds (Appendix~\ref{app:controller_stress}). We next fix the predictions to measure the effect of control directly.

\subsection{Control makes a harmful candidate useful}
\label{sec:paired_admission}
We use \emph{paired admission}: deploy a challenger only when it makes fewer queried errors than the incumbent. Ties keep the incumbent. The historical PACS replay holds the incumbent, horizontal-view challenger and 816 batches fixed. Its \PORTAL{} rule targets three symmetric-KL queries per batch; admission then equals ERM on those labels.

\begin{wrapfigure}{R}{0.45\textwidth}
\vspace{-8pt}
\centering
\setlength{\parskip}{0pt}
\captionsetup{font=small,skip=2pt}
\includegraphics[width=\linewidth]{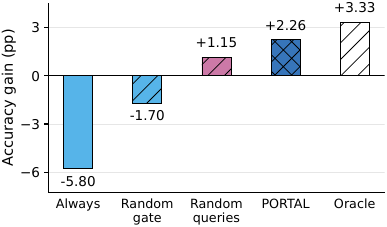}
\caption{Accuracy gains over the incumbent on the recorded PACS trajectory, including queried inputs. \PORTAL{} uses symmetric-KL queries and paired ERM. Random gate matches its admission rate; random queries use three labels. Oracle uses full-batch labels. Appendix~\ref{app:paired_admission} gives the full protocol.}
\label{fig:controller_attribution}
\par\vspace{3pt}
\centering
\captionsetup{font=small,skip=2pt}
\captionof{table}{Gain over the incumbent (pp). Selection uses paired admission on PACS/VLCS and sequential admission on ImageNet-C.}
\label{tab:intervention_control}
\fontsize{8.5}{10}\selectfont
\setlength{\tabcolsep}{4pt}
\begin{tabular*}{\linewidth}{@{\extracolsep{\fill}}lrr@{}}
\toprule
Stream & Always & Selected \\
\midrule
VLCS & -11.258 & -0.353 \\
PACS & -5.797 & +2.258 \\
ImageNet-C & +1.409 & +0.590 \\
\bottomrule
\end{tabular*}

\end{wrapfigure}

Always deploying the challenger loses 5.80 points. Targeted feedback turns this into a 2.26-point gain (Figure~\ref{fig:controller_attribution}). Random admission still loses accuracy, and uniform queries recover less of the gain. Excluding current queried inputs preserves the ordering: targeted queries gain 1.79 points and uniform queries gain 0.94. The useful intervention is unchanged; the evidence determines when to use it.

VLCS tests a more harmful transfer. Always deploying its challenger loses 11.26 points; paired admission reduces the loss to 0.35. This avoids 96.9\% of the fixed intervention's net accuracy loss (Table~\ref{tab:intervention_control}).

\paragraph{Control on a full corruption stream.}
In a shared-prefix variant, EATTA reuses the sequence of selection labels for its committed update. On 750,000 ImageNet-C images, error falls from 50.14\% to 49.55\%, improving all 15 corruptions. Of admitted batches, 77.3\% benefit. Accepted choices stop after a mean prefix of 2.87 labels. Across all batches, the mean prefix is 3.999 labels. Trial updates access additional labels outside this count (Appendix~\ref{app:empirical-imagenet}).

With separate acquisition, three selection labels in addition to three adaptation labels reduce error by 1.008 points. This retains 73.37\% of the fixed challenger's gain. An independent full stream gives 0.982 points. Appendix~\ref{app:empirical-imagenet} reports both protocols and six benchmarks' fixed-intervention and readout comparisons.

\par\WFclear

\subsection{Strong fixed rules need less feedback}
A strong fixed rule leaves less for selection to gain. On CIFAR corruptions, BN supplies most of the improvement and flipped views add to it. Fixed BN+flip reaches 79.65\% on CIFAR-10-C and 54.32\% on CIFAR-100-C. \PORTAL{} reaches 79.26\% and 53.87\%. Compared with Top-KL selection, it keeps similar accuracy while using 91.3\% and 94.6\% fewer labels. Table~\ref{tab:complete-cifar} gives every corruption result.

\begin{table}[!htb]
\centering
\caption{Severity-five CIFAR accuracy (\%) on shared held-out images, averaged over three source models. Summary: mean $\pm$ standard deviation (SD) and labels per suite. Bold marks the best deployable rule; the oracle is an evaluation reference.}
\label{tab:complete-cifar}
% BEGIN inlined sar_relayout_v3/PORTAL_SAR_style_v3/tables/corruption_matrix.tex
\begingroup
\fontsize{6.7}{7.55}\selectfont
\setlength{\tabcolsep}{3.05pt}
\renewcommand{\arraystretch}{1.06}
\arrayrulecolor{black!40}
\setlength{\arrayrulewidth}{.25pt}
\begin{tabular}{@{}l|rrr|rrrr|rrrr|rrrr|>{\columncolor{blue!8}}r|r@{}}
\toprule
&\multicolumn{3}{c|}{Noise}&\multicolumn{4}{c|}{Blur}&\multicolumn{4}{c|}{Weather}&\multicolumn{4}{c|}{Digital}&\multicolumn{2}{c}{Summary}\\
Model / rule & Gau. & Shot & Imp. & Def. & Gla. & Mot. & Zoom & Snow & Frost & Fog & Bri. & Con. & Ela. & Pix. & JPEG & Avg. & Labels\\
\midrule
\multicolumn{18}{@{}l}{\bfseries CIFAR-10-C / ResNet-18}\\[-1pt]
Source & 25.8 & 32.8 & 25.0 & 56.6 & 51.5 & 67.6 & 65.0 & 77.3 & 66.4 & 73.5 & 91.2 & 33.9 & 76.8 & 45.6 & 73.1 & 57.48\sd{1.12} & 0\\
\hspace{.45em}BN & 66.2 & 68.5 & 59.5 & 87.2 & 65.0 & 86.1 & 87.3 & 81.2 & 80.1 & 84.7 & 90.6 & 86.9 & 76.6 & 79.0 & 71.4 & 78.02\sd{0.50} & 0\\
\hspace{.45em}Flip & 26.3 & 33.2 & 25.2 & 57.1 & 52.5 & 68.3 & 65.9 & 78.2 & 67.2 & 74.0 & 92.1 & 33.9 & 78.1 & 45.7 & 74.4 & 58.14\sd{1.00} & 0\\
\hspace{.45em}BN+flip & \textbf{68.7} & \textbf{71.2} & \textbf{62.7} & \textbf{88.6} & \textbf{67.2} & \textbf{86.9} & \textbf{88.4} & 83.0 & \textbf{81.6} & \textbf{85.8} & 91.1 & \textbf{87.5} & 78.4 & \textbf{80.1} & 73.4 & \textbf{79.65}\sd{0.56} & 0\\
\hspace{.45em}Mean logits & 58.6 & 64.0 & 54.5 & 83.9 & 64.2 & 84.4 & 85.3 & \textbf{83.4} & 80.4 & 84.9 & \textbf{93.2} & 83.7 & \textbf{80.2} & 72.0 & \textbf{75.7} & 76.55\sd{1.29} & 0\\
\cmidrule{1-18}
\hspace{.45em}Uniform ERM & 66.1 & 68.9 & 60.2 & 87.1 & 64.7 & 85.7 & 87.3 & 81.8 & 79.7 & 84.6 & 91.3 & 86.7 & 78.2 & 78.6 & 73.3 & 78.28\sd{0.58} & 7,065\\
\hspace{.45em}Disagree. ERM & 66.6 & 70.0 & 60.5 & 87.7 & 64.5 & 85.9 & 87.5 & 80.9 & 79.6 & 84.0 & 91.4 & 87.1 & 77.8 & 78.8 & 73.1 & 78.36\sd{0.59} & 7,065\\
\hspace{.45em}Top-KL ERM & 68.7 & 71.0 & 62.3 & 88.5 & 65.8 & 86.8 & 88.3 & 81.2 & 80.8 & 84.9 & 91.3 & \textbf{87.5} & 77.7 & 80.1 & 73.0 & 79.20\sd{0.50} & 7,065\\
\rowcolor{pink!24}\hspace{.45em}\PORTAL{} & 68.5 & 70.9 & 62.4 & 88.3 & 66.4 & 86.6 & 87.9 & 81.8 & 81.1 & 84.9 & 91.1 & 87.4 & 78.1 & 80.1 & 73.5 & 79.26\sd{0.51} & 612.3\\
\cmidrule{1-18}
\textit{Batch oracle} & 69.7 & 72.1 & 63.7 & 89.2 & 68.4 & 87.9 & 89.3 & 84.7 & 82.7 & 86.8 & 94.1 & 88.2 & 81.6 & 81.2 & 77.4 & 81.14\sd{0.45} & --\\
\midrule
\multicolumn{18}{@{}l}{\bfseries CIFAR-100-C / ResNet-18}\\[-1pt]
Source & 11.0 & 12.4 & 7.7 & 34.2 & 19.8 & 43.2 & 41.3 & 46.5 & 34.2 & 40.1 & 67.2 & 18.0 & 50.1 & 24.3 & 44.7 & 32.98\sd{0.41} & 0\\
\hspace{.45em}BN & 39.2 & 40.5 & 34.2 & 65.2 & 41.7 & 62.5 & 64.4 & 53.3 & 53.4 & 57.8 & 67.2 & 64.5 & 53.5 & 56.1 & 45.1 & 53.25\sd{0.11} & 0\\
\hspace{.45em}Flip & 10.9 & 12.6 & 7.7 & 34.3 & 19.9 & 44.1 & 41.8 & 47.3 & 34.8 & 40.7 & 67.8 & 18.1 & 51.5 & 24.4 & 45.4 & 33.43\sd{0.41} & 0\\
\hspace{.45em}BN+flip & \textbf{40.6} & \textbf{42.1} & \textbf{35.7} & \textbf{65.9} & \textbf{42.7} & \textbf{63.2} & \textbf{65.3} & 54.6 & 54.4 & \textbf{58.9} & 68.2 & \textbf{65.2} & 54.6 & \textbf{57.2} & 46.1 & \textbf{54.32}\sd{0.05} & 0\\
\hspace{.45em}Mean logits & 32.5 & 34.8 & 30.2 & 62.4 & 39.3 & 61.6 & 63.6 & \textbf{56.8} & \textbf{55.0} & 58.2 & \textbf{72.0} & 58.9 & \textbf{56.8} & 53.7 & \textbf{49.5} & 52.36\sd{0.21} & 0\\
\cmidrule{1-18}
\hspace{.45em}Uniform ERM & 39.3 & 39.9 & 34.5 & 64.1 & 40.8 & 61.3 & 62.8 & 53.3 & 51.9 & 56.8 & 67.9 & 64.1 & 53.9 & 55.6 & 45.7 & 52.79\sd{0.08} & 7,065\\
\hspace{.45em}Disagree. ERM & 39.4 & 39.8 & 34.6 & 64.7 & 40.9 & 61.4 & 63.0 & 53.1 & 52.6 & 56.6 & 67.7 & 63.8 & 53.3 & 55.8 & 45.5 & 52.80\sd{0.23} & 7,065\\
\hspace{.45em}Top-KL ERM & 40.1 & 41.5 & 35.2 & 65.9 & 42.0 & 62.8 & 64.8 & 53.6 & 53.8 & 58.1 & 67.7 & 65.2 & 53.3 & 56.8 & 45.4 & 53.76\sd{0.12} & 7,065\\
\rowcolor{pink!24}\hspace{.45em}\PORTAL{} & 40.5 & 41.9 & 35.1 & 65.5 & 42.0 & 62.6 & 64.6 & 54.3 & 53.5 & 58.1 & 68.1 & 65.2 & 54.6 & 55.9 & 46.1 & 53.87\sd{0.23} & 381.7\\
\cmidrule{1-18}
\textit{Batch oracle} & 41.6 & 43.3 & 36.9 & 67.0 & 44.0 & 64.4 & 66.4 & 56.4 & 55.7 & 60.1 & 71.6 & 66.1 & 57.4 & 58.3 & 50.0 & 55.95\sd{0.15} & --\\
\bottomrule
\end{tabular}
\endgroup
\space
% END inlined sar_relayout_v3/PORTAL_SAR_style_v3/tables/corruption_matrix.tex
\end{table}

CLIP gives a clearer case for stopping early. On PACS and OfficeHome, \PORTAL{} matches the eight-view output with 1.4 and 7.4 labels per stream. The four readouts use an original image, prompt averaging, two views or eight views. For VLCS, we add StatA, a text-initialized Gaussian classifier fitted to unlabeled features \citep{portal2025_realistic_vlm}.

Adding StatA raises selected VLCS accuracy from 83.92\% to 84.58\%. The fixed row shows the gain available from this candidate (Table~\ref{tab:expanded_vlm}). Appendix~\ref{app:expanded-vlm} gives its settings.

An ensemble can also be a candidate. On PACS, selection reaches 62.07\% against the mean-logit ensemble's 61.63\%, improving four of five orders. Across ten arrival conditions, cumulative-history selection beats the best fixed rule in four. Discounting old counts beats every fixed rule under class blocks (Table~\ref{tab:round_dynamic}). CIFAR-10.1 needs only 1.9 queries per stream: strong fixed outputs leave little gain to resolve (Table~\ref{tab:round_overview}).
\WFclear

\subsection{Small label budgets can resolve real advantages}
\label{sec:real_pair_budget}
\begin{wrapfigure}{r}{0.52\textwidth}
\vspace{-9pt}
\centering
\captionsetup{font=small,skip=2pt}
\includegraphics[width=\linewidth]{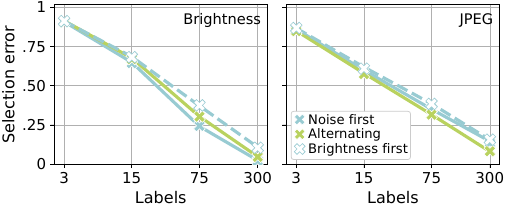}
\caption{Exact selection-error probabilities for brightness (left) and JPEG (right). Curves use uniform sampling without replacement; markers use 2,000 independent subset draws per setting. Means cover three checkpoints. Schedules use distinct image blocks.}
\label{fig:round_paired}
\vspace{-5pt}
\end{wrapfigure}

We compare BN+flip and mean logits on streams with disjoint image sets per corruption segment. BN+flip wins on noise; mean logits wins on brightness and JPEG. We sample uniformly from each segment's full query pool. Selection favors fewer queried errors, with BN+flip breaking ties (Appendix~\ref{app:round_paired}).

Even 75 labels leave a 24--39\% probability of choosing the worse output (Figure~\ref{fig:round_paired}). Exact probabilities agree with independent subset draws. Knowing that a better candidate exists does not make it easy to identify.

\par\WFclear
\Needspace{4\baselineskip}
\subsection{Stopping before high confidence is reached}
\begin{wrapfigure}{R}{0.43\textwidth}
\vspace{-9pt}
\centering
\captionsetup{font=small,skip=2pt}
\includegraphics[width=\linewidth]{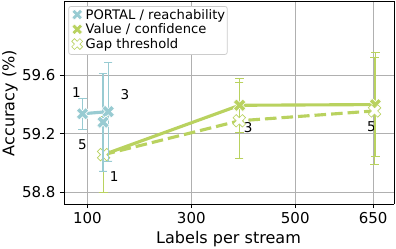}
\caption{Five PACS orders. Error bars show order SD; numbers give credits per batch. Coincident curves share a legend entry.}
\label{fig:round_stopping}
\vspace{-5pt}
\end{wrapfigure}

Uncertainty need not justify another query. A rival can remain plausible even when the remaining budget cannot make it win. On four PACS candidates, we compare \PORTAL{} with stopping based on value, reachability, confidence or the accuracy gap. Each rule carries its own history. Unused label credits accumulate up to four times the per-batch allowance (Appendix~\ref{app:round_stopping}).

At three credits per batch, \PORTAL{} uses 140.0 labels at 59.35\% accuracy. Confidence stopping uses 392.8 at 59.39\%. Reachability drives this saving: the remaining labels cannot change the winner. Removing the value trigger leaves the queries and actions unchanged at the primary label price.

Giving policies the same starting evidence each batch removes history differences. \PORTAL{} and confidence stopping then match accuracy with 85.6 and 392.8 labels. Figure~\ref{fig:round_stopping} instead follows their separate histories.

Table~\ref{tab:core-mechanism} isolates acquisition and stopping. Changing only acquisition, \PORTAL{} uses 85.6 labels versus uniform sampling's 132.8 at similar accuracy. Five queries after a reachability stop yield no mean accuracy gain.

\par\WFclear
\begin{table}[!htb]
\centering\small
\caption{Four-candidate PACS controls. A shares the posterior, history, budget and stopping rule. B adds labels after stopping. Harm averages positive batch accuracy losses relative to the unchanged model over evaluation inputs; pp denotes percentage points.}
\label{tab:core-mechanism}
\begin{tabular*}{\linewidth}{@{\extracolsep{\fill}}lrrr@{}}
\toprule
\multicolumn{4}{l}{\textit{A. Common-posterior acquisition}} \\
Acquisition & Accuracy (\%) & Labels & Harm (pp) \\
\midrule
Uniform & 59.35 $\pm$ 0.13 & 132.8 & 1.78 \\
Disagreement & 59.31 $\pm$ 0.23 & 96.8 & 1.86 \\
Maximum-pair KL & 59.33 $\pm$ 0.26 & 118.8 & 1.85 \\
PORTAL value & 59.32 $\pm$ 0.34 & 85.6 & 1.86 \\
\midrule
\multicolumn{4}{l}{\textit{B. Labels purchased after stopping}} \\
Stop setting & Extra labels & Choice change & $\Delta$ accuracy (pp) \\
\midrule
Primary unreachable & 5.0 & 0.31\% & +0.000 $\pm$ 0.051 \\
EVPI stop, price 0.05 & 1.0 & 1.20\% & -0.075 $\pm$ 0.052 \\
EVPI stop, price 0.05 & 3.0 & 2.21\% & -0.019 $\pm$ 0.100 \\
EVPI stop, price 0.05 & 5.0 & 2.81\% & -0.012 $\pm$ 0.141 \\
\bottomrule
\end{tabular*}
\end{table}

\Needspace{4\baselineskip}
\subsection{Refreshing past evidence as conditions change}
\label{sec:transport_stress}
\begin{wrapfigure}{R}{0.61\textwidth}
\vspace{-9pt}
\centering
\captionsetup{font=small,skip=4pt}
\includegraphics[width=\linewidth]{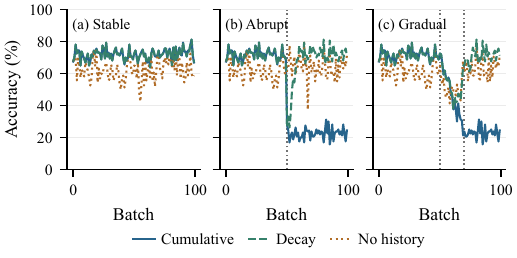}
\caption{A change in the label law with fixed candidate predictions and stratum frequencies. Vertical lines mark the transition. Decay multiplies past counts by 0.9. The original five-seed experiment and its costs appear in Table~\ref{tab:round_drift}.}
\label{fig:round_history}
\vspace{-5pt}
\end{wrapfigure}

Past labels help only while they describe the present. Figure~\ref{fig:round_history} changes class probabilities within prediction patterns. Predictions and pattern frequencies stay fixed, and query and deployment labels share the new probabilities. After an abrupt reversal, cumulative history buys no labels and keeps the old choice. Late accuracy falls from 72.21\% to 22.54\%. Discounting old counts recovers 67.66\%.

The weights $w_s$ track how often a pattern occurs. They cannot detect a change in which labels that pattern contains. Risk monitoring and domain-shift detection address related failures \citep{DBLP:conf/nips/SchirmerJNN25,DBLP:conf/iccv/ParkPKM25}. Here the candidate predictions themselves stay fixed.

We therefore reserve periodic labels to check whether old evidence still holds. These \emph{guarded selectors} buy extra labels after surprising observations. We test a stream with two fixed patterns and another with 300 hidden groups, called \emph{modes}. Each batch in the latter contains four modes. Prioritized queries oversample atypical examples; known query probabilities correct this sampling bias.

The comparator uses exponentially weighted (EW) class counts. Methods share candidates and label ceilings. Recovery begins at the first of five consecutive population-best choices. The target is delay at most ten batches in at least 90\% of streams at each change (Appendix~\ref{app:w1_guarded}).

\begin{table}[!htb]
\centering\small
\caption{Post-change gains over our reconstructed EW selector at equal label ceilings; brackets give paired intervals. Recovery is the fraction with delay $\leq10$. Sparse rows give changes 25/50/75 under prioritized queries. The sparse 160-label variant changes both budget and settings.}
\label{tab:w1_budget}
\setlength{\tabcolsep}{4pt}
\begin{tabular*}{\linewidth}{@{\extracolsep{\fill}}lrrl@{}}
\toprule
Stream & Label ceiling & Post-change gain (pp) & Recovery, delay $\leq 10$ \\
\midrule
Two-pattern abrupt & 120 & $+9.52\ [8.83,10.21]$ & 0.92 \\
Two-pattern gradual & 120 & $+4.87\ [4.16,5.58]$ & 0.80 \\
Two-pattern gradual, larger budget & 160 & $+6.44\ [5.92,6.96]$ & 0.94 \\
300-mode recurrent, priority audit & 100 & $+3.28\ [0.61,5.95]$ & 0.685 / 0.798 / 0.642 \\
300-mode recurrent, larger-budget variant & 160 & $+8.55\ [7.68,9.41]$ & 0.907 / 0.933 / 0.893 \\
\bottomrule
\end{tabular*}
\end{table}

At the original synthetic ceiling, stable-period accuracy stays close to the strongest fixed rule: $-0.122\pp$ with paired interval $[-0.150,-0.094]$. Abrupt recovery is fast; more labels improve gradual recovery (Table~\ref{tab:w1_budget}). Sparse post-change accuracy exceeds EW under both query schedules. With 100 labels, recovery is slow at the first and third recurrent changes. The 160-label variant approaches the target, still missing it at the last change.

A \emph{harmful choice} selects a candidate worse than the prespecified strong fixed rule. Across sparse scenarios and schedules, the mean harmful-batch rate per seed is 3.73\%, with upper confidence bound 3.91\%. Recurrent changes raise the mean to 7.46\%. The pooled bound describes the cross-scenario average; repeated changes remain harder.

\par\WFclear
\Needspace{4\baselineskip}
\subsection{The same labels create stronger future candidates}
\label{sec:learned_candidate}
Feedback can improve the choices available to later batches. A \emph{residual readout} learns a correction to fixed class scores. Earlier labels determine its current predictions. Labels queried for the current choice are \emph{decision labels}. They compare these predictions, then train the next readout. Uniform queries fill the remaining training allowance.

\begin{wraptable}{R}{0.53\textwidth}
\vspace{-8pt}
\centering
\captionsetup{font=small,skip=4pt}
\caption{PACS accuracy (\%). Selected is deployed accuracy, given as mean $\pm$ SD over five stream orders. Learned always uses each policy's learned candidate. Decision labels form part of the 393 total labels per stream. The fixed rule uses zero labels.}
\label{tab:expanded_learning}
\fontsize{8.4}{9.6}\selectfont
\setlength{\tabcolsep}{3pt}
\begin{tabularx}{\linewidth}{@{}Xrrr@{}}
\toprule
Rule & Selected & Learned & Decision \\
\midrule
Fixed BN+flip & \mstd{59.68}{0.28} & --- & 0 \\
Uniform ERM & \mstd{65.91}{0.49} & 73.03 & 393.0 \\
Disagreement ERM & \mstd{70.29}{0.76} & 72.98 & 393.0 \\
\rowcolor{PortalRow}\PORTAL{} & \mstd{72.64}{1.05} & 73.03 & 70.8 \\
\bottomrule
\end{tabularx}
\vspace{-4pt}
\end{wraptable}

Deployment accuracy rises from fixed BN+flip's 59.68\% to 72.64\%. The learned candidate reaches 73.03\%. \PORTAL{} therefore retains all but 0.39 points of its 13.35-point gain (Table~\ref{tab:expanded_learning}).

All learning policies use 393 labels per stream. \PORTAL{} uses 70.8 for current decisions; those labels also enter later training. At the same total budget, disagreement ERM deploys at 70.29\%, while its learned candidate reaches 72.98\%. Learning supplies a large gain, and selection determines how much reaches deployment. Appendix~\ref{app:empirical-head} gives all trajectories and settings.
\par\WFclear

\subsection{Retaining an update changes future predictions}
A useful current choice need not improve the next batch. We compare two model copies with identical future inputs, labels and updates. One retains the current update; the other undoes it. Retaining the update gains 0.23 points now but changes next-batch accuracy by $-0.20$ points. Undoing it leaves future predictions unchanged. Appendix~\ref{app:expanded_state} relates these effects to the update dynamics and bounds the later loss change. Learning a candidate and deciding when to deploy it remain distinct decisions.
\FloatBarrier
\endgroup
\space
% END inlined sections/expanded_experiments.tex
% BEGIN inlined sections/expanded_conclusion.tex
\section{Conclusion}
An intervention's value depends on the batch it receives. A rule that helps under ordinary arrivals can fail when classes arrive together. \PORTAL{} uses shared feedback to choose when to deploy each candidate and stops buying labels when they cannot improve that choice. Holding the challenger fixed shows how targeted feedback turns an aggregate loss into useful local gains. Periodic checks keep the evidence current. The same labels train stronger future candidates, and selection brings their gains to deployment.
\space
% END inlined sections/expanded_conclusion.tex
\label{main_text_end}
\clearpage
\section*{AI use statement}
ChatGPT was used only for literature search. The authors are responsible for the final content.
\section*{Reproducibility statement}
The appendix contains the model settings, query/evaluation splits, update timing, query budgets, and statistical procedures. It also provides complete corruption and domain results, comparisons of method components, and proofs. The changing-label experiments include their selector definitions, reconstructed comparison methods, and independent sampling units. The manuscript, implementation, released result data, and available frozen prediction banks are provided in our anonymous repository: \url{https://anonymous.4open.science/r/PORTAL-8322/}. Its data inventory lists the external datasets and checkpoints needed for further reproduction.
\bibliographystyle{iclr2027_conference}
% Keep the supplied BibTeX metadata verbatim. Its D2O title writes the
% superscript as $\^{} 2$; interpret that spelling only while printing references.
\begingroup
\let\portalSavedHat\^
\def\portalBibSup#1 #2{^{#2}}
\def\^{\ifmmode\let\portalNextHat\portalBibSup\else\let\portalNextHat\portalSavedHat\fi\portalNextHat}
\bibliography{references}
\endgroup
\clearpage
\appendix
% BEGIN inlined sections/expanded_guide.tex
\section{Notation and organization of the analysis}
The analyses ask three questions. How much can shared labels tell us about candidate accuracy? How often does a finite query budget select the better output? How does retaining an update affect future predictions? Tables~\ref{tab:appendix_symbols} and~\ref{tab:appendix_model_map} summarize the models and results. Appendix~\ref{app:round_evidence} adds stopping comparisons, calculations on real prediction pairs, and tests with changing class probabilities.

\begin{table}[h]
\centering\small
\caption{Statistical objects used throughout the paper.}
\label{tab:appendix_symbols}
\begin{tabularx}{\linewidth}{lX}
\toprule
Symbol & Meaning\\
\midrule
$\mathcal A_t$, $I$ & Available batch interventions and incumbent at time $t$.\\
$Q_t$, $E_t$ & Permanent eligible query and deployment/evaluation inputs.\\
$R_a$, $D_a$ & Candidate correctness and its difference from incumbent correctness.\\
$s$, $w_s$, $p_{sa}$ & Joint prediction stratum, target frequency and candidate class.\\
$\theta_s$, $A_a(\theta)$ & Conditional label law and expected target accuracy.\\
$\Delta_a$, $v_{s,a}$ & Best-candidate gap and within-stratum paired reward variance.\\
$q_s$, $N_s$, $\tau$ & Query allocation, acquired stratum count and stopping time.\\
$\alpha_s$, $\mu_a$, $\Sigma_{ab}$ & Dirichlet parameters, posterior accuracy mean and covariance.\\
$V_s$, $U$ & One-label decision value and perfect-information upper bound.\\
$r_a$, $b_a$, $\mathcal R_m$ & Deployment risk, Gaussian evidence bias and selection regret.\\
$n_+,n_-,n_0$, $S_m$ & Finite-pool rescue/harm/neutral counts and queried signed score.\\
$\Delta\Theta$, $\Delta m$, $d$ & Residual-head difference, cache difference and induced logit change.\\
\bottomrule
\end{tabularx}
\end{table}

\begin{table}[h]
\centering\small
\caption{Result dependencies and the quantity each result predicts.}
\label{tab:appendix_model_map}
\begin{tabularx}{\linewidth}{p{.22\linewidth}XX}
\toprule
Analysis & Required observation structure & Result and validation\\
\midrule
Shared-label complexity & Independent categorical labels within known, queryable strata; fixed prediction signatures & Information lower bound, paired-variance upper bound and local allocation coefficient.\\
Posterior querying & Positive independent Dirichlet priors and shared stratum label laws & Joint accuracy covariance, one-label value and sufficient local stopping conditions.\\
Gaussian finite budget & Common-factor joint Gaussian feedback with fixed risks and biases & Exact winner probabilities, budget and correlation responses; independent score simulation.\\
Frozen neural pools & Uniform subsets of a specified finite query pool; frozen predictions and evaluation outcomes & Exact discrete selection and deployment loss; independent subset simulation.\\
Persistent state & Shared future inputs, query labels and affine head/cache updates & Exact state recurrence, directional loss interval and margin bound; explicit branch replay.\\
\bottomrule
\end{tabularx}
\end{table}

The batch oracle in empirical tables chooses the best candidate after seeing that batch's evaluation labels. The categorical population oracle maximizes $A_a(\theta)$. The posterior analysis gives the covariance adjustment required to analyze a realized finite evaluation block. Mean-logit and learned-readout outputs expand the prediction functions themselves and are listed separately when they fall outside an oracle's discrete candidate set.
\space
% END inlined sections/expanded_guide.tex
% BEGIN inlined sections/expanded_joint_feedback.tex
\section{Shared-label information and finite-candidate selection}
\label{app:joint_feedback}
This appendix establishes the information bound and local allocation result in Theorem~\ref{thm:expanded_information}, then derives the posterior calculations and stopping rule in Theorem~\ref{thm:expanded_stop}. The observation is a class label. Its value comes from the joint vector of candidate outcomes that it reveals.

\subsection{Statistical experiment and assumptions}
There are $S$ strata, $C$ classes and $K$ candidate rules. A candidate $a$ predicts class $p_{sa}$ throughout stratum $s$. These prediction signatures are fixed during the decision. The known target weights satisfy $w_s>0$ and $\sum_s w_s=1$. A query in stratum $s$ returns an independent label $Y\sim\operatorname{Categorical}(\theta_s)$, with $\theta_s\in\Delta_C$, where $\Delta_C$ contains nonnegative class probabilities summing to one. Query and target labels have this same conditional law. The population accuracy of candidate $a$ is
\begin{equation}
 A_a(\theta)=\sum_{s=1}^S w_s\theta_{s,p_{sa}}.
 \label{eq:app_accuracy}
\end{equation}
A stratum may be queried repeatedly in this experiment. A fixed finite pool sampled without replacement has a different sampling law, treated in the finite-budget appendix. Conditional agreement of query and target label laws is the transport assumption: changing the frequency of strata is corrected by $w_s$, while a changed label law within a stratum requires a richer model.

Rules with the same prediction signature on every positive-weight stratum are represented once. For fixed-confidence identification, the remaining set has a unique best rule $*$ and gaps $\Delta_a=A_*-A_a>0$ for $a\ne *$. The case $K=1$ requires no queries. A uniformly $\delta$-correct strategy, $0<\delta<1/2$, stops almost surely and returns the best rule with probability at least $1-\delta$ at every admissible unique-best parameter. Randomized stratum choices and stopping times may depend on previously observed labels.

The Bayesian calculations add independent priors $\theta_s\sim\operatorname{Dirichlet}(\alpha_s^0)$, with every coordinate strictly positive. They describe posterior expected population accuracy under this model. The finite-evaluation oracle, which sees the realized labels of a target block, is distinguished in Section~\ref{app:posterior_finite}.

\subsection{Paired reward geometry}
One observed label produces all rewards $R_a=\mathbf1\{p_{sa}=Y\}$. Write $D_{s,ab}=R_a-R_b$, $d_{s,ab}=\mathbb E D_{s,ab}$, and $v_{s,ab}=\operatorname{Var}(D_{s,ab})$. On agreement strata $D_{s,ab}=0$. On disagreement strata its values $1,-1,0$ have probabilities $\theta_{s,p_{sa}},\theta_{s,p_{sb}},1-\theta_{s,p_{sa}}-\theta_{s,p_{sb}}$. Consequently,
\begin{align}
 d_{s,ab}&=\theta_{s,p_{sa}}-\theta_{s,p_{sb}},\\
 v_{s,ab}&=\mathbf1\{p_{sa}\ne p_{sb}\}
 \left[\theta_{s,p_{sa}}+\theta_{s,p_{sb}}-d_{s,ab}^2\right],
 \label{eq:app_pairvar}\\
 \operatorname{Cov}(R_a,R_b\mid s)
 &=\mathbf1\{p_{sa}=p_{sb}\}\theta_{s,p_{sa}}
 -\theta_{s,p_{sa}}\theta_{s,p_{sb}}.
 \label{eq:app_rewardcov}
\end{align}
These expressions retain dependence between candidates. A label is counted once, including when many candidates predict its class. Let $\mathcal I_{ab}=\{s:p_{sa}\ne p_{sb}\}$. Both the pair's mean difference and its estimation variance depend exclusively on these disagreement strata.

\subsection{Adaptive information lower bound}
Let $N_s(\tau)$ be the number of labels from stratum $s$ before stopping. Define
\begin{align}
 \operatorname{Alt}(\theta)
 &=\{\lambda:\lambda\text{ has a unique best candidate different from }*\},\\
 \Gamma^*(\theta)
 &=\sup_{q\in\Delta_S}\inf_{\lambda\in\operatorname{Alt}(\theta)}
 \sum_s q_s\operatorname{KL}(\theta_s\Vert\lambda_s),
 \qquad T^*(\theta)=\Gamma^*(\theta)^{-1}.
 \label{eq:app_characteristic}
\end{align}
The divergence is categorical KL, $\operatorname{KL}(\theta_s\Vert\lambda_s)=\sum_c\theta_{sc}\log(\theta_{sc}/\lambda_{sc})$.

\begin{lemma}[Stopped shared-label likelihood]
\label{lem:app_transcript}
For every uniformly $\delta$-correct strategy and every admissible alternative $\lambda$,
\begin{equation}
 \sum_s\mathbb E_\theta N_s(\tau)
 \operatorname{KL}(\theta_s\Vert\lambda_s)
 \ge\operatorname{kl}(1-\delta,\delta).
 \label{eq:app_change_measure}
\end{equation}
Here $\operatorname{kl}$ is binary relative entropy. Interior parameters suffice; boundary parameters use the usual absolute-continuity and extended-divergence conventions.
\end{lemma}
\begin{proof}
The transcript contains randomization, chosen strata, labels, the stopping decision and the final recommendation. Given its observed history, the strategy uses the same action kernel under $\theta$ and $\lambda$. These kernels cancel in the likelihood ratio. Thus
\[
 \log\frac{dP_\theta}{dP_\lambda}
 =\sum_{t=1}^{\tau}\log\frac{\theta_{S_t,Y_t}}{\lambda_{S_t,Y_t}}.
\]
For interior finite categorical laws, each increment is bounded in absolute value. If $\mathbb E_\theta\tau<\infty$, absolute summability justifies exchanging expectation and the stopped sum. Conditioning on the preceding history gives the left-hand side of~\eqref{eq:app_change_measure}. If expected time is infinite, the expected-time lower bound is already satisfied.

Apply data processing to $B=\{\widehat a=*\}$. Correctness implies $P_\theta(B)\ge1-\delta$ and $P_\lambda(B)\le\delta$. Binary KL is increasing in its first argument and decreasing in its second when the first exceeds the second: its derivatives are $\log[p(1-q)/(q(1-p))]$ and $(q-p)/(q(1-q))$. Transcript KL is therefore at least $\operatorname{kl}(1-\delta,\delta)$.
\end{proof}

Normalize the expected counts by $\mathbb E_\theta\tau$ to obtain an allocation in $\Delta_S$. Taking the infimum over alternatives and then the supremum over allocations proves Equation~\eqref{eq:expanded_lower}. An empty alternative set is a trivial identification problem; zero characteristic information yields an infinite lower bound. The proof is the change-of-measure principle of \citet{garivier2016optimal} applied to a shared categorical observation.

For a fixed $q$, the inner optimization is computable through $K-1$ convex programs. Its alternative closure is the union of constraints $A_a(\lambda)\ge A_*(\lambda)$, $a\ne *$, over product simplices. To verify the closure, concentrating each stratum on candidate $a$'s predicted class makes $a$ uniquely best among distinct signatures. A boundary point where $a$ and $*$ share the maximum can be perturbed toward that law. If another rule already exceeds $*$, the strict inequality persists locally. A further arbitrarily small interior perturbation breaks maximal ties: distinct linear accuracy functions have equality hyperplanes with empty interior. The infimum over unique-best alternatives consequently agrees with that over these weak boundary constraints. This argument uses the full product-simplex parameter space.

\subsection{A finite-sample confidence procedure}
\label{app:fixed_allocation}
Choose $q_s>0$ before collecting certification labels. At stage $h\ge0$, put $t_h=2^h$ and gather the first $n_s(h)=\lceil t_hq_s\rceil+2$ labels from each independent stratum stream, reusing previous prefixes. Let $M=K(K-1)/2$ and
\begin{equation}
 H_h=\log\frac{4SM(h+1)(h+2)}{\delta}.
 \label{eq:app_stage_threshold}
\end{equation}
All unique candidates remain in the certification set at every stage. The empirical maximizer is returned when its weighted empirical reward difference exceeds the corresponding radius against every other candidate.

\begin{lemma}[Paired variance and mean confidence]
For $n_s\ge2$, set $k_s=\lfloor n_s/2\rfloor$ and form disjoint observation pairs
\[
 W_{i,s,ab}=\tfrac12(D_{2i-1,s,ab}-D_{2i,s,ab})^2,
 \qquad i=1,\ldots,k_s.
\]
Define $U_{s,ab}=0$ on agreement strata and otherwise
\begin{equation}
 U_{s,ab}=\min\{1,2\overline W_{s,ab}+8H_h/k_s\}.
 \label{eq:app_var_ucb}
\end{equation}
With probability at least $1-\delta$, simultaneously for every stage and pair,
\begin{align}
 v_{s,ab}&\le U_{s,ab}\le4v_{s,ab}+13H_h/k_s,\\
 |\widehat d_{ab}-(A_a-A_b)|&\le\rho_{ab}(h),\\
 \rho_{ab}(h)&=
 \sqrt{2H_h\sum_s\frac{w_s^2U_{s,ab}}{n_s(h)}}
 +\frac{4H_h}{3}\max_{s\in\mathcal I_{ab}}\frac{w_s}{n_s(h)}.
 \label{eq:app_pair_radius}
\end{align}
The maximum over an empty disagreement set equals zero.
\end{lemma}
\begin{proof}
For independent centered $Z_i$, $|Z_i|\le B$, and $\sum_i\mathbb E Z_i^2\le V$, the exponential-moment expansion gives
\[
 \log\mathbb E e^{tZ_i}\le
 \frac{t^2\mathbb E Z_i^2}{2(1-tB/3)},\qquad 0\le t<3/B.
\]
Indeed, $\mathbb E|Z_i|^k\le B^{k-2}\mathbb E Z_i^2$ and $k!\ge2\cdot3^{k-2}$. Multiplication, exponential Markov and optimization yield the two-sided Bernstein bound
\[
 P\left(\left|\sum_iZ_i\right|>\sqrt{2Vx}+2Bx/3\right)\le2e^{-x}.
\]
Independence within each disjoint pair gives $\mathbb E W=v$, while $0\le W\le2$ and $\operatorname{Var}(W)\le2v$. For $a=H_h/k_s$ the preceding bound implies
$|\overline W-v|\le2\sqrt{va}+4a/3$.
Using $2\sqrt{va}\le v/2+2a$ gives $v\le2\overline W+20a/3$. Using $2\sqrt{va}\le v+a$ gives $2\overline W+8a\le4v+38a/3$. Since $v\le1$, clipping preserves the lower bound in~\eqref{eq:app_var_ucb}.

For a weighted empirical pair difference, the centered summands have total variance $\sum_s w_s^2v_s/n_s$ and absolute bound $2\max_{s\in\mathcal I_{ab}}w_s/n_s$. Bernstein gives~\eqref{eq:app_pair_radius} whenever the variance upper bounds hold. At stage $h$, union bounds over pairs and stratum-pair variance estimates cost at most $2M e^{-H_h}+2SM e^{-H_h}$. Summing stages uses $\sum_{h\ge0}[(h+1)(h+2)]^{-1}=1$, so total failure probability is at most $\delta$. The paired events need not be independent.
\end{proof}

Define
\begin{equation}
 V_{ab}(q)=\sum_s\frac{w_s^2v_{s,ab}}{q_s},\qquad
 W_{ab}(q)=\left(\sum_{s\in\mathcal I_{ab}}\frac{w_s^2}{q_s^2}\right)^{1/2}.
 \label{eq:app_VW}
\end{equation}
Because $k_s\ge n_s/3$, the confidence event implies
\begin{equation}
 \rho_{ab}(h)\le3\sqrt{H_hV_{ab}(q)/t_h}
 +11H_hW_{ab}(q)/t_h.
 \label{eq:app_radius_envelope}
\end{equation}
For the constants, split the square root into $\sqrt{8H_h\sum_s w_s^2v_s/n_s}$ and $\sqrt{78}H_h(\sum_{s\in\mathcal I}w_s^2/n_s^2)^{1/2}$; the range term adds $4/3$ times the latter norm. Use $\sqrt8<3$ and $\sqrt{78}+4/3<11$.

\begin{theorem}[Fixed-allocation identification]
\label{thm:app_fixed}
The full-set certification procedure is uniformly $\delta$-correct and stops almost surely at every unique-best law. At any stage satisfying
\begin{equation}
 t_h\ge\max_{a\ne *}\max\left\{
 \frac{576H_hV_{*,a}(q)}{\Delta_a^2},
 \frac{88H_hW_{*,a}(q)}{\Delta_a}\right\},
 \label{eq:app_fixed_time}
\end{equation}
it has stopped after at most $t_h+3S$ labels with probability at least $1-\delta$.
\end{theorem}
\begin{proof}
On the simultaneous event, an incorrect certified winner $j$ would satisfy $A_j-A_*\ge\widehat d_{j*}-\rho_{j*}>0$, a contradiction. Condition~\eqref{eq:app_fixed_time} makes each term of~\eqref{eq:app_radius_envelope} at most $\Delta_a/8$. Thus $\widehat d_{*,a}\ge3\Delta_a/4>\rho_{*,a}$ for all rivals: the true best is the empirical maximizer and is certified. Counts sum to at most $t_h+3S$.

For almost-sure termination, apply the strong law to every bounded independent stratum stream. All empirical accuracies converge, every radius tends to zero because $U\le1$, and the finite set has a positive minimum best-versus-rival gap. Eventually the best is the empirical maximizer and all its empirical gaps exceed their radii. This argument applies on the entire probability-one convergence event, including paths outside the simultaneous confidence event. Keeping the full candidate set in certification is essential to this argument when inferior population accuracies tie.
\end{proof}

One query informs every pair. The leading variance term is therefore a maximum of pair difficulties, $\max_{a\ne *}V_{*,a}(q)/\Delta_a^2$, with explicit range and logarithmic terms. For a single pair, Cauchy--Schwarz gives
\begin{equation}
 \inf_{q\in\Delta_S}V_{ab}(q)
 =\left(\sum_s w_s\sqrt{v_{s,ab}}\right)^2,
 \qquad q_s\propto w_s\sqrt{v_{s,ab}}.
 \label{eq:app_single_design}
\end{equation}
Zero-variance strata can receive zero mass in this variance-only optimization; the preceding finite-sample procedure retains positive allocations and its range term.

For outcome-dependent counts, correctness can instead be based on a simultaneous bound over all prefix vectors. Replacing the mean threshold by
$\log(4M/\delta)+\sum_s\log[n_s(n_s-1)]$ and the variance threshold at prefix-pair count $k$ by $\log[4SMk(k+1)/\delta]$ gives the same confidence conclusion through summable union bounds. The price is an $S\log n$ term. This supplies an adaptive certificate, while a stopping rate additionally requires the policy to sample every unresolved distinguishing stratum.

\subsection{Fixed-budget regret}
For deterministic stratum counts $n_s\ge1$, let $\widehat a$ maximize weighted empirical reward. Define
\[
 V_a(n)=\sum_s w_s^2v_{s,*,a}/n_s,\qquad
 B_a(n)=\max_{s\in\mathcal I_{*,a}}w_s/n_s.
\]
Then
\begin{equation}
 \mathbb E[A_*-A_{\widehat a}]
 \le\min\left\{1,\sum_{a\ne *}\Delta_a
 \exp\left[-\frac{\Delta_a^2}{2V_a(n)+(4/3)B_a(n)\Delta_a}\right]\right\}.
 \label{eq:app_budget_regret}
\end{equation}
Selecting $a$ requires $\widehat d_{*,a}\le0$. The one-sided Bernstein exponent for centered bound $2B_a(n)$ gives its probability; multiplying by $\Delta_a$ and summing possible winners proves the result. A deterministic positive-gap pair has zero error probability and may contribute zero directly. Independent pilot-selected counts followed by fresh samples satisfy the same conditional argument. Arbitrary stopped counts require a simultaneous certificate.

If candidate $0$ is the incumbent, the deployment relation is
\begin{equation}
 \mathbb E[A_{\widehat a}-A_0]
 =(A_*-A_0)-\mathbb E[A_*-A_{\widehat a}].
 \label{eq:app_capacity_regret}
\end{equation}
Thus a stronger portfolio improves available accuracy while the feedback mechanism determines how much of that improvement is realized. Under conditional query/target exchangeability, an independent target block with the specified stratum weights estimates the left side.

\section{Local information geometry and an attainable allocation}
\label{app:local_allocation}
The finite-sample bound retains conservative constants and a range term. We next identify a regime in which paired variance gives the leading information complexity and construct a procedure attaining that coefficient.

\subsection{Uniform expansion of the exact categorical problem}
For a best-versus-rival pair, define $g_{sc}=\mathbf1\{p_{s*}=c\}-\mathbf1\{p_{sa}=c\}$, $d_s=\sum_c\theta_{sc}g_{sc}$, and $\Delta=\sum_s w_sd_s>0$. Let
\[
 I_a(\theta,q)=\inf_{\lambda:\sum_{s,c}w_sg_{sc}\lambda_{sc}\le0}
 \sum_s q_s\operatorname{KL}(\theta_s\Vert\lambda_s).
\]
Fix $b_0>0$ and $\eta>0$, with $Cb_0<1$ and $S\eta<1$, and restrict to $\theta_{sc}\ge b_0$ and $Q_\eta=\{q\in\Delta_S:q_s\ge\eta\}$.

\begin{lemma}[Uniform pairwise information expansion]
\label{lem:app_local}
For a fixed finite prediction map and positive target weights,
\begin{equation}
 I_a(\theta,q)=\frac{\Delta^2}{2V_a(\theta,q)}+O(\Delta^3)
 \label{eq:app_local_expansion}
\end{equation}
uniformly on the preceding domain as $\Delta\downarrow0$. The remainder constant depends on $b_0$, $\eta$, the weights and the prediction map.
\end{lemma}
\begin{proof}
The pair is distinct, so its variance functional is bounded above and away from zero on this compact domain. Define the tangent displacement
\begin{equation}
 u_{sc}=-\frac{\Delta}{V_a}\frac{w_s}{q_s}
 \theta_{sc}(g_{sc}-d_s).
 \label{eq:app_tangent}
\end{equation}
Its rows sum to zero and $\sum_{s,c}w_sg_{sc}u_{sc}=-\Delta$. It is uniformly $O(\Delta)$, making $\theta+u$ a feasible interior boundary law for sufficiently small $\Delta$. Taylor expansion has no linear term on row-sum-zero tangents, and uniformly bounded third derivatives give
\begin{equation}
 \sum_s q_s\operatorname{KL}(\theta_s\Vert\theta_s+u_s)
 =\frac12\sum_{s,c}\frac{q_su_{sc}^2}{\theta_{sc}}+O(\|u\|^3).
 \label{eq:app_KL_taylor}
\end{equation}
Substitution gives the required upper bound.

An exact minimizer exists by lower semicontinuity on the compact product simplex; the objective diverges at a zero coordinate because $q_s\theta_{sc}>0$. Its Hessian in $\lambda$ is diagonal with entries $q_s\theta_{sc}/\lambda_{sc}^2\ge b_0\eta$. Since the gradient at $\theta$ vanishes on simplex tangents, the objective is at least $b_0\eta\|\lambda-\theta\|_2^2/2$. The feasible $O(\Delta^2)$ cost forces any optimizer into an $O(\Delta)$ neighborhood uniformly. Its constraint is active: a point strictly beyond the boundary can be moved toward $\theta$ until it meets the boundary, strictly reducing the positive convex objective.

For any boundary tangent, centering $g$ by $d_s$ and applying Cauchy--Schwarz gives
\[
 \Delta^2\le
 \left(\sum_{s,c}\frac{q_su_{sc}^2}{\theta_{sc}}\right)
 \left(\sum_{s,c}\frac{w_s^2\theta_{sc}(g_{sc}-d_s)^2}{q_s}\right).
\]
The second factor is $V_a$; the displacement~\eqref{eq:app_tangent} attains equality. Apply~\eqref{eq:app_KL_taylor} to the exact local optimizer to obtain the matching lower bound. This proves the expansion of the actual constrained KL minimum.
\end{proof}

\subsection{Candidate families and allocation optimization}
Consider $\theta^\varepsilon\to\theta^0$ with every categorical probability at least $b_0$, and with the same uniquely best candidate $*$ for $\varepsilon>0$. A nonempty set $J$ of rivals ties with $*$ at the limit. Assume
\begin{equation}
 \Delta_a^\varepsilon=\kappa_a\varepsilon+O(\varepsilon^2),\quad
 \kappa_a>0\ (a\in J),\qquad
 \liminf_{\varepsilon\downarrow0}\Delta_a^\varepsilon>0\ (a\notin J).
 \label{eq:app_local_family}
\end{equation}
Set
\begin{equation}
 H_{0,\eta}=\min_{q\in Q_\eta}\max_{a\in J}
 \frac{V_a(\theta^0,q)}{\kappa_a^2}.
 \label{eq:app_local_design}
\end{equation}
Uniformity of Lemma~\ref{lem:app_local} permits taking the finite minimum over rivals and the supremum over compact $Q_\eta$, yielding
\begin{equation}
 \sup_{q\in Q_\eta}\min_{a\ne *}I_a(\theta^\varepsilon,q)
 =\frac{\varepsilon^2}{2H_{0,\eta}}+o(\varepsilon^2).
 \label{eq:app_restricted_information}
\end{equation}
Nonlocal rivals have uniformly positive boundary information by the Hessian lower bound, so they do not determine this minimum for small $\varepsilon$. The design is convex: introduce a scalar upper bound on every $\sum_s w_s^2v_{s,a}/(\kappa_a^2q_s)$, together with the simplex and floor constraints.

\paragraph{When the allocation floor is inactive.}
Suppose every stratum separates $*$ from at least one rival in $J$. The unrestricted variance objective diverges when any $q_s\to0$, because an interior categorical law gives that pair positive variance there. Its minimizers consequently lie in the simplex interior. Exact information-maximizing allocations are also uniformly interior in the local limit. A fixed interior allocation attains at least $c\varepsilon^2$. For each stratum, changing only its categorical law by $O(\varepsilon)$ closes the gap of an associated local rival, at unweighted KL cost at most $C_s\varepsilon^2$. The information under any $q$ is therefore at most $C_s q_s\varepsilon^2$. Competing with the fixed interior allocation requires $q_s\ge c/C_s$. Thus, for a sufficiently small fixed floor containing a local minimizer,
\begin{equation}
 \Gamma^*(\theta^\varepsilon)=\frac{\varepsilon^2}{2H_0}+o(\varepsilon^2),
 \qquad T^*(\theta^\varepsilon)=\frac{2H_0}{\varepsilon^2}(1+o(1)),
 \label{eq:app_unrestricted_local}
\end{equation}
where $H_0=H_{0,\eta}$. If this coverage condition fails, Equation~\eqref{eq:app_restricted_information} still describes fixed positive-floor designs; boundary allocations require a separate analysis.

\subsection{Pilot design with fresh-sample certification}
Let $L=\log(1/\delta)$ and fix $0<\eta<1/S$. The following procedure uses only queried labels and known prediction signatures.
\begin{enumerate}
 \item Query each stratum $n_0=\lceil\max\{16,\sqrt L\}\rceil$ times and form smoothed categorical frequencies $\widehat\theta$ using fixed positive pseudocounts.
 \item Compute the proxy best $\widehat *$ and choose
 \begin{equation}
 \widehat q\in\arg\min_{q\in Q_\eta}\max_{a\ne\widehat *}
 \frac{V_{\widehat *,a}(\widehat\theta,q)}
 {\max\{\widehat\Delta_a,(n_0+1)^{-1}\}^2}.
 \label{eq:app_pilot_design}
 \end{equation}
 \item At checkpoints $t_0=1$, $t_{h+1}=\lceil(1+\gamma)t_h\rceil$, collect $\lceil t_h\widehat q_s\rceil+2$ fresh labels in each stratum. Pilot observations do not enter the certification statistics.
 \item Return the empirical maximizer when its mean difference is certified positive against every other unique candidate, keeping the full set in every certification.
\end{enumerate}
Conditioned on the pilot, the subsequent prefix counts are deterministic and their labels are independent of the design. An inaccurate pilot changes efficiency, while the fresh-sample certificate remains valid.

To obtain the local leading coefficient, use the tighter variance upper bound
\begin{equation}
 U_s=\min\left\{1,
 \left(\sqrt{a_s}+\sqrt{\overline W_s+7a_s/3}\right)^2\right\},
 \qquad a_s=H_h/\lfloor n_s/2\rfloor,
 \label{eq:app_tight_U}
\end{equation}
with zero on agreement strata. The Bernstein event from the preceding appendix gives
\begin{equation}
 v_s\le U_s\le v_s+6\sqrt{v_sa_s}+9a_s.
 \label{eq:app_tight_U_bounds}
\end{equation}
For the lower inequality solve $v-\overline W\le2\sqrt{va}+4a/3$ as a quadratic in $\sqrt v$. For the upper inequality, $\overline W+7a/3\le v+2\sqrt{va}+11a/3\le(\sqrt v+2\sqrt a)^2$. Add $\sqrt a$ and square. The stage allocation~\eqref{eq:app_stage_threshold} again gives simultaneous confidence with total failure probability at most $\delta$.

Use~\eqref{eq:app_pair_radius} with~\eqref{eq:app_tight_U}. On this event,
\begin{equation}
 \rho_{*,a}(h)\le\sqrt{2H_hV_a(\widehat q)/t_h}
 +9H_hW_a(\widehat q)/t_h.
 \label{eq:app_tight_radius}
\end{equation}
For completeness, write $V_n=\sum_s w_s^2v_s/n_s$ and $J_n=\sum_{s\in\mathcal I}w_s^2/n_s^2$. Since $k_s\ge n_s/3$, Cauchy--Schwarz gives
\[
 \sum_s\frac{w_s^2U_s}{n_s}
 \le V_n+6\sqrt3\sqrt{H_hV_nJ_n}+27H_hJ_n
 =(\sqrt{V_n}+3\sqrt3\sqrt{H_hJ_n})^2.
\]
The resulting range coefficient is $3\sqrt6+4/3<9$. In particular, if
\begin{equation}
 t_h\ge\max_{a\ne *}\max\left\{
 128H_hV_a(\widehat q)/\Delta_a^2,
 72H_hW_a(\widehat q)/\Delta_a\right\},
 \label{eq:app_two_stage_finite}
\end{equation}
then the total labels are at most $Sn_0+t_h+3S$ with probability at least $1-\delta$. The full-set argument of Theorem~\ref{thm:app_fixed} also proves almost-sure stopping.

\begin{theorem}[Ordered local efficiency]
\label{thm:app_local_efficiency}
Let the interior local family satisfy~\eqref{eq:app_local_family}, and define
\[
 H_\eta(\theta)=\min_{q\in Q_\eta}\max_{a\ne *}
 \frac{V_a(\theta,q)}{\Delta_a(\theta)^2}.
\]
For the uncapped two-stage procedure with exact convex minimization, or relative optimization error tending to zero as $\delta\downarrow0$,
\begin{equation}
 \limsup_{\varepsilon\downarrow0}\limsup_{\delta\downarrow0}
 \frac{\mathbb E_{\theta^\varepsilon}\tau_\delta}
 {2H_\eta(\theta^\varepsilon)\log(1/\delta)}\le1+\gamma.
 \label{eq:app_ordered_rate}
\end{equation}
With $\gamma=\gamma_\delta=L^{-1/4}$, the right-hand side is one. Under the active-coverage and inactive-floor conditions above, this coefficient equals the local information lower bound.
\end{theorem}
\begin{proof}
First fix an interior unique-best $\theta$. As $\delta\downarrow0$, the pilot size diverges and is $o(L)$. The proxy best is correct with probability tending to one. All true gaps are positive at this fixed parameter, so the gap floor disappears asymptotically, and the pilot objective converges uniformly on compact $Q_\eta$. Its achieved true objective therefore converges in probability to $H_\eta(\theta)$. No uniqueness of the allocation minimizer is needed.

Now consider $\theta=\theta^\varepsilon$ and the scale
$t=2(1+\zeta)H_\eta(\theta^\varepsilon)L$, for fixed $\zeta>0$. At each fixed positive $\varepsilon$, fresh means and disjoint-pair variance estimates converge uniformly in probability over $q\in Q_\eta$: every count is at least $\eta t$. The geometric stage count is $O_\gamma(\log L)$, hence $H_h/L\to1$. Uniform local convergence also gives $\varepsilon^2H_\eta(\theta^\varepsilon)\to H_{0,\eta}$.

At this scale $H_h/k_s=O(\varepsilon^2)$. Substituting the consistent pair variances into~\eqref{eq:app_tight_U} gives $U_s=v_s+O(\varepsilon)$ on disagreement strata. Their true variances are bounded away from zero by the categorical interior assumption. For $a\in J$, the limiting radius-to-gap ratio is at most
\[
 \frac{1+O(\varepsilon)}{\sqrt{1+\zeta}}+O(\varepsilon).
\]
The range term is $O(\varepsilon^2)$ and the local gap is of order $\varepsilon$. Nonlocal rivals retain positive gaps while their radii vanish with $\varepsilon$. For sufficiently small $\varepsilon$, all limiting radii are strictly below the true gaps. Empirical errors vanish as $\delta\downarrow0$ at this fixed $\varepsilon$, so certification occurs by this scale with probability tending to one. Moving to the next checkpoint costs a factor at most $1+\gamma+o(1)$.

To pass to expectation, use $U\le1$ and the allocation floor. Uniformly over pilot outcomes,
$\rho_{*,a}(h)\le C_\eta\sqrt{H_h/t_h}+C_\eta H_h/t_h$.
At a fixed unique-best $\theta$, a constant $C_\theta$ makes every radius smaller than $\Delta_{\min}/4$ for all $t_h\ge C_\theta L$ and sufficiently large $L$. Here $H_h\le L+C+2\log\log(t_h+e)$. Failure to stop at such a checkpoint requires a best-versus-rival empirical difference below half its true gap. Conditional on any pilot, bounded independent fresh streams and $n_s\ge\eta t_h$ give
\begin{equation}
 P_\theta(\tau>Sn_0+t_h+3S\mid\text{pilot})
 \le(K-1)e^{-c_\theta t_h}
 \label{eq:app_time_tail}
\end{equation}
for $c_\theta>0$ independent of the pilot allocation. Summing this tail over geometric intervals proves uniform integrability of $\tau/L$. The probability bound consequently passes to expectations; pilot cost is $Sn_0=o(L)$. Taking the ordered limits and then $\zeta\downarrow0$ proves~\eqref{eq:app_ordered_rate}.

For $\gamma_\delta=L^{-1/4}$, checkpoint overshoot tends to one and the logarithm of the stage count remains $o(L)$. The same tail argument, with at most polynomially many additional checkpoints at the relevant scale, proves the unit coefficient. Finally~\eqref{eq:app_unrestricted_local} and the adaptive lower bound identify the unrestricted coefficient when the floor contains a local oracle.
\end{proof}

This result takes confidence error to zero at each fixed positive gap before taking the local gap limit. A simultaneous shrinking-gap regime would also have to control the pilot's probability of identifying the relevant competitors. The posterior acquisition rule uses the same paired structure, while its finite-batch decisions are evaluated directly in the experiments.

\section{Posterior joint information and stopping}
\label{app:posterior_joint}
\subsection{Posterior moments and one-label updates}
After observing stratum class counts $c_{sc}$, let $\alpha_{sc}=\alpha_{sc}^0+c_{sc}$, $a_s=\sum_c\alpha_{sc}$ and $m_{sc}=\alpha_{sc}/a_s$. Adaptive stratum selection is ignorable in this likelihood when queried inputs are conditionally exchangeable with deployment inputs within a stratum: its action kernels depend on observed history, while the remaining parameter factors are the categorical label probabilities. Selection favoring a subgroup with a different label law is not ignorable. A single class label updates its stratum count once.

Dirichlet integration gives
\[
 \mathbb E\theta_{sc}=\frac{\alpha_{sc}}{a_s},\qquad
 \mathbb E[\theta_{sc}\theta_{sd}]
 =\frac{\alpha_{sc}(\alpha_{sd}+\mathbf1\{c=d\})}{a_s(a_s+1)}.
\]
These identities follow by adding one or two to the corresponding exponents in the normalized Dirichlet density. Independence between strata and the linear accuracy definition then give Equations~\eqref{eq:expanded_mu}--\eqref{eq:expanded_cov}:
\begin{align}
 \mu_a&=\sum_s w_sm_{s,p_{sa}},\\
 \Sigma_{ab}&=\sum_s\frac{w_s^2}{a_s+1}
 \left[\mathbf1\{p_{sa}=p_{sb}\}m_{s,p_{sa}}
 -m_{s,p_{sa}}m_{s,p_{sb}}\right].
 \label{eq:app_dirichlet_cov}
\end{align}
Thus $\operatorname{Var}(A_a-A_b\mid\mathcal F)=\Sigma_{aa}+\Sigma_{bb}-2\Sigma_{ab}$. Exact prediction copies have zero difference variance.

If the next label from stratum $s$ is $y$, each candidate mean changes by
\begin{equation}
 \mu_a'=\mu_a+\frac{w_s}{a_s+1}
 \left[\mathbf1\{p_{sa}=y\}-m_{s,p_{sa}}\right].
 \label{eq:app_posterior_update}
\end{equation}
Weighting the new maximum by $P(Y=y\mid\mathcal F)=m_{sy}$ yields the one-step value in Equation~\eqref{eq:expanded_voi}. The law of total variance further gives the expected reduction in posterior pair variance:
\begin{equation}
 \operatorname{Var}(A_a-A_b\mid\mathcal F)
 -\mathbb E[\operatorname{Var}(A_a-A_b\mid\mathcal F,Y_s)\mid\mathcal F]
 =\frac{w_s^2\operatorname{Var}_{m_s}(R_a-R_b)}{(a_s+1)^2}.
 \label{eq:app_variance_reduction}
\end{equation}
Indeed, the right side is the variance of the conditional-mean change in~\eqref{eq:app_posterior_update}. It supplies the paired information score used when one label cannot yet change the posterior winner.

\subsection{Perfect-information upper bound}
Let $j$ maximize the current posterior means, with deterministic tie handling, and set $d_a=\mu_j-\mu_a\ge0$ and $v_a=\operatorname{Var}(A_a-A_j\mid\mathcal F)$. For $X_a=A_a-A_j$, the identity $x_+=(|x|+x)/2$ and Cauchy--Schwarz imply
\begin{equation}
 \mathbb E[(X_a)_+\mid\mathcal F]
 \le\frac{\sqrt{d_a^2+v_a}-d_a}{2}.
 \label{eq:app_positive_part}
\end{equation}
Since $X_j=0$, $\max_a X_a\le\sum_{a\ne j}(X_a)_+$. Also $\max_a A_a\le1$. These inequalities prove
\begin{equation}
 \mathbb E[\max_a A_a-A_j\mid\mathcal F]
 \le\min\left\{1-\mu_j,\sum_{a\ne j}
 \frac{\sqrt{d_a^2+v_a}-d_a}{2}\right\}=U,
 \label{eq:app_EVPI}
\end{equation}
as stated in Theorem~\ref{thm:expanded_stop}.

For any future query policy and eventual recommendation $J$, the pointwise inequality $A_J\le\max_a A_a$ bounds its conditional expected improvement by $U$. If its first further label costs $c\ge U$ in the same expected-accuracy units and later costs are nonnegative, its net expected improvement is nonpositive. If one decision serves $L_d$ equally weighted future deployments and label cost is in aggregate utility units, the corresponding condition is $L_dU\le c$. This is a sufficient cost-based stopping condition.

The sum in~\eqref{eq:app_EVPI} includes all unique candidates. Removing exact copies preserves the maximum; repeatedly including a nonleading copy would only loosen the sum. A smaller competitive set may be used for acquisition calculations after establishing reachability, while the displayed perfect-information bound is evaluated over the full unique set.

\subsection{Finite-budget reachable winners}
\label{app:reachable_winners}
Suppose at most $B$ additional labels are allowed, and at most $b_s$ remain queryable in stratum $s$. Take $b_s=\infty$ for independent unlimited streams. For a rival $a$ and feasible integer allocation $n$, its largest possible future posterior mean advantage over current leader $j$ is
\begin{equation}
 F_a(n)=\sum_{s\in\mathcal I_{aj}}w_s
 \frac{\alpha_{s,p_{sa}}-\alpha_{s,p_{sj}}+n_s}{a_s+n_s},
 \quad 0\le n_s\le b_s,\quad\sum_s n_s\le B.
 \label{eq:app_reachable}
\end{equation}
Every new label in a disagreement stratum is assigned to the rival's predicted class to maximize the numerator. Agreement strata always contribute zero. Positivity of the Dirichlet parameters makes every such finite label history have positive posterior predictive probability.

\begin{lemma}[Exact finite-horizon no-switch criterion]
All feasible further query policies have zero expected accuracy improvement if and only if $\max_nF_a(n)\le0$ for every rival $a$.
\end{lemma}
\begin{proof}
If all these maxima are nonpositive, $j$ remains a posterior-mean maximizer after every feasible history. Its conditional expected accuracy satisfies the tower property, so revealing additional labels cannot raise the expected posterior optimum.

Conversely, if a feasible allocation makes $F_a(n)>0$, use that fixed query schedule. The favorable finite history has positive probability and gives a rival strictly greater posterior mean than $j$. Write $\mu_a'=\mathbb E[A_a\mid\mathcal F']$ for the posterior after this schedule. Then
\[
 \mathbb E[\max_a\mu_a'\mid\mathcal F]-\mu_j
 =\mathbb E[\max_a\mu_a'-\mu_j'\mid\mathcal F]>0.
\]
The integrand is nonnegative on every history and strictly positive on the favorable event. This proves positive information value when labels have zero cost; a positive information value need not exceed their acquisition cost.
\end{proof}

An exact integer solution is available by dynamic programming. Let $f_s(n)$ be the stratum contribution in~\eqref{eq:app_reachable}, $D_0(0)=0$ and $D_0(b)=-\infty$ for $b>0$. Then
\[
 D_s(b)=\max_{0\le n\le\min\{b,b_s\}}
 \{D_{s-1}(b-n)+f_s(n)\}.
\]
Maximizing $D_S(b)$ over $b\le B$ costs $O(SB^2)$ per rival. A simpler equivalent implementation exploits decreasing marginal returns. If $d_s=m_{s,p_{sa}}-m_{s,p_{sj}}$, the gain from the $n$th favorable label is
\begin{equation}
 f_s(n)-f_s(n-1)
 =\frac{w_s(1-d_s)a_s}{(a_s+n)(a_s+n-1)}.
 \label{eq:app_reachable_marginal}
\end{equation}
These gains are nonnegative and nonincreasing. The $B$ largest gains across the capped stratum sequences therefore solve the allocation: each selected gain has predecessors at least as large, so all prefix constraints can be satisfied. This is the competitive-set computation associated with Equation~\eqref{eq:expanded_reachable}.

A rival can require several favorable labels before overtaking $j$. Thus a zero one-step value does not imply zero remaining multi-label value. Reachability separates these cases; Equation~\eqref{eq:app_variance_reduction} then measures information gained before the winner changes. In numerical implementations, a tolerance $\epsilon$ on the maximal reachable advantage gives an $\epsilon$ bound on remaining expected accuracy value in exact arithmetic. Floating-point roundoff is an additional numerical approximation.

\subsection{Population and realized-block uncertainty}
\label{app:posterior_finite}
Suppose a target block has $M_s\ge1$ unseen conditionally independent labels in stratum $s$. Its realized candidate accuracy is
$\widetilde A_a=\sum_s w_sM_s^{-1}\sum_{i=1}^{M_s}R_{s,i,a}$.
The mean remains $\mu_a$, while total covariance gives
\begin{equation}
 \operatorname{Cov}(\widetilde A_a,\widetilde A_b\mid\mathcal F)
 =\sum_s w_s^2\frac{a_s+M_s}{M_s(a_s+1)}
 \left[\mathbf1\{p_{sa}=p_{sb}\}m_{s,p_{sa}}
 -m_{s,p_{sa}}m_{s,p_{sb}}\right].
 \label{eq:app_finite_cov}
\end{equation}
The expected conditional categorical covariance equals $a_s/(a_s+1)$ times the bracket; averaging $M_s$ labels divides it by $M_s$. The posterior covariance of the conditional mean adds $1/(a_s+1)$ times the bracket, yielding the displayed factor. Applying~\eqref{eq:app_positive_part} with this covariance bounds the hindsight finite-block oracle instead of the population oracle.

For example, two opposite binary predictions under a known balanced label law both have population accuracy $1/2$, giving zero population perfect-information value. On one realized target item, one prediction is always correct, so a hindsight selector has expected advantage $1/2$. The two oracle values answer different questions. The experiment tables identify their reference unit, and the posterior stopping rule uses conditional expected target accuracy.
\space
% END inlined sections/expanded_joint_feedback.tex
% BEGIN inlined sections/expanded_finite_budget.tex
\section{Finite-budget selection models}
\label{app:finite_budget}
A finite budget determines how often an available intervention is selected, even when its deployment risk is fixed. This section derives exact selection probabilities for a joint Gaussian experiment and a finite categorical query pool. The first model isolates gaps, correlation and evidence bias. The second conditions on actual neural predictions and keeps the discrete sampling law and tie convention. Both calculate expected deployment loss by combining selection probabilities with the corresponding candidate risks.

\subsection{Binary risk, evidence alignment and concentration}
Let action 0 have deployment risk $r_0$, and let action 1 have risk $r_1=r_0+\Delta$, where $\Delta\ge0$. Suppose the estimated risk difference after $m$ labels follows
\begin{equation}
 \widehat d_m\sim\mathcal N(\Delta+b,\sigma^2/m).
\end{equation}
The bias difference $b$ concerns the feedback experiment; the deployment risk gap remains $\Delta$. For $\sigma>0$, selection of action 1 has probability
$\Phi[-\sqrt m(\Delta+b)/\sigma]$. A tie has probability zero. Selection regret is zero for action 0 and $\Delta$ for action 1, giving Equation~\eqref{eq:expanded_gaussian}.

\begin{proposition}[Direction of the budget response]
For $\sigma>0$ and $\Delta>0$, Gaussian selection regret decreases strictly with $m$ when $\Delta+b>0$, increases strictly when $\Delta+b<0$, and equals $\Delta/2$ for every $m$ when $\Delta+b=0$.
\end{proposition}
\begin{proof}
Let $\phi=\Phi'$ denote the standard normal density. Extend $m$ to positive real values and differentiate:
\begin{align}
 \frac{\partial\mathcal R_m}{\partial m}
 &=-\frac{\Delta(\Delta+b)}{2\sigma\sqrt m}
 \phi\!\left(\frac{\sqrt m(\Delta+b)}{\sigma}\right),\\
 \frac{\partial\mathcal R_m}{\partial b}
 &=-\frac{\Delta\sqrt m}{\sigma}
 \phi\!\left(\frac{\sqrt m(\Delta+b)}{\sigma}\right).
\end{align}
The normal density is positive at every finite argument, establishing the sign. Restricting back to integer budgets preserves the monotonicity. At a zero mean, symmetry gives probability $1/2$.
\end{proof}

When $\sigma=0$, the evidence difference is deterministic. Regret is $\Delta$ if $\Delta+b<0$ and zero if $\Delta+b>0$. At equality, it is $\Delta$ multiplied by the probability assigned to action 1 by the specified tie rule. If $\Delta=0$, all decisions have zero regret. These cases matter for duplicated candidates and for noise whose common component cancels completely.

The result separates the amount of feedback from its ordering. Increasing $m$ concentrates the estimator on its mean. The resulting deployment value depends on the sign of that mean relative to the true risk gap. This same mechanism will appear below when a realized query pool and evaluation pool favor different candidates.

\subsection{A solvable joint candidate experiment}
For candidates $j=0,\ldots,K-1$, let $r_j$ be fixed deployment risks and $b_j$ be fixed feedback biases. One joint observation is
\begin{equation}
 X_{tj}=r_j+b_j+s\left(\sqrt\rho Z_{t0}
       +\sqrt{1-\rho}Z_{t,g(j)}\right),
 \qquad t=1,\ldots,m,
 \label{eq:app_common_factor}
\end{equation}
where $s\ge0$, $0\le\rho\le1$, and all common and group-specific standard normals are independent across distinct indices and times. Candidates in the same group are exact copies with equal $r_j,b_j$ and the same residual draw. The common-factor index is distinct from every group index. We minimize the empirical risk score and break ties by the first candidate index.

Writing $\nu_j=r_j+b_j$ and $\tau=s\sqrt{(1-\rho)/m}$, the empirical scores have the joint law
\begin{equation}
 \overline X_j=\nu_j+\frac{s\sqrt\rho}{\sqrt m}Z_0+\tau Z_{g(j)}.
\end{equation}
The common term cancels from all comparisons. Let $\mathcal G$ contain the smallest-index representative of each exact-copy group.

\begin{proposition}[Joint winner probabilities]
For $\tau>0$, representative $j\in\mathcal G$ wins with probability
\begin{equation}
 P_j=\int_{-\infty}^{\infty}\phi(z)
 \prod_{\ell\in\mathcal G\setminus\{j\}}
 \Phi\!\left(\frac{\nu_\ell-\nu_j}{\tau}-z\right)\,dz.
 \label{eq:app_winner_integral}
\end{equation}
Nonrepresentative copies have probability zero under first-index ties. Expected deployed risk is $\sum_j P_jr_j$ and regret is $\sum_jP_j(r_j-\min_\ell r_\ell)$.
\end{proposition}
\begin{proof}
Condition on the representative residual $Z_{g(j)}=z$. Candidate $j$ wins if, for every other group $\ell$,
$Z_{g(\ell)}>z+(\nu_j-\nu_\ell)/\tau$.
The residuals are conditionally independent and continuous. Their joint conditional probability is the product in Equation~\eqref{eq:app_winner_integral}; integration over $z$ proves the result. Exact copies always tie, assigning their common win to the first representative. The risk formulas follow by total expectation.
\end{proof}

For $\tau=0$, the winner is the first index minimizing $\nu_j$. For two distinct groups, the single-observation difference variance is $\sigma^2=2s^2(1-\rho)$, recovering the binary expression. A correlation increase reduces this difference variance at fixed marginal variance. It improves concentration around a correct ordering and strengthens concentration around an incorrect ordering.

\paragraph{Redundancy and complementarity.}
Appending exact copies changes neither the group competition nor the deployed risk. Complementarity can increase both oracle capacity and selection loss. Consider two equally likely regimes, an incumbent with risk $(0.3,0.3)$, and a specialist with risk $(0.05,0.3)$. A four-candidate redundant portfolio contains three copies of that specialist. A complementary four-candidate portfolio contains the incumbent, the original specialist, a second specialist with risks $(0.3,0.05)$ and a generalist with risks $(0.2,0.2)$. With $s=0.3$, $\rho=0.5$ and zero bias, the oracle risks are $0.175$ and $0.05$. At $m=1$, the expected deployed risks are $0.20029$ and $0.14021$. The complementary portfolio has better deployed risk and larger selection regret, $0.09021$ against $0.02529$. At $m=16$, deployed risks approach $0.17505$ and $0.05355$. Reporting both quantities shows how additional useful capacity can create a larger evidence requirement while still improving the system.

\begin{figure}[!htb]
\centering
\includegraphics[width=.88\linewidth]{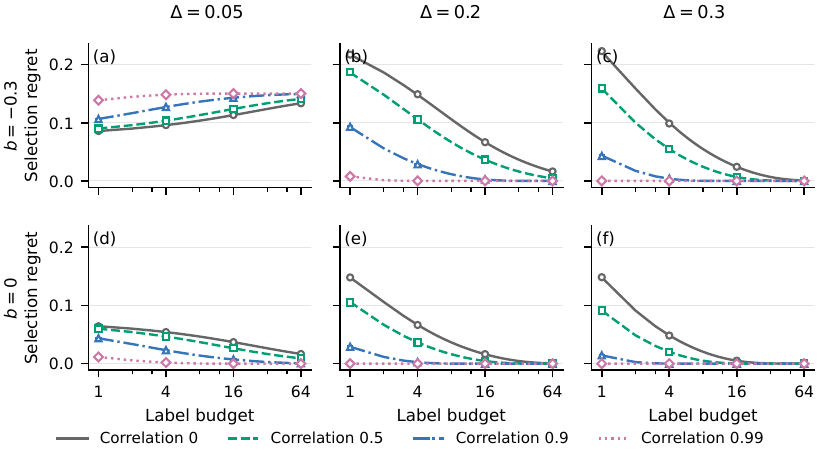}
\caption{Four-candidate selection under varying accuracy gaps (columns), query bias (rows) and error correlation $\rho$ (curves). Lines integrate Gaussian winner probabilities. Markers and 95\% error bars use 100,000 independent score vectors per condition. Canceling shared noise strengthens the query ranking, including when bias makes that ranking wrong. Appendix~\ref{app:finite_budget} gives parameters and neural-network validation.}
\label{fig:expanded_gaussian}
\end{figure}

\subsection{Numerical solution and independent simulation}
The deterministic solver evaluates Equation~\eqref{eq:app_winner_integral} in standard-normal coordinates. It computes products of cumulative distribution functions (CDFs) by summing their logarithms. Truncation to $[-12,12]$ discards at most $2\Phi(-12)<3.6\times10^{-33}$ probability per integral, because the integrand is bounded by $\phi(z)$. Adaptive quadrature splits the interval at zero and at clipped transition points between candidate means. The winner probabilities are checked for summing to one without renormalization.

The simulation generates raw joint Gaussian score vectors and applies the actual first-index argmin. It neither reads winner probabilities nor samples winners from the theoretical distribution. Gaussian closure permits direct generation of the empirical mean vector for a budget $m$. The common input to the two programs is the declared parameter grid.

The main factorial contains three risk gaps, four correlations, three bias settings and four budgets, giving 144 conditions. Additional binary, fixed-count portfolio, duplicate and deterministic cases bring the total to 197. Every condition uses 100,000 independent score vectors. Across 740 winner-count comparisons, exact binomial tests use a Bonferroni threshold of $0.01/740$. No comparison rejects. The maximum absolute probability discrepancy is $0.00412$, and the largest regret discrepancy is 2.58 Monte Carlo standard errors. The binary integral matches its closed form within $5.6\times10^{-17}$. These checks concern the solved common-factor experiment, with separate theoretical curves and sampled decisions.

\subsection{Exact finite-pool selection probabilities}
For a frozen pair of neural predictors, each eligible query label reveals a signed accuracy difference $D\in\{-1,0,1\}$. Let $n_+,n_-,n_0$ denote the counts in a fixed pool of size $N$. A uniform subset of $m\le N$ examples has counts $K_+,K_-,K_0$ with multivariate hypergeometric probability in Equation~\eqref{eq:expanded_hypergeom}. Terms outside $0\le K_c\le n_c$ or $K_++K_-\le m$ have probability zero. The score $S_m=K_+-K_-$ determines the empirical rule.

\begin{equation}
 \Pr(K_+=u,K_-=v)=
 \frac{\binom{n_+}{u}\binom{n_-}{v}\binom{n_0}{m-u-v}}
 {\binom{n_++n_-+n_0}{m}}.
 \label{eq:expanded_hypergeom}
\end{equation}

For strict incumbent-default selection, the candidate is chosen if $S_m>0$; candidate-default ties change this to $S_m\ge0$. For a threshold $h$, the strict rule requires $S_m\ge\lfloor mh\rfloor+1$. Let $p_m$ be the corresponding summed hypergeometric probability, and let $e_I,e_C$ be the fixed evaluation error counts. Then
\begin{align}
 \mathbb E[e_{\mathrm{selected}}]&=(1-p_m)e_I+p_me_C,\\
 \mathbb E[\text{gain}]&=\frac{100p_m(e_I-e_C)}{|E|},\\
 \mathbb E[\text{oracle regret}]&=
 \frac{100[(1-p_m)e_I+p_me_C-\min(e_I,e_C)]}{|E|}.
 \label{eq:app_exact_deployment}
\end{align}
Aggregating error counts and evaluation sizes across batches yields the stream metrics. This exact calculation conditions on the entire frozen pool and is used by the evaluator. The running selector sees acquired labels alone.

\paragraph{Finite-population Gaussian approximation.}
The pool mean and variance are
\begin{equation}
 \bar d=\frac{n_+-n_-}{N},\qquad
 v=\frac{n_++n_-}{N}-\bar d^2.
\end{equation}
For $N>1$, uniform sampling without replacement gives
\begin{equation}
 \mathbb E S_m=m\bar d,\qquad
 \operatorname{Var}(S_m)=m\frac{N-m}{N-1}v.
\end{equation}
A normal tail approximates the probability of crossing the signed-count threshold. The calculation uses the midpoint between adjacent attainable support values as its continuity correction; zero variance and thresholds outside support are handled exactly. The support can have spacing two when only rescue and harm outcomes are present, so a universal half-unit correction is inappropriate. The exact hypergeometric law remains the reference at small budgets.

\paragraph{Query-to-evaluation alignment.}
For a query pool $P\subseteq Q$, write
\begin{equation}
 \bar d_P-\bar d_E=(\bar d_P-\bar d_Q)+(\bar d_Q-\bar d_E).
 \label{eq:app_pool_mismatch}
\end{equation}
The first term describes the realized effect of restricting the acquisition pool; the second records the realized difference between the permanent query and evaluation pools. Both terms can affect a batch decision. Even uniform querying has a finite realized $Q/E$ discrepancy. The sign groups below are evaluator-defined conditional diagnostics, not strata available from unlabeled inputs.

\subsection{Neural prediction validation}
We retain all 655 PACS batches from five source-checkpoint stream orders. Two pools are evaluated: the complete eligible query pool, and its label-free top half ranked by symmetric KL. Both sample uniformly without replacement within the specified pool. The grid includes label caps $1,3,5,10$ and both tie defaults. Actual usage is $\min(m,|P|)$: seven top-half pools contain fewer than ten items and remain in the evaluation.

The exact solver covers 10,480 conditions. The independent simulator samples actual subsets, sums their signed outcomes and applies the rule, giving 52.4 million subset draws. All 10,480 winner-frequency tests pass a simultaneous 1\% Bonferroni test family. The largest aggregate difference is 1.99 Monte Carlo standard errors. Exhaustive enumeration independently verifies another 2,310 small-population configurations. The Gaussian approximation has mean absolute probability error $0.01039$ and maximum error $0.17174$.

For incumbent-default ties, uniform-query regret on aligned pools falls from 11.35 to 2.62 percentage points as the cap grows from one to ten. On opposed pools it rises from 4.39 to 6.45 points. The corresponding top-half-KL curves change from 9.01 to 1.26 and from 4.50 to 7.02 points. With candidate-default ties, the strong fixed branch already supplies a high starting accuracy; intermediate-budget curves need not be monotone. Fixed BN+flip remains above the aggregate selected gain for every displayed budget. The opposite conditional trends describe how the same concentration mechanism operates within the observed deployment stream.

\begin{figure}[t]
\centering
\includegraphics[width=\linewidth]{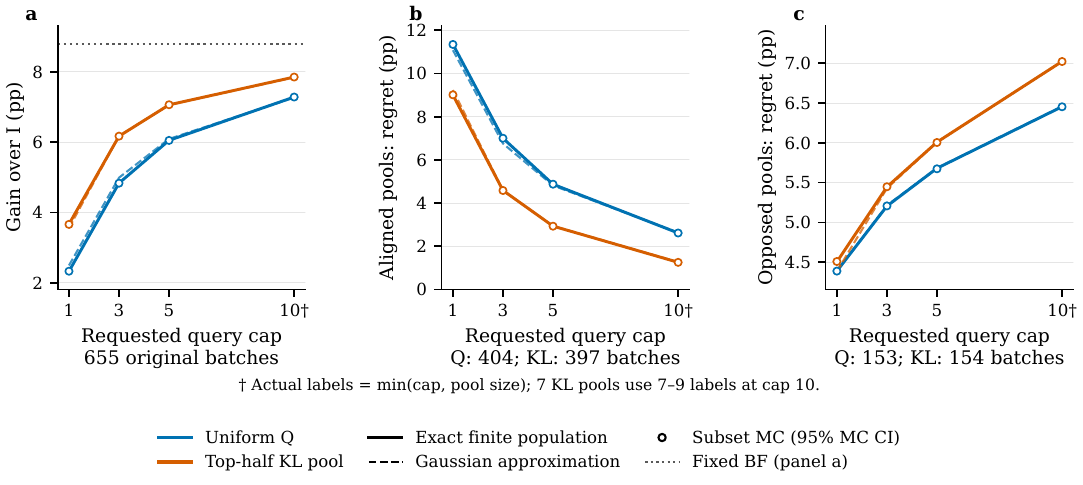}
\caption{Exact finite-budget predictions on all 655 PACS batches with incumbent-default ties. Solid lines use the hypergeometric law; dashed lines use its finite-population Gaussian approximation; points and 95\% error bars use independent subset simulation. The left panel evaluates the same full population for both pools. The middle and right panels condition on agreement or opposition between realized pool and evaluation advantage signs. Uniform and top-half-KL pools have 404/397 aligned and 153/154 opposed batches, respectively. Seven top-half pools use fewer than ten labels at cap ten. Error bars measure simulation error.}
\end{figure}
\begin{figure}[t]
\centering
\includegraphics[width=\linewidth]{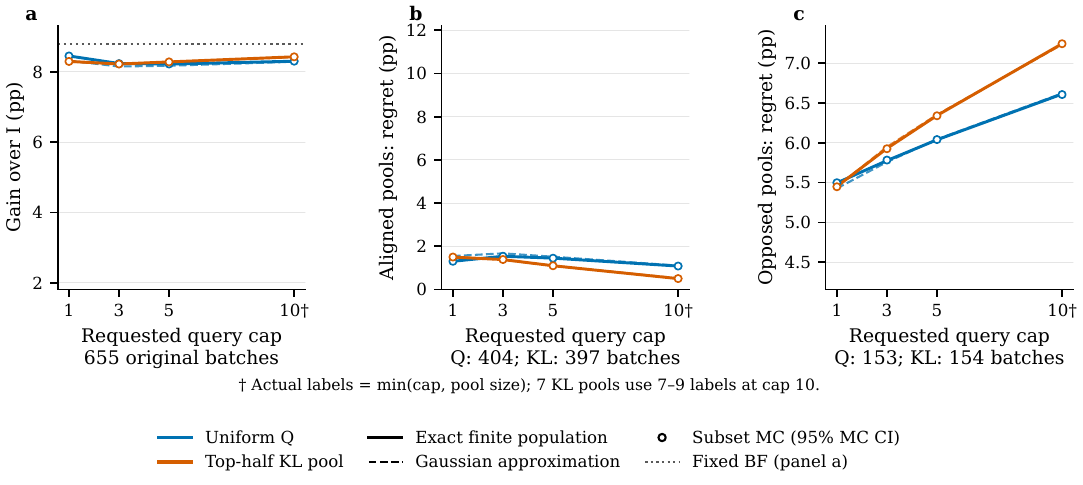}
\caption{The same finite-budget experiment with BN+flip selected on ties. All-batch gain uses all 20,775 evaluation occurrences across five orders; conditional panels retain their pool-specific aligned and opposed groups. The strong fixed default changes the small-budget response while leaving the exact sampling calculation unchanged. Curves and uncertainty follow the preceding figure.}
\end{figure}
\space
% END inlined sections/expanded_finite_budget.tex
% BEGIN inlined sections/expanded_state.tex
\section{Persistent state and future deployment loss}
\label{app:expanded_state}
Retaining an update changes future predictions. Undoing it after the current prediction restores the previous state. This section tracks the difference between these two choices under identical future data and updates. We derive exact recurrences for the residual classifier and feature cache, then bound their effects on future losses.

\subsection{A general common-update bound}
Let two branches leave a decision with states $S^a$ and $S^I$, whose distance is $e=\|S^a-S^I\|$. Both subsequently use the same update maps $U_j$. Assume $\|U_j(s)-U_j(s')\|\le L_j\|s-s'\|$ and that the scored loss at future step $k$ is $L_{\ell,k}$-Lipschitz in state. Repeated application gives
\begin{equation}
 \|S_k^a-S_k^I\|\le e\prod_{j<k}L_j,
 \qquad
 |\ell_k(S_k^a)-\ell_k(S_k^I)|
 \le L_{\ell,k}e\prod_{j<k}L_j.
\end{equation}
Let $\xi(H)=\sum_{k=0}^{H-1}[\ell_k(S_k^a)-\ell_k(S_k^I)]$ be the cumulative loss difference over $H$ future steps. Summing the preceding bound gives
\begin{equation}
 |\xi(H)|\le e\sum_{k=0}^{H-1}L_{\ell,k}\prod_{j<k}L_j.
 \label{eq:app_general_continuation}
\end{equation}
For constants $L_j=L_U$ and $L_{\ell,k}=L_\ell$, this becomes $L_\ell e\sum_{k=0}^{H-1}L_U^k$. The common update maps include their future data and labels; changing those inputs between branches changes the comparison. The experiments implement shared inputs explicitly, so the branch-state difference is the manipulated variable.

\subsection{Residual-head dynamics}
Let $d$ be the feature dimension and $C$ the number of classes. The rows of $X_s\in\mathbb R^{q\times d}$ are the unit-normalized features of $q=3$ queried examples. Their targets $R_s$ subtract the fixed source probabilities from the one-hot class labels. The residual classifier has weights $\Theta\in\mathbb R^{d\times C}$. With learning rate $\eta$ and regularization strength $\lambda$, a squared-loss update is
\begin{align}
 U_s(\Theta)&=M_s\Theta+B_s,\\
 M_s&=(1-\eta\lambda)I-\eta X_s^\top X_s/q,\qquad
 B_s=\eta X_s^\top R_s/q.
 \label{eq:app_head_update}
\end{align}
The implemented values are $\eta=0.5$ and $\lambda=0.1$. Fixed-source targets make the affine term identical on both branches. Subtraction therefore gives the exact recurrence
\begin{equation}
 \Delta\Theta_{s+1}=M_s\Delta\Theta_s.
 \label{eq:app_head_difference}
\end{equation}

\begin{proposition}[Head contraction and projected output]
For unit-norm rows of $X_s$ and the specified parameters,
$\|\Delta\Theta_{s+1}\|_F\le0.95\|\Delta\Theta_s\|_F$.
For any unit-norm future evaluation feature $x$, the logit difference after $k$ common updates is
\begin{equation}
 d_k=xM_{t+k-1}\cdots M_t\Delta\Theta_t,
 \qquad \|d_k\|_2\le0.95^k\|\Delta\Theta_t\|_F.
\end{equation}
\end{proposition}
\begin{proof}
The matrix $X_s^\top X_s/q$ is positive semidefinite and has trace one. Its eigenvalues belong to $[0,1]$. Equation~\eqref{eq:app_head_update} gives eigenvalues of $M_s$ in $[0.45,0.95]$, so its spectral norm is at most $0.95$. Submultiplicativity proves contraction and the norm envelope. Iterating the exact difference recurrence and applying the linear readout gives the displayed $d_k$.
\end{proof}
The product tracks the direction of the update difference. A large parameter difference can have a small future output effect when future features have little projection onto that direction. The implementation applies the rank-$q$ update to the difference without constructing a dense $d\times d$ matrix. Independently updating both explicit head states verifies the recurrence.

\subsection{Feature-centering cache}
Let $m_s$ be the stored feature mean, $\overline f_{Q,s}$ the mean of the selected query features, and $\rho$ the weight retained from the previous mean. The cache update is
\begin{equation}
 m_{s+1}=\rho m_s+(1-\rho)\overline f_{Q,s},\qquad\rho=0.9,
\end{equation}
The classifier subtracts $\alpha m_sW^\top$ from its logits, where $W$ is the fixed source classifier and $\alpha=0.05$ controls the correction strength. The two model copies use the same query features. Their differences obey
\begin{equation}
 \Delta m_{s+1}=\rho\Delta m_s,\qquad
 d_k=-\alpha\rho^k\Delta m_tW^\top.
 \label{eq:app_cache_difference}
\end{equation}
Both identities follow by subtracting the shared affine updates. The continuation scale is $\alpha\|\Delta m_tW^\top\|_2$, the classifier-visible state difference. A component in the null space of $W$ has zero prediction effect at every horizon, despite contributing to the feature-state distance.

The cache operates at the classifier input. Its state is an actual persistent feature mean, distinct from the internal running statistics of backbone batch-normalization layers. Both this cache and the residual head change future logits while the backbone stays fixed.

\subsection{Directional cross-entropy bounds}
For $\ell_y(z)=\log\sum_c e^{z_c}-z_y$, the gradient is $p-e_y$ and the Hessian is $\operatorname{diag}(p)-pp^\top$. For any $v$, its quadratic form is the variance of the coordinates $v_c$ under $c\sim p$. The range bound for a bounded variable gives
\begin{equation}
 v^\top\nabla^2\ell_y(z)v
 \le\frac{(\max_c v_c-\min_c v_c)^2}{4}
 \le\frac{\|v\|_2^2}{2}.
\end{equation}
The last step uses $(v_i-v_j)^2\le2(v_i^2+v_j^2)$. The Hessian is positive semidefinite and has operator norm at most $1/2$.

Taylor's theorem with integral remainder now gives, for every $z,d,y$,
\begin{align}
 \ell_y(z+d)-\ell_y(z)
 &=(p-e_y)^\top d+
 \int_0^1(1-u)d^\top\nabla^2\ell_y(z+ud)d\,du,\\
 0&\le\ell_y(z+d)-\ell_y(z)-(p-e_y)^\top d
 \le\tfrac14\|d\|_2^2.
 \label{eq:app_ce_interval}
\end{align}
Summing over future evaluation examples yields the interval tested in the state experiments. The propagated $d$ comes from the known update operators and initial state difference. Evaluation labels determine the directional gradient, making this a retrospective mechanism calculation. It can distinguish positive and negative loss effects that a magnitude bound treats identically.

A label-free envelope follows from $\|p-e_y\|_2\le\sqrt2$. For one mean-loss score per future batch, head continuation satisfies
\begin{equation}
 \sum_{k=0}^{H-1}|\Delta\overline\ell_k|
 \le\sqrt2\|\Delta\Theta_t\|_F\sum_{k=0}^{H-1}0.95^k.
\end{equation}
The cache envelope replaces the initial norm by $\alpha\|\Delta m_tW^\top\|_2$ and $0.95$ by $\rho$. For image-sum rather than mean-batch losses, each term is multiplied by its evaluation batch size.

\subsection{Prediction changes and zero-one loss}
Let $m(z)$ be the difference between the largest and second-largest baseline logits. Each logit moves by at most $\|d\|_\infty$, so the winning class cannot change if $m(z)>2\|d\|_\infty$. For a deterministic shared tie rule,
\begin{equation}
 \left|\mathbf1\{\arg\max(z+d)\ne y\}
       -\mathbf1\{\arg\max z\ne y\}\right|
 \le\mathbf1\{m(z)\le2\|d\|_\infty\}.
 \label{eq:app_margin}
\end{equation}
The bound concerns the set where a prediction could change. Corrected errors, newly introduced errors and neutral prediction changes are measured separately. Contraction of state or cross-entropy does not fix the sign of an accuracy change, because the decision boundary depends on individual margins and perturbation directions.

\subsection{Paired update diagnostic}
An update can change both the current choice and future predictions. We separate these effects with two model copies. One keeps the current update; the other undoes it. Both then receive the same inputs, labels and updates. We test a learned linear classifier and a cache that adjusts scores using a running feature mean.

\Needspace{6\baselineskip}
The parameter difference contracts by at most $0.95$ per classifier step and $0.9$ per cache step, predicting future logit differences $d$. For baseline logits $z$, let $p=\operatorname{softmax}(z)$, $e_y$ be the true-class unit vector, and $\ell_y(z)=\log\sum_c e^{z_c}-z_y$ be cross-entropy. Then
\begin{equation}
 (p-e_y)^\top d\le\ell_y(z+d)-\ell_y(z)
 \le(p-e_y)^\top d+\tfrac14\|d\|_2^2.
 \label{eq:expanded_state}
\end{equation}
Define the loss change $\Delta\mathrm{CE}$, lower bound $L=(p-e_y)^\top d$ and width $R=\tfrac14\|d\|_2^2$. For $R>0$, $(\Delta\mathrm{CE}-L)/R\in[0,1]$ (Figure~\ref{fig:expanded_state}).

Retaining an update improves current accuracy by 0.23 points and changes next-batch accuracy by $-0.20$ points. Undoing it leaves future predictions unchanged. Immediate gains and lasting effects need separate measurement.

\begin{figure}[!htb]
\centering
\includegraphics[width=.85\linewidth]{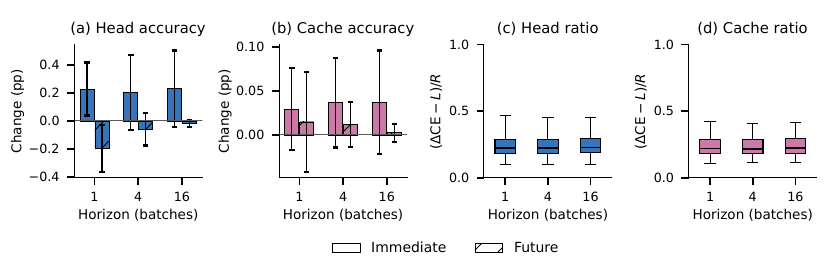}
\caption{Retain versus undo an update. (a,b) Immediate and future accuracy changes with order SD. (c,d) Normalized loss changes $(\Delta\mathrm{CE}-L)/R$: quartiles, medians and full-range whiskers within the predicted interval $[0,1]$.}
\label{fig:expanded_state}
\end{figure}

\subsection{Update timing and experimental verification}
At a chosen batch $t$, we score an intervention and either retain or undo its update. We call this the branch point. Both copies then receive the same future query features, residual targets, and update order. Each future batch is scored before its update. Thus $H$ future scored batches contain $H-1$ common updates before the final score. The baseline copy skips only the update at $t$. The reference stream continues its prescribed classifier and cache updates, resetting at domain boundaries.

Each of five complete PACS orders acquires 393 distinct query labels. The model copies use these same labels. Horizons $H=1,4,16$ retain 128,119,83 valid branch points per order. Their future windows overlap. The means summarize paired update effects, and the stream-order standard deviations describe sensitivity to arrival order. Immediate means vary with the horizon because the eligible branch points differ.

Source logits reconstructed from the released classifier agree with the stored predictions within $1.1\times10^{-6}$. Explicit branch updates and propagated state differences agree within $1.3\times10^{-15}$. All cross-entropy observations satisfy Equation~\eqref{eq:app_ce_interval}, and all hard-prediction changes satisfy Equation~\eqref{eq:app_margin}. Changing evaluation labels leaves learned state unchanged. Output rollback produces exactly zero future logit and loss difference.

The residual-head intervention improves immediate accuracy by 0.23 percentage points at one-batch branch points, while next-batch accuracy changes by $-0.20$ points and next-batch cross-entropy by $-0.00059$. Immediate and future cross-entropy changes have opposite signs in 20.16\% of one-batch comparisons. The feature cache produces smaller persistent changes. Together these observations establish that a local intervention's immediate benefit, future accuracy and future cross-entropy are distinct quantities, connected by the update operator and the direction of the parameter change.
\space
% END inlined sections/expanded_state.tex
\newcommand{\EmpiricalRoot}{appendix_materials/empirical}
\newcommand{\EmpJoint}{\PORTAL{}}
% BEGIN inlined sections/expanded_appendix_empirical.tex
\providecommand{\EmpiricalRoot}{.}
\providecommand{\EmpJoint}{Joint selection}

\section{Experimental populations, feedback and outputs}
\label{app:empirical-populations}

The experiments measure fixed intervention capacity, its realization through output selection, and its expansion through learned states. Each comparison retains its model, evaluation population, output timing and annotation budget.

\paragraph{Corruption populations.}
CIFAR10-C and CIFAR100-C use all 15 severity-5 corruptions and three source ResNet18 checkpoints. Each corruption contains 10,000 images. Batches contain 64 images, including a final batch of 16. The total is 150,000 images per dataset and checkpoint. A permanent split assigns 5,000 base-image indices to the query population $Q$ and 5,000 to the evaluation population $E$. The role of an underlying image is unchanged across corruptions and checkpoints. Each corruption is a separate deployment: selection statistics and annotation history reset at its boundary. An annotation purchased in one deployment is charged again if queried in another.

\paragraph{Natural domain shifts.}
The CLIP ViT-L/14 experiments include all four domains of PACS, OfficeHome and VLCS, with 9,991, 15,588 and 7,510 image instances, respectively. Their permanent query/evaluation counts are 5,070/4,921, 7,801/7,787 and 3,751/3,759. The split is assigned to encoded-content groups, so repeated occurrences remain on the same side. All dataset instances and their original annotations are retained. OfficeHome has 15,176 content groups, including 87 groups whose 178 occurrences have conflicting directory labels. VLCS contains 7,506 groups and no conflicting labels. PACS has no repeated content in this enumeration. Annotation costs count queried instances. PACS and VLCS identities are recorded during inference. OfficeHome identities follow the deterministic class/file order of its saved four-readout stream and are checked against its stored labels.

\paragraph{Timing and aggregation.}
Fixed whole-dataset scores include every example and use zero target labels. Controlled comparisons score the common $E$ population after acquiring current-batch labels from $Q$. These labels choose among predictions already formed. In the online-classifier experiment, they train the classifier after current predictions have been scored. That update affects later batches. CIFAR scores pool equal-size corruption populations. Vision-language model (VLM) results report both equal-domain and image-pooled accuracy, with the equal-domain mean as primary. Standard deviations cover three CIFAR checkpoints or five natural-domain stream orders, as indicated in each table.

\paragraph{Query rules and defaults.}
Empirical selection minimizes queried classification errors and keeps the designated default on ties. Uniform querying samples three available $Q$ examples. Disagreement-first querying prioritizes inputs with different candidate class predictions, in random order. Symmetric-KL querying ranks disagreement between predicted class probabilities. The complete CIFAR experiment sums the six pairwise scores; the four-readout CLIP experiment uses their mean. Both give the same ranking for a fixed candidate set. The StatA and convex-readout experiments use the largest pairwise score.

The fixed default is BN+flip on CIFAR and the eight-view readout on CLIP. \PORTAL{} uses total prior mass 1: it assigns 0.75 to candidate-predicted classes and spreads the remaining 0.25 over all classes (Equation~\eqref{eq:app_prior}). The label price is 0.002, with at most 12 queries per batch. Three query credits arrive per batch. CIFAR caps the stored balance at 12; the VLM experiments carry all unused credits. Tables report actual label use.

\section{Complete corruption results}
\label{app:empirical-cifar}

\subsection{Capacity of fixed normalization and view interventions}

The four CIFAR candidates share one source network. I uses source normalization statistics; B uses current-batch statistics. F averages the original and horizontally flipped source logits. BF averages these two views using batch statistics. B and BF leave the source normalization buffers unchanged. The four candidates require four forward passes, with each view shared by the relevant readouts. A fixed I or B output requires one forward pass, and F or BF requires two.

Tables~\ref{tab:emp-cifar10-fixed} and~\ref{tab:emp-cifar100-fixed} retain every corruption and checkpoint. BF gives 79.9218\% and 54.7544\% on CIFAR10-C and CIFAR100-C, respectively. Its fixed increments over I are 22.2178 and 21.5864 percentage points. These gains precede label acquisition. The B and F columns separate normalization from views, while local variation leaves room for different batch decisions.

Short names denote the standard noise, blur, elastic-transformation and JPEG-compression corruptions. All scores include the final 16 images.

% BEGIN inlined appendix_materials/empirical/tables/cifar10-fixed.tex
{
\small
\setlength{\tabcolsep}{2pt}\setlength{\LTpre}{3pt}\setlength{\LTpost}{4pt}
\setlength{\LTcapwidth}{\linewidth}
\begin{longtable}{@{}lrrrrrrrrrrrr@{}}
\caption{Complete CIFAR10-C fixed-intervention accuracy (\%). Each checkpoint evaluates all 10,000 images in every severity-5 corruption; all columns use zero feedback.}\label{tab:emp-cifar10-fixed}\\
\toprule
& \multicolumn{4}{c}{Checkpoint 0} & \multicolumn{4}{c}{Checkpoint 1} & \multicolumn{4}{c}{Checkpoint 2} \\
Corruption & I & B & F & BF & I & B & F & BF & I & B & F & BF \\
\midrule
\endfirsthead
\caption[]{Complete CIFAR10-C fixed-intervention accuracy (\%) (continued).}\\
\toprule
& \multicolumn{4}{c}{Checkpoint 0} & \multicolumn{4}{c}{Checkpoint 1} & \multicolumn{4}{c}{Checkpoint 2} \\
Corruption & I & B & F & BF & I & B & F & BF & I & B & F & BF \\
\midrule
\endhead
\bottomrule
\endfoot
Gaussian & 22.13 & 66.36 & 22.03 & \textbf{68.41} & 27.73 & 67.01 & 28.48 & \textbf{69.52} & 28.06 & 67.37 & 28.30 & \textbf{69.90} \\
Shot & 29.25 & 67.61 & 29.26 & \textbf{70.43} & 34.26 & 69.00 & 34.81 & \textbf{71.78} & 36.06 & 69.52 & 36.32 & \textbf{72.13} \\
Impulse & 22.53 & 58.86 & 22.97 & \textbf{62.09} & 22.02 & 60.92 & 22.04 & \textbf{63.55} & 29.56 & 60.85 & 30.38 & \textbf{63.89} \\
Defocus & 54.54 & 87.34 & 55.31 & \textbf{88.61} & 55.88 & 87.73 & 56.48 & \textbf{88.93} & 59.74 & 87.70 & 60.19 & \textbf{88.99} \\
Glass & 55.10 & 65.17 & 56.08 & \textbf{67.32} & 50.14 & 64.88 & 51.63 & \textbf{67.06} & 50.83 & 65.94 & 51.99 & \textbf{67.76} \\
Motion & 67.45 & 86.04 & 68.18 & \textbf{87.02} & 66.11 & 86.09 & 67.04 & \textbf{87.15} & 68.58 & 86.55 & 69.30 & \textbf{87.61} \\
Zoom & 64.36 & 86.64 & 65.63 & \textbf{88.18} & 62.50 & 87.64 & 63.83 & \textbf{88.68} & 67.54 & 87.97 & 68.14 & \textbf{88.54} \\
Snow & 78.24 & 80.88 & 79.09 & \textbf{82.19} & 77.52 & 81.56 & 78.22 & \textbf{83.28} & 78.14 & 81.86 & 78.92 & \textbf{83.68} \\
Frost & 69.93 & 79.57 & 71.02 & \textbf{80.98} & 65.31 & 80.72 & 66.59 & \textbf{81.68} & 65.76 & 80.95 & 66.51 & \textbf{82.49} \\
Fog & 74.26 & 85.32 & 75.24 & \textbf{86.16} & 73.26 & 84.84 & 73.90 & \textbf{86.12} & 74.92 & 86.09 & 75.83 & \textbf{87.09} \\
Brightness & 91.29 & 90.75 & \textbf{92.36} & 91.23 & 91.43 & 90.60 & \textbf{92.17} & 91.11 & 91.79 & 91.15 & \textbf{92.45} & 91.80 \\
Contrast & 35.45 & 87.05 & 35.42 & \textbf{87.54} & 32.01 & 87.20 & 32.29 & \textbf{87.84} & 34.25 & 87.02 & 34.44 & \textbf{87.69} \\
Elastic & 77.20 & 76.49 & \textbf{78.56} & 78.26 & 75.89 & 75.82 & 77.52 & \textbf{77.75} & 76.76 & 77.14 & 77.99 & \textbf{78.93} \\
Pixelate & 46.55 & 79.41 & 46.47 & \textbf{80.75} & 45.74 & 79.11 & 46.11 & \textbf{80.32} & 46.89 & 79.30 & 47.05 & \textbf{80.58} \\
JPEG & 73.39 & 71.31 & \textbf{74.83} & 73.36 & 72.46 & 71.34 & \textbf{73.46} & 73.38 & 73.87 & 72.81 & \textbf{75.41} & 74.72 \\
Mean & 57.44 & 77.92 & 58.16 & 79.50 & 56.82 & 78.30 & 57.64 & 79.88 & 58.85 & 78.81 & 59.55 & 80.39 \\
\end{longtable}
}
\space
% END inlined appendix_materials/empirical/tables/cifar10-fixed.tex
% BEGIN inlined appendix_materials/empirical/tables/cifar100-fixed.tex
{
\small
\setlength{\tabcolsep}{2pt}\setlength{\LTpre}{3pt}\setlength{\LTpost}{4pt}
\setlength{\LTcapwidth}{\linewidth}
\begin{longtable}{@{}lrrrrrrrrrrrr@{}}
\caption{Complete CIFAR100-C fixed-intervention accuracy (\%). Each checkpoint evaluates all 10,000 images in every severity-5 corruption; all columns use zero feedback.}\label{tab:emp-cifar100-fixed}\\
\toprule
& \multicolumn{4}{c}{Checkpoint 0} & \multicolumn{4}{c}{Checkpoint 1} & \multicolumn{4}{c}{Checkpoint 2} \\
Corruption & I & B & F & BF & I & B & F & BF & I & B & F & BF \\
\midrule
\endfirsthead
\caption[]{Complete CIFAR100-C fixed-intervention accuracy (\%) (continued).}\\
\toprule
& \multicolumn{4}{c}{Checkpoint 0} & \multicolumn{4}{c}{Checkpoint 1} & \multicolumn{4}{c}{Checkpoint 2} \\
Corruption & I & B & F & BF & I & B & F & BF & I & B & F & BF \\
\midrule
\endhead
\bottomrule
\endfoot
Gaussian & 10.21 & 39.76 & 10.19 & \textbf{41.13} & 9.53 & 39.58 & 9.83 & \textbf{41.09} & 12.13 & 38.31 & 12.01 & \textbf{40.48} \\
Shot & 11.99 & 41.11 & 12.21 & \textbf{42.85} & 11.03 & 40.19 & 11.38 & \textbf{42.05} & 13.41 & 40.20 & 13.56 & \textbf{42.22} \\
Impulse & 6.64 & 34.72 & 6.77 & \textbf{36.32} & 6.63 & 33.61 & 6.62 & \textbf{35.27} & 8.23 & 33.57 & 8.23 & \textbf{35.83} \\
Defocus & 34.80 & 66.18 & 35.00 & \textbf{67.23} & 34.49 & 66.08 & 34.85 & \textbf{66.98} & 35.22 & 65.67 & 35.47 & \textbf{66.23} \\
Glass & 20.19 & 42.66 & 19.93 & \textbf{43.56} & 19.06 & 42.69 & 19.49 & \textbf{43.66} & 19.48 & 41.91 & 19.56 & \textbf{43.46} \\
Motion & 43.89 & 63.33 & 44.81 & \textbf{64.21} & 44.78 & 63.81 & 45.53 & \textbf{64.62} & 43.60 & 62.99 & 44.52 & \textbf{63.97} \\
Zoom & 41.65 & 64.93 & 42.15 & \textbf{65.64} & 42.44 & 64.76 & 42.79 & \textbf{65.70} & 41.70 & 64.66 & 42.63 & \textbf{65.41} \\
Snow & 46.38 & 53.87 & 46.92 & \textbf{55.13} & 46.05 & 53.65 & 47.06 & \textbf{54.90} & 46.55 & 53.41 & 47.10 & \textbf{54.40} \\
Frost & 34.05 & 53.59 & 34.45 & \textbf{54.63} & 33.93 & 53.63 & 34.36 & \textbf{54.43} & 34.61 & 53.76 & 35.29 & \textbf{54.80} \\
Fog & 40.01 & 57.85 & 40.63 & \textbf{58.89} & 41.44 & 58.22 & 42.22 & \textbf{59.19} & 41.71 & 57.93 & 42.34 & \textbf{59.28} \\
Brightness & 67.44 & 68.04 & 68.37 & \textbf{68.69} & 67.32 & 68.07 & 68.30 & \textbf{69.04} & 67.15 & 68.17 & 67.88 & \textbf{69.23} \\
Contrast & 17.79 & 64.24 & 17.87 & \textbf{64.75} & 18.61 & 64.62 & 18.66 & \textbf{65.33} & 19.00 & 64.08 & 19.00 & \textbf{65.03} \\
Elastic & 50.73 & 54.26 & 51.63 & \textbf{55.41} & 50.11 & 53.55 & 51.35 & \textbf{54.77} & 50.13 & 54.49 & 51.25 & \textbf{55.68} \\
Pixelate & 22.59 & 57.24 & 22.86 & \textbf{57.87} & 24.69 & 56.67 & 24.71 & \textbf{57.85} & 27.48 & 56.46 & 27.74 & \textbf{57.48} \\
JPEG & 45.28 & 45.65 & 45.87 & \textbf{46.44} & 44.41 & 44.70 & 44.90 & \textbf{46.10} & 44.00 & 45.46 & 44.91 & \textbf{46.72} \\
Mean & 32.91 & 53.83 & 33.31 & 54.85 & 32.97 & 53.59 & 33.47 & 54.73 & 33.63 & 53.40 & 34.10 & 54.68 \\
\end{longtable}
}
\space
% END inlined appendix_materials/empirical/tables/cifar100-fixed.tex

\subsection{Selection on a common evaluation population}

Table~\ref{tab:emp-cifar-summary} evaluates the entire portfolio on the same held-aside image roles. The fixed mean-logit rule averages the four candidate logits and needs no annotations. The batch oracle chooses its action from all current $E$ labels and provides an evaluation-only upper reference. Each empirical rule spends 471 labels per corruption, including three on the final 16-image batch, for 7,065 labels per suite.

\PORTAL{} attains 79.2587\% on CIFAR10-C with 612.3 labels per suite and 53.8680\% on CIFAR100-C with 381.7 labels, on average. Symmetric-KL empirical selection gives 79.2022\% and 53.7604\%, using 7,065 labels in either dataset. \PORTAL{} achieves similar accuracy with far fewer labels. Fixed BF has the highest deployable accuracy here: 79.6462\% and 54.3213\% on the same $E$ populations.

The six checkpoint tables retain every common-$E$ corruption outcome, including individual fixed actions, the batch oracle, and actual joint-query counts.

% BEGIN inlined appendix_materials/empirical/tables/cifar-summary.tex
{
\small
\setlength{\tabcolsep}{2pt}\setlength{\LTpre}{3pt}\setlength{\LTpost}{4pt}
\setlength{\LTcapwidth}{\linewidth}
\begin{longtable}{@{}lrrrr@{}}
\caption{Complete-suite comparison on the common evaluation population. Accuracy is averaged over three checkpoints; the accompanying value is the sample SD. Label counts are per 15-corruption suite.}\label{tab:emp-cifar-summary}\\
\toprule
Dataset & Method & Accuracy (\%) & $\Delta$ BF (pp) & Labels \\
\midrule
\endfirsthead
\caption[]{Complete-suite comparison on the common evaluation population (continued).}\\
\toprule
Dataset & Method & Accuracy (\%) & $\Delta$ BF (pp) & Labels \\
\midrule
\endhead
\bottomrule
\endfoot
CIFAR10-C & Fixed BF & $79.646\pm0.561$ & 0.000 & 0.0 \\
CIFAR10-C & Mean logits & $76.548\pm1.295$ & -3.099 & 0.0 \\
CIFAR10-C & Uniform ERM & $78.275\pm0.583$ & -1.371 & 7065.0 \\
CIFAR10-C & Disagreement ERM & $78.361\pm0.591$ & -1.285 & 7065.0 \\
CIFAR10-C & Symmetric-KL ERM & $79.202\pm0.499$ & -0.444 & 7065.0 \\
CIFAR10-C & \EmpJoint & $79.259\pm0.506$ & -0.388 & 612.3 \\
CIFAR10-C & Batch oracle & $81.135\pm0.449$ & 1.489 & 75000.0 \\
CIFAR100-C & Fixed BF & $54.321\pm0.052$ & 0.000 & 0.0 \\
CIFAR100-C & Mean logits & $52.356\pm0.206$ & -1.966 & 0.0 \\
CIFAR100-C & Uniform ERM & $52.788\pm0.082$ & -1.533 & 7065.0 \\
CIFAR100-C & Disagreement ERM & $52.802\pm0.232$ & -1.519 & 7065.0 \\
CIFAR100-C & Symmetric-KL ERM & $53.760\pm0.119$ & -0.561 & 7065.0 \\
CIFAR100-C & \EmpJoint & $53.868\pm0.231$ & -0.453 & 381.7 \\
CIFAR100-C & Batch oracle & $55.952\pm0.149$ & 1.630 & 75000.0 \\
\end{longtable}
}
\space
% END inlined appendix_materials/empirical/tables/cifar-summary.tex
% BEGIN inlined appendix_materials/empirical/tables/cifar10-control-0.tex
{
\small
\setlength{\tabcolsep}{2pt}\setlength{\LTpre}{3pt}\setlength{\LTpost}{4pt}
\setlength{\LTcapwidth}{\linewidth}
\begin{longtable}{@{}lrrrrrrrrrrr@{}}
\caption{CIFAR10-C, checkpoint 0: accuracy (\%) on the same 5,000 evaluation images per corruption. U, D and K are uniform, disagreement-first and symmetric-KL empirical selection; J is joint selection; O is the batch oracle. The last column counts J labels, with a suite total in the Mean row.}\label{tab:emp-cifar10-control-0}\\
\toprule
Corruption & I & B & F & BF & Mean & U & D & K & J & O & $Q_{J}$ \\
\midrule
\endfirsthead
\caption[]{CIFAR10-C, checkpoint 0: accuracy (\%) on the same 5,000 evaluation images per corruption (continued).}\\
\toprule
Corruption & I & B & F & BF & Mean & U & D & K & J & O & $Q_{J}$ \\
\midrule
\endhead
\bottomrule
\endfoot
Gaussian & 21.78 & 65.92 & 21.86 & 67.56 & 56.70 & 65.08 & 65.58 & 67.58 & 67.56 & 68.76 & 53 \\
Shot & 29.04 & 68.12 & 29.14 & 70.48 & 63.04 & 68.62 & 68.84 & 70.24 & 70.24 & 71.42 & 50 \\
Impulse & 23.06 & 57.74 & 23.04 & 61.38 & 55.34 & 57.96 & 58.04 & 61.28 & 61.38 & 62.24 & 54 \\
Defocus & 54.62 & 87.14 & 54.90 & 88.26 & 83.20 & 86.80 & 87.34 & 88.26 & 87.96 & 89.06 & 44 \\
Glass & 54.60 & 65.04 & 55.58 & 67.44 & 65.86 & 65.08 & 64.74 & 66.24 & 65.72 & 68.72 & 72 \\
Motion & 67.50 & 86.20 & 68.22 & 86.64 & 84.72 & 84.94 & 85.64 & 86.42 & 86.34 & 87.72 & 42 \\
Zoom & 64.42 & 86.58 & 65.32 & 88.18 & 85.10 & 87.04 & 86.38 & 88.08 & 87.96 & 88.78 & 39 \\
Snow & 77.56 & 80.70 & 78.44 & 81.62 & 82.88 & 80.84 & 80.18 & 79.94 & 81.08 & 84.14 & 27 \\
Frost & 69.32 & 79.02 & 69.92 & 80.80 & 81.34 & 79.38 & 78.14 & 80.00 & 79.74 & 82.02 & 44 \\
Fog & 73.70 & 84.26 & 73.88 & 85.24 & 84.90 & 84.24 & 83.24 & 84.18 & 84.04 & 86.40 & 60 \\
Brightness & 91.10 & 90.50 & 92.06 & 91.06 & 93.10 & 90.98 & 91.18 & 91.12 & 90.96 & 93.98 & 10 \\
Contrast & 35.20 & 86.84 & 35.20 & 87.32 & 84.16 & 87.32 & 87.08 & 87.32 & 87.26 & 88.00 & 59 \\
Elastic & 77.50 & 76.54 & 78.52 & 78.20 & 80.40 & 78.30 & 78.04 & 77.78 & 78.18 & 81.90 & 34 \\
Pixelate & 45.94 & 79.56 & 45.70 & 80.56 & 72.38 & 79.24 & 79.48 & 80.52 & 80.56 & 81.64 & 30 \\
JPEG & 73.46 & 71.04 & 74.76 & 72.66 & 75.74 & 72.62 & 72.66 & 73.04 & 73.06 & 77.38 & 31 \\
Mean & 57.25 & 77.68 & 57.77 & 79.16 & 76.59 & 77.90 & 77.77 & 78.80 & 78.80 & 80.81 & 649 \\
\end{longtable}
}
\space
% END inlined appendix_materials/empirical/tables/cifar10-control-0.tex
% BEGIN inlined appendix_materials/empirical/tables/cifar10-control-1.tex
{
\small
\setlength{\tabcolsep}{2pt}\setlength{\LTpre}{3pt}\setlength{\LTpost}{4pt}
\setlength{\LTcapwidth}{\linewidth}
\begin{longtable}{@{}lrrrrrrrrrrr@{}}
\caption{CIFAR10-C, checkpoint 1: accuracy (\%) on the same 5,000 evaluation images per corruption. U, D and K are uniform, disagreement-first and symmetric-KL empirical selection; J is joint selection; O is the batch oracle. The last column counts J labels, with a suite total in the Mean row.}\label{tab:emp-cifar10-control-1}\\
\toprule
Corruption & I & B & F & BF & Mean & U & D & K & J & O & $Q_{J}$ \\
\midrule
\endfirsthead
\caption[]{CIFAR10-C, checkpoint 1: accuracy (\%) on the same 5,000 evaluation images per corruption (continued).}\\
\toprule
Corruption & I & B & F & BF & Mean & U & D & K & J & O & $Q_{J}$ \\
\midrule
\endhead
\bottomrule
\endfoot
Gaussian & 27.50 & 66.08 & 28.64 & 69.28 & 56.76 & 66.60 & 67.16 & 69.24 & 68.80 & 70.14 & 58 \\
Shot & 33.60 & 68.32 & 34.50 & 71.30 & 61.78 & 67.86 & 70.12 & 71.20 & 70.94 & 72.00 & 55 \\
Impulse & 22.04 & 60.30 & 22.04 & 62.82 & 50.36 & 60.20 & 60.86 & 62.48 & 62.74 & 63.92 & 62 \\
Defocus & 55.72 & 87.28 & 56.48 & 88.66 & 82.80 & 87.38 & 87.94 & 88.64 & 88.66 & 89.26 & 34 \\
Glass & 49.28 & 64.18 & 50.62 & 66.74 & 62.28 & 63.88 & 64.18 & 64.86 & 66.34 & 67.72 & 34 \\
Motion & 66.32 & 85.60 & 66.86 & 86.44 & 83.02 & 85.74 & 85.52 & 86.32 & 86.34 & 87.46 & 28 \\
Zoom & 62.70 & 87.58 & 64.10 & 88.58 & 84.38 & 87.38 & 88.10 & 88.42 & 87.64 & 89.44 & 33 \\
Snow & 76.86 & 81.18 & 77.74 & 83.36 & 83.10 & 82.00 & 81.18 & 81.40 & 81.90 & 84.62 & 39 \\
Frost & 64.54 & 80.42 & 65.94 & 81.58 & 79.60 & 79.38 & 80.20 & 80.78 & 81.32 & 82.84 & 36 \\
Fog & 72.80 & 84.18 & 73.46 & 85.50 & 84.24 & 84.08 & 84.28 & 85.02 & 85.30 & 86.42 & 29 \\
Brightness & 91.06 & 90.24 & 91.90 & 90.58 & 92.90 & 90.90 & 91.10 & 90.92 & 90.64 & 93.84 & 16 \\
Contrast & 32.30 & 87.10 & 32.40 & 87.78 & 81.86 & 86.22 & 87.00 & 87.78 & 87.52 & 88.56 & 46 \\
Elastic & 75.96 & 76.12 & 77.74 & 77.96 & 79.66 & 77.72 & 77.32 & 77.10 & 77.38 & 81.26 & 25 \\
Pixelate & 44.60 & 78.02 & 45.04 & 79.28 & 70.76 & 77.68 & 78.02 & 79.26 & 79.10 & 80.32 & 27 \\
JPEG & 72.10 & 70.30 & 73.06 & 72.92 & 74.98 & 72.72 & 72.42 & 72.28 & 72.94 & 76.40 & 34 \\
Mean & 56.49 & 77.79 & 57.37 & 79.52 & 75.23 & 77.98 & 78.36 & 79.05 & 79.17 & 80.95 & 556 \\
\end{longtable}
}
\space
% END inlined appendix_materials/empirical/tables/cifar10-control-1.tex
% BEGIN inlined appendix_materials/empirical/tables/cifar10-control-2.tex
{
\small
\setlength{\tabcolsep}{2pt}\setlength{\LTpre}{3pt}\setlength{\LTpost}{4pt}
\setlength{\LTcapwidth}{\linewidth}
\begin{longtable}{@{}lrrrrrrrrrrr@{}}
\caption{CIFAR10-C, checkpoint 2: accuracy (\%) on the same 5,000 evaluation images per corruption. U, D and K are uniform, disagreement-first and symmetric-KL empirical selection; J is joint selection; O is the batch oracle. The last column counts J labels, with a suite total in the Mean row.}\label{tab:emp-cifar10-control-2}\\
\toprule
Corruption & I & B & F & BF & Mean & U & D & K & J & O & $Q_{J}$ \\
\midrule
\endfirsthead
\caption[]{CIFAR10-C, checkpoint 2: accuracy (\%) on the same 5,000 evaluation images per corruption (continued).}\\
\toprule
Corruption & I & B & F & BF & Mean & U & D & K & J & O & $Q_{J}$ \\
\midrule
\endhead
\bottomrule
\endfoot
Gaussian & 28.14 & 66.70 & 28.40 & 69.34 & 62.20 & 66.70 & 67.04 & 69.24 & 69.16 & 70.28 & 56 \\
Shot & 35.88 & 69.10 & 35.84 & 71.80 & 67.14 & 70.26 & 70.92 & 71.64 & 71.38 & 72.74 & 66 \\
Impulse & 29.98 & 60.48 & 30.44 & 64.04 & 57.80 & 62.32 & 62.50 & 63.24 & 63.14 & 64.96 & 64 \\
Defocus & 59.58 & 87.24 & 60.00 & 88.74 & 85.76 & 87.16 & 87.80 & 88.66 & 88.34 & 89.38 & 43 \\
Glass & 50.66 & 65.82 & 51.42 & 67.44 & 64.54 & 65.18 & 64.60 & 66.42 & 67.02 & 68.76 & 47 \\
Motion & 68.98 & 86.62 & 69.86 & 87.68 & 85.48 & 86.28 & 86.52 & 87.66 & 87.14 & 88.40 & 38 \\
Zoom & 67.84 & 87.80 & 68.16 & 88.52 & 86.28 & 87.52 & 88.10 & 88.50 & 88.20 & 89.56 & 40 \\
Snow & 77.42 & 81.58 & 78.42 & 83.94 & 84.26 & 82.64 & 81.40 & 82.18 & 82.34 & 85.26 & 38 \\
Frost & 65.22 & 80.76 & 65.76 & 82.50 & 80.34 & 80.32 & 80.52 & 81.54 & 82.22 & 83.30 & 37 \\
Fog & 74.08 & 85.56 & 74.76 & 86.62 & 85.48 & 85.48 & 84.58 & 85.38 & 85.34 & 87.72 & 55 \\
Brightness & 91.32 & 91.08 & 92.20 & 91.74 & 93.60 & 91.96 & 92.02 & 91.88 & 91.56 & 94.56 & 13 \\
Contrast & 34.22 & 86.82 & 34.12 & 87.42 & 84.94 & 86.56 & 87.10 & 87.42 & 87.42 & 88.08 & 47 \\
Elastic & 77.06 & 77.02 & 77.96 & 78.94 & 80.40 & 78.56 & 77.94 & 78.30 & 78.76 & 81.56 & 29 \\
Pixelate & 46.38 & 79.50 & 46.32 & 80.60 & 72.78 & 78.82 & 79.04 & 80.56 & 80.62 & 81.76 & 33 \\
JPEG & 73.66 & 72.88 & 75.38 & 74.58 & 76.30 & 74.44 & 74.20 & 73.78 & 74.40 & 78.40 & 26 \\
Mean & 58.69 & 78.60 & 59.27 & 80.26 & 77.82 & 78.95 & 78.95 & 79.76 & 79.80 & 81.65 & 632 \\
\end{longtable}
}
\space
% END inlined appendix_materials/empirical/tables/cifar10-control-2.tex
% BEGIN inlined appendix_materials/empirical/tables/cifar100-control-0.tex
{
\small
\setlength{\tabcolsep}{2pt}\setlength{\LTpre}{3pt}\setlength{\LTpost}{4pt}
\setlength{\LTcapwidth}{\linewidth}
\begin{longtable}{@{}lrrrrrrrrrrr@{}}
\caption{CIFAR100-C, checkpoint 0: accuracy (\%) on the same 5,000 evaluation images per corruption. U, D and K are uniform, disagreement-first and symmetric-KL empirical selection; J is joint selection; O is the batch oracle. The last column counts J labels, with a suite total in the Mean row.}\label{tab:emp-cifar100-control-0}\\
\toprule
Corruption & I & B & F & BF & Mean & U & D & K & J & O & $Q_{J}$ \\
\midrule
\endfirsthead
\caption[]{CIFAR100-C, checkpoint 0: accuracy (\%) on the same 5,000 evaluation images per corruption (continued).}\\
\toprule
Corruption & I & B & F & BF & Mean & U & D & K & J & O & $Q_{J}$ \\
\midrule
\endhead
\bottomrule
\endfoot
Gaussian & 10.70 & 39.66 & 10.54 & 40.84 & 32.32 & 39.46 & 39.84 & 40.34 & 40.84 & 41.88 & 52 \\
Shot & 12.26 & 41.08 & 12.52 & 42.68 & 34.58 & 41.06 & 40.72 & 41.94 & 42.36 & 44.00 & 63 \\
Impulse & 7.16 & 35.04 & 7.22 & 36.34 & 29.18 & 35.34 & 35.32 & 36.16 & 34.74 & 37.74 & 46 \\
Defocus & 34.18 & 65.30 & 34.00 & 66.22 & 62.44 & 64.24 & 64.84 & 66.22 & 65.50 & 67.24 & 27 \\
Glass & 20.32 & 42.12 & 19.92 & 42.86 & 39.44 & 40.30 & 40.46 & 42.30 & 42.56 & 44.30 & 36 \\
Motion & 43.28 & 62.58 & 44.28 & 63.28 & 61.76 & 61.46 & 60.86 & 62.64 & 63.08 & 64.48 & 12 \\
Zoom & 41.04 & 64.10 & 41.40 & 64.98 & 63.70 & 62.06 & 62.22 & 64.64 & 63.64 & 66.14 & 25 \\
Snow & 46.32 & 53.10 & 47.00 & 54.46 & 56.74 & 52.84 & 52.48 & 53.70 & 53.98 & 56.08 & 8 \\
Frost & 33.84 & 52.98 & 34.42 & 54.14 & 54.78 & 51.48 & 52.14 & 53.62 & 52.38 & 55.26 & 25 \\
Fog & 39.04 & 57.50 & 39.76 & 58.76 & 57.92 & 56.58 & 56.38 & 58.20 & 58.76 & 59.88 & 21 \\
Brightness & 67.52 & 66.62 & 68.00 & 67.58 & 71.62 & 67.36 & 67.52 & 67.44 & 67.38 & 71.56 & 4 \\
Contrast & 17.42 & 64.62 & 17.38 & 65.06 & 58.66 & 64.16 & 63.18 & 65.04 & 65.04 & 66.14 & 34 \\
Elastic & 50.20 & 53.62 & 51.54 & 54.26 & 56.90 & 53.54 & 53.04 & 53.18 & 54.28 & 57.38 & 10 \\
Pixelate & 22.06 & 56.42 & 22.34 & 57.10 & 53.26 & 55.98 & 55.54 & 56.92 & 53.84 & 58.40 & 38 \\
JPEG & 45.70 & 45.64 & 46.54 & 46.18 & 50.10 & 46.16 & 45.74 & 45.42 & 46.20 & 50.68 & 8 \\
Mean & 32.74 & 53.36 & 33.12 & 54.32 & 52.23 & 52.80 & 52.69 & 53.85 & 53.64 & 56.08 & 409 \\
\end{longtable}
}
\space
% END inlined appendix_materials/empirical/tables/cifar100-control-0.tex
% BEGIN inlined appendix_materials/empirical/tables/cifar100-control-1.tex
{
\small
\setlength{\tabcolsep}{2pt}\setlength{\LTpre}{3pt}\setlength{\LTpost}{4pt}
\setlength{\LTcapwidth}{\linewidth}
\begin{longtable}{@{}lrrrrrrrrrrr@{}}
\caption{CIFAR100-C, checkpoint 1: accuracy (\%) on the same 5,000 evaluation images per corruption. U, D and K are uniform, disagreement-first and symmetric-KL empirical selection; J is joint selection; O is the batch oracle. The last column counts J labels, with a suite total in the Mean row.}\label{tab:emp-cifar100-control-1}\\
\toprule
Corruption & I & B & F & BF & Mean & U & D & K & J & O & $Q_{J}$ \\
\midrule
\endfirsthead
\caption[]{CIFAR100-C, checkpoint 1: accuracy (\%) on the same 5,000 evaluation images per corruption (continued).}\\
\toprule
Corruption & I & B & F & BF & Mean & U & D & K & J & O & $Q_{J}$ \\
\midrule
\endhead
\bottomrule
\endfoot
Gaussian & 9.92 & 39.70 & 10.20 & 40.94 & 31.30 & 39.64 & 39.60 & 40.72 & 40.94 & 42.08 & 58 \\
Shot & 11.12 & 40.30 & 11.56 & 41.78 & 33.60 & 38.40 & 38.68 & 41.48 & 41.78 & 42.82 & 57 \\
Impulse & 7.36 & 34.02 & 7.36 & 35.30 & 30.76 & 34.56 & 34.90 & 34.96 & 35.10 & 36.50 & 41 \\
Defocus & 33.86 & 64.76 & 33.88 & 65.88 & 62.60 & 63.38 & 64.46 & 65.76 & 65.48 & 66.84 & 25 \\
Glass & 19.10 & 42.24 & 19.78 & 43.16 & 39.14 & 41.76 & 42.14 & 42.54 & 42.32 & 44.28 & 34 \\
Motion & 43.52 & 62.80 & 44.54 & 63.64 & 62.12 & 61.80 & 62.38 & 63.46 & 62.42 & 64.96 & 11 \\
Zoom & 41.86 & 64.20 & 42.18 & 65.44 & 63.70 & 62.44 & 62.80 & 65.02 & 65.44 & 66.58 & 22 \\
Snow & 46.82 & 53.52 & 47.70 & 54.94 & 57.42 & 54.02 & 53.38 & 53.58 & 54.72 & 56.82 & 9 \\
Frost & 34.06 & 53.44 & 34.60 & 54.18 & 54.44 & 51.90 & 52.50 & 53.34 & 53.68 & 55.56 & 19 \\
Fog & 40.20 & 58.04 & 40.96 & 58.98 & 57.90 & 56.58 & 56.80 & 57.88 & 58.20 & 60.18 & 19 \\
Brightness & 67.24 & 67.46 & 68.16 & 68.60 & 71.90 & 68.44 & 68.24 & 68.02 & 68.48 & 71.88 & 8 \\
Contrast & 18.06 & 65.02 & 18.26 & 65.66 & 59.34 & 64.54 & 64.70 & 65.62 & 65.66 & 66.60 & 46 \\
Elastic & 49.58 & 53.06 & 51.14 & 54.44 & 56.34 & 53.66 & 53.14 & 52.56 & 54.44 & 57.20 & 10 \\
Pixelate & 24.04 & 55.82 & 23.90 & 57.16 & 54.12 & 56.20 & 56.54 & 56.88 & 57.16 & 58.16 & 29 \\
JPEG & 44.54 & 44.20 & 44.76 & 45.54 & 49.02 & 45.62 & 45.78 & 45.26 & 45.68 & 49.40 & 5 \\
Mean & 32.75 & 53.24 & 33.27 & 54.38 & 52.25 & 52.86 & 53.07 & 53.81 & 54.10 & 55.99 & 393 \\
\end{longtable}
}
\space
% END inlined appendix_materials/empirical/tables/cifar100-control-1.tex
% BEGIN inlined appendix_materials/empirical/tables/cifar100-control-2.tex
{
\small
\setlength{\tabcolsep}{2pt}\setlength{\LTpre}{3pt}\setlength{\LTpost}{4pt}
\setlength{\LTcapwidth}{\linewidth}
\begin{longtable}{@{}lrrrrrrrrrrr@{}}
\caption{CIFAR100-C, checkpoint 2: accuracy (\%) on the same 5,000 evaluation images per corruption. U, D and K are uniform, disagreement-first and symmetric-KL empirical selection; J is joint selection; O is the batch oracle. The last column counts J labels, with a suite total in the Mean row.}\label{tab:emp-cifar100-control-2}\\
\toprule
Corruption & I & B & F & BF & Mean & U & D & K & J & O & $Q_{J}$ \\
\midrule
\endfirsthead
\caption[]{CIFAR100-C, checkpoint 2: accuracy (\%) on the same 5,000 evaluation images per corruption (continued).}\\
\toprule
Corruption & I & B & F & BF & Mean & U & D & K & J & O & $Q_{J}$ \\
\midrule
\endhead
\bottomrule
\endfoot
Gaussian & 12.32 & 38.18 & 12.06 & 39.90 & 34.02 & 38.80 & 38.66 & 39.38 & 39.66 & 40.76 & 53 \\
Shot & 13.68 & 40.16 & 13.62 & 41.86 & 36.10 & 40.10 & 39.98 & 41.20 & 41.56 & 43.06 & 36 \\
Impulse & 8.64 & 33.44 & 8.66 & 35.58 & 30.74 & 33.60 & 33.52 & 34.60 & 35.58 & 36.60 & 32 \\
Defocus & 34.64 & 65.46 & 34.90 & 65.72 & 62.14 & 64.62 & 64.76 & 65.70 & 65.62 & 67.06 & 29 \\
Glass & 19.90 & 40.74 & 19.88 & 42.22 & 39.28 & 40.30 & 40.22 & 41.24 & 41.12 & 43.30 & 27 \\
Motion & 42.68 & 62.22 & 43.58 & 62.74 & 60.96 & 60.78 & 60.90 & 62.42 & 62.28 & 63.72 & 14 \\
Zoom & 41.10 & 64.98 & 41.96 & 65.40 & 63.46 & 63.80 & 63.84 & 64.76 & 64.80 & 66.54 & 16 \\
Snow & 46.50 & 53.22 & 47.08 & 54.36 & 56.38 & 52.98 & 53.40 & 53.46 & 54.08 & 56.16 & 10 \\
Frost & 34.70 & 53.86 & 35.46 & 54.96 & 55.64 & 52.32 & 53.08 & 54.48 & 54.40 & 56.22 & 23 \\
Fog & 40.98 & 57.98 & 41.40 & 59.08 & 58.70 & 57.32 & 56.60 & 58.30 & 57.38 & 60.30 & 14 \\
Brightness & 66.74 & 67.52 & 67.34 & 68.44 & 72.40 & 67.88 & 67.38 & 67.58 & 68.36 & 71.42 & 8 \\
Contrast & 18.62 & 63.90 & 18.78 & 65.00 & 58.62 & 63.66 & 63.54 & 65.02 & 65.00 & 65.70 & 36 \\
Elastic & 50.56 & 53.96 & 51.78 & 55.10 & 57.26 & 54.46 & 53.58 & 54.06 & 55.10 & 57.74 & 2 \\
Pixelate & 26.90 & 56.12 & 27.10 & 57.24 & 53.70 & 54.68 & 55.20 & 56.74 & 56.60 & 58.38 & 31 \\
JPEG & 43.92 & 45.38 & 44.86 & 46.48 & 49.50 & 45.20 & 45.12 & 45.44 & 46.44 & 49.84 & 12 \\
Mean & 33.46 & 53.14 & 33.90 & 54.27 & 52.59 & 52.70 & 52.65 & 53.63 & 53.87 & 55.79 & 343 \\
\end{longtable}
}
\space
% END inlined appendix_materials/empirical/tables/cifar100-control-2.tex

\section{Visual-language readouts across natural domains}
\label{app:empirical-vlm}

\subsection{Whole-dataset fixed performance}

The CLIP candidates combine text prompts and image views. S uses one text template and T averages eight templates, both on the original image. V2 averages the eight-template log probabilities over the original image and its horizontal flip. V8 adds three resized crops and their flips, for eight views in total. Later tables call V2 and V8 ``U2V'' and ``U8V'' to indicate the shared template set. The crop scale is $[0.8,1.0]$ and the aspect-ratio interval is $[3/4,4/3]$. All experiments use these same transformations. Five repeats vary only the stream and query order.

Table~\ref{tab:emp-vlm-whole} reports every domain and both aggregation rules. The V8 readout gives 98.5171\% equal-domain accuracy on PACS and 83.7927\% on VLCS without target labels. The gains are distributed unevenly across domains. On VLCS, the SUN domain rises from 66.3909\% under S to 77.7536\% under V8, while LABELME decreases from 71.6514\% to 71.3825\%. This domain variation motivates retaining the complete domain table alongside the aggregate.

The OfficeHome saved prediction stream used for current control has V8 accuracy 88.8263\%. The separate full-stream fixed endpoint used for the reported-reference comparison is 88.8435\%. Table~\ref{tab:emp-reported-comparison} retains that endpoint. The control comparisons below always use their own fixed predictions as the paired reference.

% BEGIN inlined appendix_materials/empirical/tables/vlm-whole.tex
{
\small
\setlength{\tabcolsep}{2pt}\setlength{\LTpre}{3pt}\setlength{\LTpost}{4pt}
\setlength{\LTcapwidth}{\linewidth}
\begin{longtable}{@{}llrrrrr@{}}
\caption{Whole-dataset CLIP ViT-L/14 readouts with zero feedback. S uses one template and one view; T uses eight templates on the original view; V2 and V8 average universal-template log probabilities over two and eight views. OfficeHome uses the saved prediction stream evaluated by its controller.}\label{tab:emp-vlm-whole}\\
\toprule
Dataset & Domain & $n$ & S & T & V2 & V8 \\
\midrule
\endfirsthead
\caption[]{Whole-dataset CLIP ViT-L/14 readouts with zero feedback (continued).}\\
\toprule
Dataset & Domain & $n$ & S & T & V2 & V8 \\
\midrule
\endhead
\bottomrule
\endfoot
PACS & Art Painting & 2048 & 98.68 & 99.02 & 99.17 & 99.17 \\
PACS & Cartoon & 2344 & 99.45 & 99.57 & 99.57 & 99.62 \\
PACS & Photo & 1670 & 99.88 & 99.94 & 99.94 & 99.94 \\
PACS & Sketch & 3929 & 95.47 & 95.04 & 95.19 & 95.34 \\
PACS & Equal-domain & 9991 & 98.369 & 98.393 & 98.468 & 98.517 \\
PACS & Pooled & 9991 & 97.798 & 97.738 & 97.828 & 97.898 \\
OfficeHome & Art & 2427 & 86.98 & 88.01 & 88.26 & 88.71 \\
OfficeHome & Clipart & 4365 & 73.08 & 76.01 & 76.33 & 78.49 \\
OfficeHome & Product & 4439 & 92.90 & 93.74 & 93.67 & 94.12 \\
OfficeHome & Real World & 4357 & 93.05 & 93.94 & 93.83 & 93.99 \\
OfficeHome & Equal-domain & 15588 & 86.503 & 87.925 & 88.022 & 88.826 \\
OfficeHome & Pooled & 15588 & 86.470 & 87.939 & 88.016 & 88.863 \\
VLCS & Caltech & 991 & 100.00 & 100.00 & 100.00 & 100.00 \\
VLCS & Labelme & 1859 & 71.65 & 71.87 & 71.60 & 71.38 \\
VLCS & Pascal & 2363 & 84.85 & 85.78 & 86.16 & 86.03 \\
VLCS & Sun & 2297 & 66.39 & 74.75 & 75.88 & 77.75 \\
VLCS & Equal-domain & 7510 & 80.723 & 83.099 & 83.410 & 83.793 \\
VLCS & Pooled & 7510 & 77.936 & 80.839 & 81.238 & 81.718 \\
\end{longtable}
}
\space
% END inlined appendix_materials/empirical/tables/vlm-whole.tex

\begin{table}[!htb]
\centering
\captionsetup{font=small,skip=4pt}
\caption{CLIP accuracy (\%) / labels on shared evaluation images. Accuracy averages domains; labels average five orders. Mean-pair KL ranks inputs by average pairwise divergence. A dash marks an unmeasured entry. Appendix~\ref{app:empirical-vlm} gives full readouts.}
\label{tab:expanded_vlm}
\fontsize{8}{9.3}\selectfont
\setlength{\tabcolsep}{3pt}
\begin{tabularx}{\linewidth}{@{}Xrrr@{}}
\toprule
Rule & PACS & OfficeHome & VLCS \\
\midrule
Fixed eight views & 98.44/0 & 89.01/0 & 84.10/0 \\
Uniform ERM & 98.43/473.6 & 88.90/738.0 & 84.07/354.8 \\
Mean-pair KL & 98.44/473.6 & 88.63/738.0 & 83.91/354.8 \\
\rowcolor{PortalRow}\PORTAL{}, four & 98.44/1.4 & 89.01/7.4 & 83.92/89.4 \\
\midrule
Fixed StatA & 98.40/0 & --- & 85.11/0 \\
\rowcolor{PortalRow}\PORTAL{} + StatA & 98.44/2.4 & --- & 84.58/96.0 \\
\bottomrule
\end{tabularx}
\end{table}

\subsection{Fixed capacity and actual use of feedback}

Table~\ref{tab:emp-vlm-common} places all fixed readouts and query rules on their common $E$ populations. Joint selection uses 1.4 labels per PACS stream and changes the equal-domain score by $+0.0050$ percentage points relative to V8. On OfficeHome, it uses 7.4 labels and changes accuracy by $+0.0021$ points. These near-default outcomes retain the high accuracy of the fixed readout with little annotation. On VLCS, joint selection gives 83.9188\%, below fixed V8 at 84.0966\% and uniform empirical selection at 84.0657\%. The actual effect of feedback therefore depends on both the available readouts and the dataset.

Table~\ref{tab:emp-vlm-domains} resolves these effects by domain. Fixed readouts have zero stream-order variance because their predictions and evaluation membership remain unchanged.

% BEGIN inlined appendix_materials/empirical/tables/vlm-common.tex
{
\small
\setlength{\tabcolsep}{2pt}\setlength{\LTpre}{3pt}\setlength{\LTpost}{4pt}
\setlength{\LTcapwidth}{\linewidth}
\begin{longtable}{@{}llrrrr@{}}
\caption{VLM evaluation on a permanent common E population. Accuracy is an equal-domain mean over five stream orders; SD measures stream/query variation on one checkpoint and set of image transformations. Pooled accuracy weights images. Labels are actual per-stream annotations.}\label{tab:emp-vlm-common}\\
\toprule
Dataset & Method & Equal-domain (\%) & Pooled & $\Delta$ V8 & Labels \\
\midrule
\endfirsthead
\caption[]{VLM evaluation on a permanent common E population (continued).}\\
\toprule
Dataset & Method & Equal-domain (\%) & Pooled & $\Delta$ V8 & Labels \\
\midrule
\endhead
\bottomrule
\endfoot
PACS & Fixed S & $98.1906\pm0.0000$ & 97.5412 & -0.2462 & 0.0 \\
PACS & Fixed T & $98.2791\pm0.0000$ & 97.5615 & -0.1578 & 0.0 \\
PACS & Fixed V2 & $98.4249\pm0.0000$ & 97.7037 & -0.0120 & 0.0 \\
PACS & Fixed V8 & $98.4369\pm0.0000$ & 97.7444 & 0.0000 & 0.0 \\
PACS & Uniform ERM & $98.4343\pm0.0056$ & 97.7403 & -0.0025 & 473.6 \\
PACS & Disagreement ERM & $98.4343\pm0.0314$ & 97.7403 & -0.0025 & 473.6 \\
PACS & Symmetric-KL ERM & $98.4394\pm0.0287$ & 97.7484 & 0.0025 & 473.6 \\
PACS & \EmpJoint & $98.4419\pm0.0069$ & 97.7525 & 0.0050 & 1.4 \\
OfficeHome & Fixed S & $86.6214\pm0.0000$ & 86.4133 & -2.3866 & 0.0 \\
OfficeHome & Fixed T & $88.1175\pm0.0000$ & 88.0057 & -0.8905 & 0.0 \\
OfficeHome & Fixed V2 & $88.2257\pm0.0000$ & 88.1341 & -0.7823 & 0.0 \\
OfficeHome & Fixed V8 & $89.0080\pm0.0000$ & 88.9174 & 0.0000 & 0.0 \\
OfficeHome & Uniform ERM & $88.8990\pm0.0930$ & 88.7967 & -0.1091 & 738.0 \\
OfficeHome & Disagreement ERM & $88.5772\pm0.1174$ & 88.4911 & -0.4308 & 738.0 \\
OfficeHome & Symmetric-KL ERM & $88.6338\pm0.1278$ & 88.5579 & -0.3742 & 738.0 \\
OfficeHome & \EmpJoint & $89.0101\pm0.0124$ & 88.9200 & 0.0021 & 7.4 \\
VLCS & Fixed S & $80.8449\pm0.0000$ & 78.0261 & -3.2517 & 0.0 \\
VLCS & Fixed T & $83.5229\pm0.0000$ & 81.2982 & -0.5737 & 0.0 \\
VLCS & Fixed V2 & $83.8614\pm0.0000$ & 81.7239 & -0.2352 & 0.0 \\
VLCS & Fixed V8 & $84.0966\pm0.0000$ & 82.0165 & 0.0000 & 0.0 \\
VLCS & Uniform ERM & $84.0657\pm0.1743$ & 81.9633 & -0.0309 & 354.8 \\
VLCS & Disagreement ERM & $83.8731\pm0.1502$ & 81.7398 & -0.2235 & 354.8 \\
VLCS & Symmetric-KL ERM & $83.9097\pm0.1239$ & 81.7930 & -0.1869 & 354.8 \\
VLCS & \EmpJoint & $83.9188\pm0.1703$ & 81.7930 & -0.1778 & 89.4 \\
\end{longtable}
}
\space
% END inlined appendix_materials/empirical/tables/vlm-common.tex
% BEGIN inlined appendix_materials/empirical/tables/vlm-domains.tex
{
\small
\setlength{\tabcolsep}{2pt}\setlength{\LTpre}{3pt}\setlength{\LTpost}{4pt}
\setlength{\LTcapwidth}{\linewidth}
\begin{longtable}{@{}llrrrrrr@{}}
\caption{VLM common-E domain-level accuracy (\%) averaged over five stream orders. U, D and K denote uniform, disagreement-first and symmetric-KL empirical selection. J denotes \PORTAL{}, and $Q_J$ counts its labels. V8 is the fixed default.}\label{tab:emp-vlm-domains}\\
\toprule
Dataset & Domain & V8 & U & D & K & J & $Q_J$ \\
\midrule
\endfirsthead
\caption[]{VLM common-E domain-level accuracy (\%) averaged over five stream orders (continued).}\\
\toprule
Dataset & Domain & V8 & U & D & K & J & $Q_J$ \\
\midrule
\endhead
\bottomrule
\endfoot
PACS & Art Painting & 99.069 & 99.069 & 99.069 & 99.069 & 99.069 & 0.0 \\
PACS & Cartoon & 99.570 & 99.570 & 99.570 & 99.570 & 99.570 & 0.0 \\
PACS & Photo & 100.000 & 100.000 & 100.000 & 100.000 & 100.000 & 0.0 \\
PACS & Sketch & 95.108 & 95.098 & 95.098 & 95.119 & 95.129 & 1.4 \\
OfficeHome & Art & 88.898 & 88.865 & 88.405 & 88.388 & 88.898 & 0.0 \\
OfficeHome & Clipart & 78.016 & 77.668 & 77.239 & 77.426 & 78.034 & 6.0 \\
OfficeHome & Product & 94.177 & 94.131 & 93.949 & 93.985 & 94.177 & 0.6 \\
OfficeHome & Real World & 94.941 & 94.932 & 94.717 & 94.735 & 94.932 & 0.8 \\
VLCS & Caltech & 100.000 & 100.000 & 100.000 & 100.000 & 100.000 & 0.0 \\
VLCS & Labelme & 71.354 & 71.647 & 71.396 & 71.228 & 71.438 & 61.2 \\
VLCS & Pascal & 86.356 & 86.304 & 85.872 & 85.717 & 85.665 & 14.6 \\
VLCS & Sun & 78.677 & 78.312 & 78.225 & 78.695 & 78.573 & 13.6 \\
\end{longtable}
}
\space
% END inlined appendix_materials/empirical/tables/vlm-domains.tex

\subsection{Output timing at a higher annotation density}

The OfficeHome experiment with batches of eight measures the effect of decision timing. Four readouts are formed first, then three queried labels select one of them. We apply that choice either to the current batch or to the next batch. Table~\ref{tab:emp-officehome-timing} shows a 1.3689-point gain when all current examples are scored. Excluding the queried examples reduces the gain to 0.0639 points; delaying application gives $-0.0887$ points. Most of the current-batch gain comes from the examples used to make the choice. The queried-excluded score and the permanent-$E$ experiment use different evaluation populations.

% BEGIN inlined appendix_materials/empirical/tables/officehome-timing.tex
{
\small
\setlength{\tabcolsep}{2pt}\setlength{\LTpre}{3pt}\setlength{\LTpost}{4pt}
\setlength{\LTcapwidth}{\linewidth}
\begin{longtable}{@{}lrr@{}}
\caption{OfficeHome with three labels per eight-image microbatch. Each timing rule is compared with the fixed output on the same scored inputs. Fixed accuracy is 88.8263\% on whole or next batches and 92.8201\% on unqueried inputs.}\label{tab:emp-officehome-timing}\\
\toprule
Estimator & Accuracy (\%) & $\Delta$ fixed (pp) \\
\midrule
\endfirsthead
\caption[]{OfficeHome with three labels per eight-image microbatch (continued).}\\
\toprule
Estimator & Accuracy (\%) & $\Delta$ fixed (pp) \\
\midrule
\endhead
\bottomrule
\endfoot
Unqueried fixed V8 & 92.8201 & 0.0000 \\
Unqueried selected output & 92.8839 & 0.0639 \\
Whole current batch & 90.1952 & 1.3689 \\
Next microbatch & 88.7376 & -0.0887 \\
\end{longtable}
}
\space
% END inlined appendix_materials/empirical/tables/officehome-timing.tex

\section{ImageNet interventions and reported reference results}
\label{app:empirical-imagenet}

\subsection{Fixed restoration and source-state compositions}

The ImageNet-C experiment evaluates all 15 severity-5 corruptions with 50,000 images each. EATTA, the active-labeling adaptation method of \citet{wang2025eatta}, supplies the base model update and acquires three labels per batch. Restoration and source-reset candidates use those same query positions. After comparing their outputs, we restore the updated EATTA state. Restoration uses a $3\times3$ median filter with reflection padding or an unsharp transform with Gaussian scale 1 and amount 1. The source-reset candidate starts from the source model with an empty sample buffer. The U readout uses batch statistics and averages original and horizontal-flip logits.

Write the five outputs as $z_E,z_U,z_M,z_D,z_S$: EATTA, two-view, median, unsharp, and source reset. Their restoration blend is
\begin{equation}
 z_{\mathrm{blend}}=\tfrac12 z_U+\tfrac16(z_M+z_D+z_S).
 \label{eq:empirical-restoration-blend}
\end{equation}
The pairwise mixtures UM, UD and US use equal weights. UMD uses weights $(1/2,1/6,1/3)$. Avg weights all five readouts equally. $U_{M.25}$ and $U_{D.25}$ use a quarter restoration weight and three quarters U. These fixed compositions expose which operations contribute across different corruption families.

Table~\ref{tab:emp-imagenet-fixed} gives all composition errors. The blend reaches 49.3395\% error, improving the locally evaluated EATTA at 50.8808\% by 1.5413 points. Its margin against the reported EATTA value of 50.90\% is 1.5605 points. The local and literature-reference differences use their respective EATTA results.

% BEGIN inlined appendix_materials/empirical/tables/imagenet-fixed.tex
{
\small
\setlength{\tabcolsep}{2pt}\setlength{\LTpre}{3pt}\setlength{\LTpost}{4pt}
\setlength{\LTcapwidth}{\linewidth}
\begin{longtable}{@{}lrrrrrrrrrr@{}}
\caption{ImageNet-C error (\%) for fixed logit compositions on all 15 severity-5 corruptions, 50,000 images each. E is the EATTA readout; U is the two-view batch-statistic readout; M, D and S are median restoration, unsharp restoration and a source-reset branch. Weights are defined in the text.}\label{tab:emp-imagenet-fixed}\\
\toprule
Corruption & E & U & UM & UD & US & UMD & Blend & Avg & $U_{M.25}$ & $U_{D.25}$ \\
\midrule
\endfirsthead
\caption[]{ImageNet-C error (\%) for fixed logit compositions on all 15 severity-5 corruptions, 50,000 images each (continued).}\\
\toprule
Corruption & E & U & UM & UD & US & UMD & Blend & Avg & $U_{M.25}$ & $U_{D.25}$ \\
\midrule
\endhead
\bottomrule
\endfoot
Gaussian & 62.61 & 60.33 & 60.23 & 61.86 & 64.26 & 60.71 & 60.54 & 61.36 & 59.77 & 60.73 \\
Shot & 55.25 & 53.35 & 54.04 & 54.60 & 57.33 & 54.01 & 53.71 & 54.48 & 53.34 & 53.69 \\
Impulse & 54.88 & 53.12 & 50.12 & 53.86 & 57.03 & 52.21 & 51.93 & 52.19 & 51.31 & 53.29 \\
Defocus & 64.08 & 62.50 & 63.02 & 62.79 & 65.30 & 62.75 & 62.61 & 63.08 & 62.66 & 62.54 \\
Glass & 61.54 & 60.28 & 60.54 & 60.58 & 63.43 & 60.53 & 60.26 & 60.61 & 60.34 & 60.32 \\
Motion & 53.52 & 53.51 & 53.61 & 53.01 & 54.85 & 53.11 & 52.60 & 52.61 & 53.38 & 53.14 \\
Zoom & 49.57 & 48.16 & 49.22 & 48.44 & 47.80 & 48.61 & 47.53 & 47.86 & 48.61 & 48.21 \\
Snow & 50.62 & 49.30 & 51.05 & 49.79 & 50.45 & 49.88 & 49.26 & 49.67 & 49.80 & 49.39 \\
Frost & 54.32 & 53.01 & 53.25 & 53.58 & 54.07 & 53.31 & 52.47 & 52.89 & 52.88 & 53.19 \\
Fog & 42.29 & 41.29 & 42.28 & 41.53 & 41.23 & 41.47 & 40.95 & 41.20 & 41.58 & 41.29 \\
Brightness & 34.70 & 34.16 & 34.89 & 34.36 & 32.21 & 34.43 & 33.32 & 33.44 & 34.37 & 34.20 \\
Contrast & 51.89 & 50.22 & 51.31 & 50.89 & 55.77 & 50.68 & 50.95 & 51.68 & 50.45 & 50.44 \\
Elastic & 44.19 & 43.08 & 44.14 & 43.17 & 43.29 & 43.29 & 42.61 & 42.95 & 43.38 & 43.01 \\
Pixelate & 40.08 & 39.17 & 39.46 & 39.62 & 39.61 & 39.44 & 38.77 & 39.01 & 39.15 & 39.32 \\
JPEG & 43.67 & 42.59 & 42.98 & 43.13 & 44.19 & 42.92 & 42.56 & 42.89 & 42.66 & 42.77 \\
Mean & 50.8808 & 49.6044 & 50.0092 & 50.0808 & 51.3871 & 49.8236 & 49.3395 & 49.7279 & 49.5789 & 49.7016 \\
\end{longtable}
}
\space
% END inlined appendix_materials/empirical/tables/imagenet-fixed.tex

\subsection{Separate and shared label acquisition}

Table~\ref{tab:emp-imagenet-control} compares two selection procedures on complete 750,000-image streams. In the first, EATTA updates the model before the horizontal-view challenger predicts the current batch. Three separately acquired symmetric-divergence queries choose between their outputs. The error reduction is 1.0083 points, and an independent full stream gives 0.9824 points. Both streams improve all 15 corruptions. In the first stream, always using the challenger improves error by 1.3741 points. Admission therefore retains 73.37\% of that gain. A paired moving-block bootstrap gives a 95\% interval of $[0.9681,1.0478]$ points for admission gain and $[71.06,75.97]\%$ for the retained fraction. It resamples ten-batch blocks within each corruption, using 10,000 replicates. Selection requests three labels in addition to EATTA's three. The recorded counts identify each request set; their overlap was not recorded.

The shared-prefix procedure passes the selector's ordered labels to EATTA's committed update. The recorded prefix contains 46,908 labels over 11,730 batches, or 3.99898 labels per batch. Error falls from 50.1383\% to 49.5487\%, a gain of 0.5896 points. Every corruption improves. Always using the challenger on the same model trajectory gives 48.7289\%, a fixed gain of 1.4093 points.

Admission occurs on 20.32\% of batches. Among admitted batches, 77.31\% improve accuracy, 8.98\% worsen it and 13.72\% tie, up to rounding. Mean recorded prefixes contain 2.87 labels for admission, 2.94 for evidence rejection and 4.54 when the budget is exhausted. The \emph{evidence floor} is the configured threshold for admitting a challenger. All 2,384 recorded admissions cross this threshold. These are decision-trace statistics.

Provisional branches run supervised EATTA updates before the prefix stops. These trial updates access at least 56,299 current-batch positions. The 46,908 committed-prefix labels therefore omit some supervised access. Predictions are missing for 9,391 available but unused positions. The saved trace cannot reconstruct evaluation on all never-accessed inputs. Removing only committed-prefix inputs gives a 0.2899-point gain; this is a prefix-excluded score.

The provisional model changes are rolled back, but admission affects the label-credit bank carried into later decisions. Holding other recorded values fixed, the opposite admission choice changes the next bank in 7,686 of 11,730 batches. A shared model successor therefore does not imply the same future query process. Table~\ref{tab:emp-imagenet-control} preserves the measured errors and identifies $q$ as the committed-prefix count.

% Sources: findings.md, R470 completion and 2026-09-14 label-access audit.
% BEGIN inlined appendix_materials/empirical/tables/imagenet-control.tex
{
\small
\setlength{\tabcolsep}{2pt}\setlength{\LTpre}{3pt}\setlength{\LTpost}{4pt}
\setlength{\LTcapwidth}{\linewidth}
\begin{longtable}{@{}lrrrrrrrrrr@{}}
\caption{ImageNet-C current-batch error (\%) under two feedback protocols. The first two streams use three adaptation labels plus three separate selection labels per batch. The shared-prefix stream reuses selection labels for the committed update; q counts its committed prefix, excluding additional provisional supervised reads. P, F and A denote EATTA, fixed challenger and selected output.}\label{tab:emp-imagenet-control}\\
\toprule
& \multicolumn{3}{c}{Additional-query stream} & \multicolumn{3}{c}{Independent repeat} & \multicolumn{4}{c}{Shared-prefix stream} \\
Corruption & P & F & A & P & F & A & P & F & A & q \\
\midrule
\endfirsthead
\caption[]{ImageNet-C current-batch error (\%) under two feedback protocols (continued).}\\
\toprule
& \multicolumn{3}{c}{Additional-query stream} & \multicolumn{3}{c}{Independent repeat} & \multicolumn{4}{c}{Shared-prefix stream} \\
Corruption & P & F & A & P & F & A & P & F & A & q \\
\midrule
\endhead
\bottomrule
\endfoot
Gaussian & 62.82 & 60.32 & 61.26 & 62.98 & 60.36 & 61.14 & 61.77 & 59.90 & 61.10 & 3.92 \\
Shot & 55.32 & 53.33 & 53.98 & 55.38 & 53.58 & 54.22 & 54.69 & 53.37 & 54.06 & 3.99 \\
Impulse & 54.90 & 53.02 & 53.66 & 54.88 & 53.08 & 53.63 & 54.24 & 51.38 & 53.23 & 3.49 \\
Defocus & 64.11 & 62.44 & 63.05 & 64.26 & 62.56 & 63.08 & 63.12 & 61.85 & 62.66 & 4.19 \\
Glass & 61.39 & 60.00 & 60.51 & 61.46 & 60.26 & 60.72 & 60.47 & 59.37 & 59.99 & 4.30 \\
Motion & 52.99 & 52.88 & 52.58 & 53.73 & 53.48 & 53.18 & 52.39 & 51.54 & 51.94 & 4.28 \\
Zoom & 49.47 & 48.06 & 48.43 & 49.85 & 48.51 & 48.91 & 49.00 & 47.11 & 48.32 & 3.53 \\
Snow & 50.49 & 49.15 & 49.45 & 50.55 & 49.12 & 49.63 & 49.90 & 48.52 & 49.33 & 3.97 \\
Frost & 54.45 & 53.03 & 53.42 & 54.13 & 52.84 & 53.23 & 53.32 & 51.82 & 52.68 & 3.90 \\
Fog & 42.08 & 40.97 & 41.11 & 42.22 & 41.12 & 41.37 & 41.86 & 40.53 & 41.32 & 4.00 \\
Brightness & 34.78 & 34.10 & 34.14 & 34.96 & 34.23 & 34.30 & 34.39 & 32.99 & 33.81 & 3.83 \\
Contrast & 52.00 & 50.17 & 50.75 & 51.64 & 50.08 & 50.55 & 50.84 & 50.08 & 50.31 & 4.09 \\
Elastic & 44.23 & 42.95 & 43.21 & 44.04 & 42.89 & 43.12 & 43.67 & 42.28 & 43.11 & 4.16 \\
Pixelate & 39.79 & 38.81 & 38.97 & 39.94 & 39.03 & 39.20 & 39.43 & 38.20 & 38.84 & 4.16 \\
JPEG & 43.52 & 42.49 & 42.69 & 43.56 & 42.36 & 42.58 & 42.98 & 41.99 & 42.52 & 4.17 \\
Mean & 50.8223 & 49.4481 & 49.8140 & 50.9057 & 49.5675 & 49.9233 & 50.1383 & 48.7289 & 49.5487 & 3.9990 \\
\end{longtable}
}
\space
% END inlined appendix_materials/empirical/tables/imagenet-control.tex

\subsection{System endpoints and external comparisons}

Table~\ref{tab:emp-reported-comparison} compares complete-system scores with published results. These include active labeling with EATTA on ImageNet-C/K~\citep{wang2025eatta} and TDA, a training-free adaptation method for CLIP, on ImageNet-A~\citep{karmanov2024tda}. For natural-domain CLIP results, the references use weight-average test-time adaptation (WATT), test-time prompt tuning (TPT), and entropy minimization (TENT)~\citep{osowiechi2024watt,shu2022tpt,wang2021tent}. CIFAR also includes active-feedback SimATTA~\citep{gui2024simatta}. Our local selection comparisons use paired predictions on the same evaluation population.

The label-free OfficeHome endpoint also lies 4.6435 percentage points above the reported 84.20\% CLIPTTA ViT-B/16 result~\citep{lafon2025cliptta}. This cross-backbone comparison provides context for the full system's score. The ViT-L/14 reference in the table is 86.72\%, giving a within-backbone margin of 2.1235 points.

% BEGIN inlined appendix_materials/empirical/tables/reported-comparison.tex
{
\small
\setlength{\tabcolsep}{2pt}\setlength{\LTpre}{3pt}\setlength{\LTpost}{4pt}
\setlength{\LTcapwidth}{\linewidth}
\begin{longtable}{@{}llllrrr@{}}
\caption{Direct fixed-intervention endpoints and reported literature values. ImageNet-C/K use EATTA with three labels per batch; the visual-language rows use zero target labels. Margins orient improvement positively. The ImageNet-A endpoint averages two declared runs. CIFAR fixed endpoints use zero labels, while their SimATTA references use active feedback.}\label{tab:emp-reported-comparison}\\
\toprule
Benchmark & Backbone & Reference & Metric & Reported & Fixed & $\Delta$ \\
\midrule
\endfirsthead
\caption[]{Direct fixed-intervention endpoints and reported literature values (continued).}\\
\toprule
Benchmark & Backbone & Reference & Metric & Reported & Fixed & $\Delta$ \\
\midrule
\endhead
\bottomrule
\endfoot
ImageNet-C & ResNet-50 & EATTA & Error & 50.90 & 49.3395 & 1.5605 \\
ImageNet-K & ResNet-50 & EATTA & Error & 61.40 & 61.1252 & 0.2748 \\
ImageNet-A & ViT-B/16 & TDA & Acc. & 60.11 & 60.2000 & 0.0900 \\
PACS & ViT-L/14 & WATT & Acc. & 97.51 & 98.5171 & 1.0071 \\
OfficeHome & ViT-L/14 & TPT & Acc. & 86.72 & 88.8435 & 2.1235 \\
VLCS & ViT-L/14 & TENT & Acc. & 80.20 & 83.7927 & 3.5927 \\
CIFAR10-C & ResNet18 & SimATTA & Acc. & 83.13 & 79.9218 & -3.2082 \\
CIFAR100-C & ResNet18 & SimATTA & Acc. & 55.86 & 54.7544 & -1.1056 \\
\end{longtable}
}
\space
% END inlined appendix_materials/empirical/tables/reported-comparison.tex

\section{Learning a stronger candidate from past feedback}
\label{app:empirical-head}

The online-head experiment expands the PACS portfolio with a candidate that changes as labels accumulate. The Photo-source ResNet50 from test-time distillation~\citep{chen2026ttd} supplies the fixed BF default and its original classifier-input features. The three target domains contain 8,321 images, of which 4,166 belong to $Q$ and 4,155 to $E$. Each of five complete orders contains 131 batches and purchases 393 distinct $Q$ labels. Head state resets at a domain boundary.

Normalize each feature as $x=f/\lVert f\rVert_2$. Let $X$ contain past queried features and $Y$ their one-hot class labels. The residual targets are $R=Y-\operatorname{softmax}(z_{\mathrm{BF}})$. With regularization strength $\lambda$ and residual scale $s$, the learned output is
\begin{equation}
 z_{\mathrm{learned}}(x)=z_{\mathrm{BF}}(x)
 +s\,xX^\top(XX^\top+\lambda I)^{-1}R.
 \label{eq:empirical-online-head}
\end{equation}
The primary setting uses $\lambda=1$ and $s=4$. Current labels choose between predictions already formed from past feedback, then enter the residual fit for subsequent batches.

Always using the learned candidate gives approximately 73\% accuracy across query policies, compared with 59.6775\% for BF (Table~\ref{tab:emp-head-primary}). Uniform three-label selection retains 6.2335 points of this improvement, disagreement-based selection retains 10.6089 points, and \PORTAL{} retains 12.9627 points. \PORTAL{} stays 0.3851 points below always using its learned candidate. Learning supplies the main gain; selection determines how much reaches deployment.

\PORTAL{} uses an average of 70.8 labels for current decisions. Uniform queries fill the remaining training allowance, giving the same total of 393 labels as the simple policies. Each policy uses its own acquired labels and learned state. Table~\ref{tab:emp-head-sensitivity} reports the primary setting and every prespecified regularization and scale comparison.

% BEGIN inlined appendix_materials/empirical/tables/head-primary.tex
{
\small
\setlength{\tabcolsep}{2pt}\setlength{\LTpre}{3pt}\setlength{\LTpost}{4pt}
\setlength{\LTcapwidth}{\linewidth}
\begin{longtable}{@{}llrrrrr@{}}
\caption{Causal PACS residual-head learning with ridge coefficient 1 and logit scale 4. Each query policy generates its own full learning trajectory over five orders. Decision labels are part of the total 393 labels, and unused joint-decision slots train the next-batch head. Fixed BF alone uses zero target labels.}\label{tab:emp-head-primary}\\
\toprule
Query & Output & Accuracy (\%) & $\Delta$ BF & CE & Decision Q & Total Q \\
\midrule
\endfirsthead
\caption[]{Causal PACS residual-head learning with ridge coefficient 1 and logit scale 4 (continued).}\\
\toprule
Query & Output & Accuracy (\%) & $\Delta$ BF & CE & Decision Q & Total Q \\
\midrule
\endhead
\bottomrule
\endfoot
Uniform & Fixed BF & $59.677\pm0.278$ & 0.000 & 1.2205 & 0.0 & 0 \\
Uniform & Always learned & $73.030\pm0.695$ & 13.353 & 0.9430 & 393.0 & 393 \\
Uniform & Selected output & $65.911\pm0.493$ & 6.233 & 1.0918 & 393.0 & 393 \\
Disagreement & Fixed BF & $59.677\pm0.278$ & 0.000 & 1.2205 & 0.0 & 0 \\
Disagreement & Always learned & $72.982\pm0.779$ & 13.304 & 0.9424 & 393.0 & 393 \\
Disagreement & Selected output & $70.286\pm0.759$ & 10.609 & 0.9999 & 393.0 & 393 \\
\EmpJoint & Fixed BF & $59.677\pm0.278$ & 0.000 & 1.2205 & 0.0 & 0 \\
\EmpJoint & Always learned & $73.025\pm0.972$ & 13.348 & 0.9390 & 70.8 & 393 \\
\EmpJoint & Selected output & $72.640\pm1.048$ & 12.963 & 0.9553 & 70.8 & 393 \\
\end{longtable}
}
\space
% END inlined appendix_materials/empirical/tables/head-primary.tex
% BEGIN inlined appendix_materials/empirical/tables/head-sensitivity.tex
{
\small
\setlength{\tabcolsep}{2pt}\setlength{\LTpre}{3pt}\setlength{\LTpost}{4pt}
\setlength{\LTcapwidth}{\linewidth}
\begin{longtable}{@{}rrlrrr@{}}
\caption{All residual-head sensitivity settings fixed before evaluation on PACS. Every row uses five complete orders and 393 total labels per order. The primary setting remains lambda=1 and scale=4.}\label{tab:emp-head-sensitivity}\\
\toprule
$\lambda$ & Scale & Query & Always learned & Selected & $\Delta$ learned \\
\midrule
\endfirsthead
\caption[]{All residual-head sensitivity settings fixed before evaluation on PACS (continued).}\\
\toprule
$\lambda$ & Scale & Query & Always learned & Selected & $\Delta$ learned \\
\midrule
\endhead
\bottomrule
\endfoot
1 & 4 & Uniform & 73.030 & 65.911 & -7.119 \\
1 & 4 & Disagreement & 72.982 & 70.286 & -2.696 \\
1 & 4 & \EmpJoint & 73.025 & 72.640 & -0.385 \\
0.1 & 4 & Uniform & 75.321 & 68.125 & -7.196 \\
0.1 & 4 & Disagreement & 75.576 & 71.981 & -3.596 \\
0.1 & 4 & \EmpJoint & 75.302 & 74.729 & -0.573 \\
10 & 4 & Uniform & 67.475 & 62.373 & -5.102 \\
10 & 4 & Disagreement & 67.466 & 65.083 & -2.383 \\
10 & 4 & \EmpJoint & 67.114 & 66.503 & -0.611 \\
1 & 1 & Uniform & 65.483 & 60.934 & -4.549 \\
1 & 1 & Disagreement & 65.584 & 64.438 & -1.146 \\
1 & 1 & \EmpJoint & 65.526 & 64.414 & -1.112 \\
1 & 8 & Uniform & 72.606 & 66.416 & -6.190 \\
1 & 8 & Disagreement & 71.957 & 69.049 & -2.907 \\
1 & 8 & \EmpJoint & 72.226 & 72.125 & -0.101 \\
\end{longtable}
}
\space
% END inlined appendix_materials/empirical/tables/head-sensitivity.tex

\section{Immediate output and persistent-state continuation}
\label{app:empirical-state}

Persistent-state effects are measured with two explicit update operators on the Photo-source features. One state is a residual linear classifier $\Theta$ trained by a regularized squared-loss step. The other is a feature-mean cache $\mu$ that changes logits through the released classifier $W$. The output is $z=z_{\mathrm{source}}+x\Theta-\alpha\mu W^\top$, with $\alpha=0.05$. The cache averages raw features, while the residual head uses unit-normalized features.

At each batch, we copy the pre-update state to test three actions: use an update only for the current output, retain a classifier update, or retain a cache update. The first action undoes its update after prediction. All copies then receive identical future queries and updates, with each batch scored before its update. The reference trajectory applies both updates after the current batch. States reset at domain boundaries.

The head step uses learning rate $\eta=0.5$ and regularization $\lambda=0.1$. If $X_t$ contains its three queried unit feature rows, the difference between committed and rollback head states evolves as
\begin{equation}
 \Delta\Theta_{t+1}=A_t\Delta\Theta_t,
 \qquad A_t=(1-\eta\lambda)I-\frac{\eta}{3}X_t^\top X_t.
 \label{eq:empirical-state-head}
\end{equation}
The eigenvalues of $A_t$ lie in $[0.45,0.95]$. The cache uses $\mu_{t+1}=0.9\mu_t+0.1\bar f_{Q,t}$, so its branch difference contracts exactly by 0.9. These operators determine the direction and scale of later logit changes.

Table~\ref{tab:emp-state-continuation} measures horizons of 1, 4 and 16 future batches. There are respectively 128, 119 and 83 valid within-domain branch points per order. At one-step branch points, the head improves immediate accuracy by 0.2296 points but changes next-batch accuracy by $-0.1972$ points. Next-batch cross-entropy nevertheless decreases. The same immediate action can therefore have a different future accuracy effect, and cross-entropy and zero-one loss capture different aspects of that continuation.

For cross-entropy, the directional expansion satisfies
\begin{equation}
 L\leq\ell_y(z+\Delta z)-\ell_y(z)
 \leq L+\tfrac14\lVert\Delta z\rVert_2^2,
 \quad L=(\operatorname{softmax}(z)-e_y)^\top\Delta z.
 \label{eq:empirical-ce-interval}
\end{equation}
The linear term uses evaluation labels to measure the loss change retrospectively. The exact state recurrence supplies the logit difference. Undoing the current update gives zero future effect. A hard prediction can change only if the baseline top-two logit margin is at most $2\lVert\Delta z\rVert_\infty$. We measure accuracy changes directly. Future windows overlap, and the paired averages describe these shared trajectories.

% BEGIN inlined appendix_materials/empirical/tables/state-continuation.tex
{
\small
\setlength{\tabcolsep}{2pt}\setlength{\LTpre}{3pt}\setlength{\LTpost}{4pt}
\setlength{\LTcapwidth}{\linewidth}
\begin{longtable}{@{}lrrrrrrr@{}}
\caption{Persistent-state continuation on PACS. Accuracy changes are in percentage points. Future cross-entropy (CE) change, directional linear term L and remainder ceiling R are multiplied by 1,000. The actual CE change lies in [L,L+R] before averaging. Changes counts hard-prediction differences over all valid overlapping comparisons per order.}\label{tab:emp-state-continuation}\\
\toprule
State & $H$ & $\Delta$ now & $\Delta$ future & $\Delta$ CE & $L$ & $R$ & Changes \\
\midrule
\endfirsthead
\caption[]{Persistent-state continuation on PACS (continued).}\\
\toprule
State & $H$ & $\Delta$ now & $\Delta$ future & $\Delta$ CE & $L$ & $R$ & Changes \\
\midrule
\endhead
\bottomrule
\endfoot
Head commit & 1 & 0.2296 & -0.1972 & -0.5907 & -2.6530 & 8.1582 & 174.2 \\
Head commit & 4 & 0.2038 & -0.0578 & -0.7994 & -1.5634 & 3.0299 & 333.6 \\
Head commit & 16 & 0.2313 & -0.0169 & -0.4486 & -0.6574 & 0.8191 & 302.8 \\
Cache commit & 1 & 0.0293 & 0.0148 & -0.2187 & -0.2212 & 0.0115 & 7.6 \\
Cache commit & 4 & 0.0366 & 0.0118 & -0.1917 & -0.1935 & 0.0086 & 23.4 \\
Cache commit & 16 & 0.0372 & 0.0023 & -0.1503 & -0.1510 & 0.0036 & 40.4 \\
\end{longtable}
}
\space
% END inlined appendix_materials/empirical/tables/state-continuation.tex

\section{Implementation and resource accounting}
\label{app:empirical-resources}

The CIFAR experiment retains four float32 logits and four corresponding 512-dimensional feature vectors per image. The linear classifier maps averaged original/flip features to the matching averaged logits. CLIP inference processes eight views sequentially at batch size 16. Its saved outputs retain the original features, text prototypes and individual view log probabilities. Selection operates on batches of 64. The online ridge head operates in double precision on normalized 2,048-dimensional features and stores at most 186 queried supports in a target domain.

% BEGIN inlined appendix_materials/empirical/tables/component-cost.tex
{
\small
\setlength{\tabcolsep}{2pt}\setlength{\LTpre}{3pt}\setlength{\LTpost}{4pt}
\setlength{\LTcapwidth}{\linewidth}
\begin{longtable}{@{}lrrrr@{}}
\caption{Measured OfficeHome CLIP ViT-L/14 component costs, batch size 8. Forward columns are ms/image; decision is ms/batch after action logits exist. Each domain contributes 64 measured images. Memory is peak allocation in MiB.}\label{tab:emp-component-cost}\\
\toprule
Domain & Single & Eight-view & Decision & Eight-view MiB \\
\midrule
\endfirsthead
\caption[]{Measured OfficeHome CLIP ViT-L/14 component costs, batch size 8 (continued).}\\
\toprule
Domain & Single & Eight-view & Decision & Eight-view MiB \\
\midrule
\endhead
\bottomrule
\endfoot
Art & 2.354 & 16.669 & 0.507 & 2616.9 \\
Clipart & 2.084 & 16.732 & 0.546 & 2617.6 \\
Product & 2.093 & 16.784 & 0.633 & 2617.6 \\
Real World & 2.100 & 16.809 & 0.451 & 2617.6 \\
Mean & 2.158 & 16.749 & 0.534 & 2617.4 \\
\end{longtable}
}
\space
% END inlined appendix_materials/empirical/tables/component-cost.tex

Table~\ref{tab:emp-component-cost} separates visual-intervention cost from decision cost on OfficeHome. The eight-view forward costs 16.7485 ms/image, and the decision arithmetic adds 0.5343 ms per eight-image batch after logits are available. This corresponds to 16.8153 ms/image for the combined path, compared with 2.1577 ms/image for one view. Annotation-response latency is excluded, and feedback is reported separately as three labels per microbatch.

The five complete online-head orders require approximately 0.62 seconds per order for the simple query rules and 0.73 seconds for joint selection on the saved features. Their peak stored support/head arrays occupy 2.93 MiB, with temporary kernel-solver memory additional. The complete CIFAR collector peaks at 134.58 MiB of allocated CUDA memory, while the sequential-view CLIP collection peaks at approximately 1,051 MiB. These are measurements of the implemented schedules: the CLIP component benchmark above retains a different allocation pattern. Backbone evaluation, feature storage, controller arithmetic and target annotations remain separate resource components.
\space
% END inlined sections/expanded_appendix_empirical.tex
% BEGIN inlined sections/expanded_observational_controls.tex
\providecommand{\EmpiricalRoot}{.}
\providecommand{\EmpJoint}{Joint selection}
\section{Fixed-output mechanism and cost controls}
\label{app:empirical-observational}

These PACS controls hold the Photo-source ResNet50 and permanent image roles fixed. The targets are Art, Cartoon and Sketch, with 1,007, 1,175 and 1,973 evaluation images. Five stream orders recompute the original, normalized, flipped and normalized-plus-flipped predictions. Each comparison uses the running-statistic original prediction I as its default and one challenger at a time. This identifies the effect of the candidate component and acquisition rule before introducing the learned head.

For an evaluation image, let $d_i$ be the default classification error minus the selected output's error. With $N$ evaluation images, gain is $G=100\sum_i d_i/N$, the new-error percentage is $S=100\sum_i\mathbf1\{d_i=-1\}/N$, and negative batch contribution is $H_B=100\sum_b[-\sum_{i\in b}d_i]_+/N$. These quantities distinguish individual errors from a batch's net effect. Gains and costs always refer to the same fixed outputs and evaluation population.

\subsection{Candidate components and confidence scale}

Tables~\ref{tab:emp-obs-component-bn_only}--\ref{tab:emp-obs-component-bn_flip} give the three-label decision rules for normalization, horizontal views, and their combination. The count $k$ is the minimum queried signed advantage needed to select the challenger. Ties retain I. The K-Margin query rule scales source features by $1-(p_{(1)}-p_{(2)})$, where $p_{(1)}$ and $p_{(2)}$ are the largest two class probabilities. It greedily selects distant scaled features. Here we compare its queries on the same fixed candidate predictions.

% BEGIN inlined appendix_materials/empirical/tables/obs-component-bn_only.tex
{
\small
\setlength{\tabcolsep}{2pt}\setlength{\LTpre}{3pt}\setlength{\LTpost}{4pt}
\setlength{\LTcapwidth}{\linewidth}
\begin{longtable}{@{}lrrrr@{}}
\caption{Normalization on the frozen Photo-source PACS model. Each acquisition rule uses three query labels and the distinct minimum signed count k. The default is running-statistic I. G and $H_B$ are percentage points; S is the percentage of newly introduced errors.}\label{tab:emp-obs-component-bn_only}\\
\toprule
Acquisition & $k$ & $G$ & $H_B$ & $S$ \\
\midrule
\endfirsthead
\caption[]{Normalization on the frozen Photo-source PACS model (continued).}\\
\toprule
Acquisition & $k$ & $G$ & $H_B$ & $S$ \\
\midrule
\endhead
\bottomrule
\endfoot
Always candidate & - & 7.49 & 2.47 & 11.85 \\
Uniform & 1 & 4.54 & 0.84 & 4.33 \\
Uniform & 2 & 1.58 & 0.26 & 1.17 \\
Uniform & 3 & 0.30 & 0.00 & 0.05 \\
Disagreement & 1 & 5.77 & 1.25 & 5.81 \\
Disagreement & 2 & 3.50 & 0.53 & 2.47 \\
Disagreement & 3 & 1.19 & 0.09 & 0.51 \\
Symmetric KL & 1 & 6.13 & 1.23 & 5.68 \\
Symmetric KL & 2 & 4.48 & 0.55 & 2.38 \\
Symmetric KL & 3 & 2.07 & 0.12 & 0.78 \\
K-Margin & 1 & 5.15 & 1.28 & 6.36 \\
K-Margin & 2 & 2.60 & 0.26 & 1.94 \\
K-Margin & 3 & 0.33 & 0.08 & 0.34 \\
\end{longtable}
}
\space
% END inlined appendix_materials/empirical/tables/obs-component-bn_only.tex
% BEGIN inlined appendix_materials/empirical/tables/obs-component-flip_only.tex
{
\small
\setlength{\tabcolsep}{2pt}\setlength{\LTpre}{3pt}\setlength{\LTpost}{4pt}
\setlength{\LTcapwidth}{\linewidth}
\begin{longtable}{@{}lrrrr@{}}
\caption{Horizontal views on the frozen Photo-source PACS model. Each acquisition rule uses three query labels and the distinct minimum signed count k. The default is running-statistic I. G and $H_B$ are percentage points; S is the percentage of newly introduced errors.}\label{tab:emp-obs-component-flip_only}\\
\toprule
Acquisition & $k$ & $G$ & $H_B$ & $S$ \\
\midrule
\endfirsthead
\caption[]{Horizontal views on the frozen Photo-source PACS model (continued).}\\
\toprule
Acquisition & $k$ & $G$ & $H_B$ & $S$ \\
\midrule
\endhead
\bottomrule
\endfoot
Always candidate & - & 0.02 & 1.19 & 2.07 \\
Uniform & 1 & 0.01 & 0.05 & 0.12 \\
Uniform & 2 & 0.00 & 0.00 & 0.00 \\
Uniform & 3 & 0.00 & 0.00 & 0.00 \\
Disagreement & 1 & -0.04 & 0.43 & 0.68 \\
Disagreement & 2 & 0.03 & 0.07 & 0.12 \\
Disagreement & 3 & 0.01 & 0.01 & 0.02 \\
Symmetric KL & 1 & -0.00 & 0.28 & 0.51 \\
Symmetric KL & 2 & 0.00 & 0.03 & 0.05 \\
Symmetric KL & 3 & 0.00 & 0.00 & 0.00 \\
K-Margin & 1 & 0.03 & 0.17 & 0.31 \\
K-Margin & 2 & -0.02 & 0.02 & 0.03 \\
K-Margin & 3 & 0.00 & 0.00 & 0.00 \\
\end{longtable}
}
\space
% END inlined appendix_materials/empirical/tables/obs-component-flip_only.tex
% BEGIN inlined appendix_materials/empirical/tables/obs-component-bn_flip.tex
{
\small
\setlength{\tabcolsep}{2pt}\setlength{\LTpre}{3pt}\setlength{\LTpost}{4pt}
\setlength{\LTcapwidth}{\linewidth}
\begin{longtable}{@{}lrrrr@{}}
\caption{Normalization and views on the frozen Photo-source PACS model. Each acquisition rule uses three query labels and the distinct minimum signed count k. The default is running-statistic I. G and $H_B$ are percentage points; S is the percentage of newly introduced errors.}\label{tab:emp-obs-component-bn_flip}\\
\toprule
Acquisition & $k$ & $G$ & $H_B$ & $S$ \\
\midrule
\endfirsthead
\caption[]{Normalization and views on the frozen Photo-source PACS model (continued).}\\
\toprule
Acquisition & $k$ & $G$ & $H_B$ & $S$ \\
\midrule
\endhead
\bottomrule
\endfoot
Always candidate & - & 8.80 & 2.00 & 11.18 \\
Uniform & 1 & 5.15 & 0.72 & 4.30 \\
Uniform & 2 & 1.86 & 0.19 & 1.18 \\
Uniform & 3 & 0.33 & 0.00 & 0.10 \\
Disagreement & 1 & 7.03 & 0.69 & 5.16 \\
Disagreement & 2 & 4.26 & 0.24 & 2.31 \\
Disagreement & 3 & 1.53 & 0.03 & 0.62 \\
Symmetric KL & 1 & 7.25 & 1.15 & 6.15 \\
Symmetric KL & 2 & 5.20 & 0.51 & 2.66 \\
Symmetric KL & 3 & 2.78 & 0.12 & 0.90 \\
K-Margin & 1 & 6.08 & 1.03 & 6.08 \\
K-Margin & 2 & 3.10 & 0.24 & 1.83 \\
K-Margin & 3 & 0.48 & 0.04 & 0.31 \\
\end{longtable}
}
\space
% END inlined appendix_materials/empirical/tables/obs-component-bn_flip.tex

Positive temperature scaling preserves every hard prediction but changes the probability discrepancies used for acquisition. Table~\ref{tab:emp-obs-temperature} compares all five tested temperatures. At temperature 0.25, the selected query set has 68.40\% mean overlap with temperature one, and 19.54\% of batch decisions change. Since both candidates' zero-one losses are fixed, the resulting differences arise through which labels are acquired.

% BEGIN inlined appendix_materials/empirical/tables/obs-temperature.tex
{
\small
\setlength{\tabcolsep}{2pt}\setlength{\LTpre}{3pt}\setlength{\LTpost}{4pt}
\setlength{\LTcapwidth}{\linewidth}
\begin{longtable}{@{}rrrrrr@{}}
\caption{Positive-temperature acquisition control with fixed hard predictions. The candidate is BF, the budget is three, and k=1. Overlap and decision changes use temperature one as the reference.}\label{tab:emp-obs-temperature}\\
\toprule
$T$ & $G$ & $H_B$ & $S$ & Overlap (\%) & Changed (\%) \\
\midrule
\endfirsthead
\caption[]{Positive-temperature acquisition control with fixed hard predictions (continued).}\\
\toprule
$T$ & $G$ & $H_B$ & $S$ & Overlap (\%) & Changed (\%) \\
\midrule
\endhead
\bottomrule
\endfoot
0.25 & 7.50 & 1.40 & 7.54 & 68.40 & 19.54 \\
0.5 & 7.37 & 1.34 & 7.22 & 77.51 & 13.74 \\
1 & 7.25 & 1.15 & 6.15 & 100.00 & 0.00 \\
2 & 7.10 & 1.01 & 5.78 & 82.49 & 9.31 \\
4 & 6.85 & 1.05 & 5.90 & 76.18 & 13.59 \\
\end{longtable}
}
\space
% END inlined appendix_materials/empirical/tables/obs-temperature.tex

The score-bin comparison in Table~\ref{tab:emp-obs-score-bins} uses the mean score of the three selected query points to form quintiles. One construction pools domains; the other forms bins within each domain and then combines corresponding bins. Query advantage and evaluation advantage are calculated after bin assignment. Their separation shows how a query score can organize observed disagreement without directly determining the evaluation benefit of deploying the challenger.

% BEGIN inlined appendix_materials/empirical/tables/obs-score-bins.tex
{
\small
\setlength{\tabcolsep}{2pt}\setlength{\LTpre}{3pt}\setlength{\LTpost}{4pt}
\setlength{\LTcapwidth}{\linewidth}
\begin{longtable}{@{}lrrrrrr@{}}
\caption{Symmetric-KL score quintiles for BF, from low to high. Query and evaluation gains use their own image populations. Negative is the percentage of batches where BF increases evaluation errors. Bin membership depends on the label-free query score.}\label{tab:emp-obs-score-bins}\\
\toprule
Binning & Bin & Batches & $n_E$ & Query gain & E gain & Negative (\%) \\
\midrule
\endfirsthead
\caption[]{Symmetric-KL score quintiles for BF, from low to high (continued).}\\
\toprule
Binning & Bin & Batches & $n_E$ & Query gain & E gain & Negative (\%) \\
\midrule
\endhead
\bottomrule
\endfoot
All domains & 1 & 27 & 849.4 & 14.07 & 6.10 & 17.04 \\
All domains & 2 & 26 & 829.2 & 45.64 & 17.16 & 10.00 \\
All domains & 3 & 26 & 829.2 & 33.08 & 10.60 & 26.92 \\
All domains & 4 & 26 & 842.4 & 21.79 & 4.75 & 38.46 \\
All domains & 5 & 26 & 804.8 & 33.85 & 5.34 & 33.08 \\
Within domain & 1 & 28 & 900.4 & 25.71 & 9.23 & 22.86 \\
Within domain & 2 & 28 & 903.4 & 24.29 & 8.88 & 27.14 \\
Within domain & 3 & 25 & 800.0 & 26.40 & 8.61 & 24.80 \\
Within domain & 4 & 25 & 775.0 & 35.47 & 7.57 & 30.40 \\
Within domain & 5 & 25 & 776.2 & 37.07 & 9.64 & 20.00 \\
\end{longtable}
}
\space
% END inlined appendix_materials/empirical/tables/obs-score-bins.tex

\subsection{Distinct budgets and empirical rule mixtures}

Table~\ref{tab:emp-obs-budget-grid} retains all distinct measured count rules at budgets 1, 3, 5 and 10. Rules with equivalent thresholds appear once. A budget-$m$ rule spends $131m$ labels per complete order. The hard-cap-three library contains its fixed outputs, all distinct rules with budgets at most three, and the measured three-label K-Margin rules. The expected-budget library additionally admits the measured five- and ten-label rules.

% BEGIN inlined appendix_materials/empirical/tables/obs-budget-grid.tex
{
\small
\setlength{\tabcolsep}{2pt}\setlength{\LTpre}{3pt}\setlength{\LTpost}{4pt}
\setlength{\LTcapwidth}{\linewidth}
\begin{longtable}{@{}rrrrrrrrrrr@{}}
\caption{All distinct measured frozen-PACS budget/count rules. Each acquisition block reports gain G, negative batch contribution $H_B$, and sample-new-error percentage S. Thresholds that induce the same signed count are represented once.}\label{tab:emp-obs-budget-grid}\\
\toprule
& & \multicolumn{3}{c}{Uniform} & \multicolumn{3}{c}{Disagreement} & \multicolumn{3}{c}{Symmetric KL} \\
$m$ & $k$ & $G$ & $H_B$ & $S$ & $G$ & $H_B$ & $S$ & $G$ & $H_B$ & $S$ \\
\midrule
\endfirsthead
\caption[]{All distinct measured frozen-PACS budget/count rules (continued).}\\
\toprule
& & \multicolumn{3}{c}{Uniform} & \multicolumn{3}{c}{Disagreement} & \multicolumn{3}{c}{Symmetric KL} \\
$m$ & $k$ & $G$ & $H_B$ & $S$ & $G$ & $H_B$ & $S$ & $G$ & $H_B$ & $S$ \\
\midrule
\endhead
\bottomrule
\endfoot
1 & 1 & 2.42 & 0.40 & 2.22 & 5.05 & 0.51 & 3.89 & 5.85 & 1.00 & 5.36 \\
3 & 1 & 5.15 & 0.72 & 4.30 & 7.03 & 0.69 & 5.16 & 7.25 & 1.15 & 6.15 \\
3 & 2 & 1.86 & 0.19 & 1.18 & 4.26 & 0.24 & 2.31 & 5.20 & 0.51 & 2.66 \\
3 & 3 & 0.33 & 0.00 & 0.10 & 1.53 & 0.03 & 0.62 & 2.78 & 0.12 & 0.90 \\
5 & 1 & 6.51 & 0.74 & 4.87 & 7.56 & 1.01 & 6.26 & 7.63 & 1.20 & 6.80 \\
5 & 2 & 3.71 & 0.38 & 2.23 & 6.04 & 0.47 & 3.63 & 6.72 & 0.53 & 3.53 \\
5 & 3 & 1.33 & 0.07 & 0.60 & 4.06 & 0.13 & 1.66 & 4.90 & 0.21 & 1.76 \\
5 & 4 & 0.23 & 0.03 & 0.13 & 2.22 & 0.07 & 0.73 & 2.82 & 0.08 & 0.68 \\
10 & 1 & 7.47 & 0.98 & 6.01 & 8.18 & 1.13 & 6.98 & 8.24 & 1.16 & 7.20 \\
10 & 3 & 3.69 & 0.37 & 2.06 & 6.68 & 0.52 & 3.41 & 6.61 & 0.47 & 3.20 \\
10 & 6 & 0.52 & 0.00 & 0.05 & 2.75 & 0.04 & 0.75 & 3.23 & 0.02 & 0.56 \\
10 & 8 & 0.00 & 0.00 & 0.00 & 0.56 & 0.00 & 0.11 & 0.66 & 0.00 & 0.11 \\
\end{longtable}
}
\space
% END inlined appendix_materials/empirical/tables/obs-budget-grid.tex

A mixture draws its rule before acquiring labels. Its expected gain, cost and query count are the corresponding convex combinations of measured operating points. Table~\ref{tab:emp-obs-mixtures} gives two fitted examples. Mixing the disagreement rule with fixed BF yields gain 7.6495 points at the same negative batch cost as the 7.2491-point symmetric-KL rule, using 1.9522 expected labels per batch. The mixture introduces more individual errors, so this comparison changes when S is the constrained cost. The second mixture uses two distinct disagreement counts. These fitted values characterize the observed rule library.

% BEGIN inlined appendix_materials/empirical/tables/obs-mixtures.tex
{
\small
\setlength{\tabcolsep}{2pt}\setlength{\LTpre}{3pt}\setlength{\LTpost}{4pt}
\setlength{\LTcapwidth}{\linewidth}
\begin{longtable}{@{}llrrrr@{}}
\caption{Expected outcomes of two fitted rule mixtures and their symmetric-KL targets at matched negative batch cost. The mixtures use high-endpoint probabilities 0.3492647 and 0.5957447, respectively.}\label{tab:emp-obs-mixtures}\\
\toprule
Mixture family & Output & $G$ & $H_B$ & $S$ & Labels/batch \\
\midrule
\endfirsthead
\caption[]{Expected outcomes of two fitted rule mixtures and their symmetric-KL targets at matched negative batch cost (continued).}\\
\toprule
Mixture family & Output & $G$ & $H_B$ & $S$ & Labels/batch \\
\midrule
\endhead
\bottomrule
\endfoot
Disagreement + fixed BF & Mixture & 7.6495 & 1.1456 & 7.2663 & 1.9522 \\
Disagreement + fixed BF & Symmetric KL & 7.2491 & 1.1456 & 6.1468 & 3.0000 \\
Two disagreement counts & Mixture & 5.9117 & 0.5054 & 4.0110 & 3.0000 \\
Two disagreement counts & Symmetric KL & 5.1986 & 0.5054 & 2.6619 & 3.0000 \\
\end{longtable}
}
\space
% END inlined appendix_materials/empirical/tables/obs-mixtures.tex

For Table~\ref{tab:emp-obs-frontier}, the alternative library excludes symmetric-KL rules and retains K-Margin. Each of the $5^5=3,125$ ordered stream resamples recomputes gains and costs, then refits the frontier at the resampled symmetric-KL cost. The resulting percentile ranges include variation of the fitted probabilities. They condition on the common images, source checkpoint and measured rules. The fixed-I endpoint keeps the complete frontier feasible when an interpolation between two particular rules cannot attain the target cost.

% BEGIN inlined appendix_materials/empirical/tables/obs-frontier.tex
{
\small
\setlength{\tabcolsep}{2pt}\setlength{\LTpre}{3pt}\setlength{\LTpost}{4pt}
\setlength{\LTcapwidth}{\linewidth}
\begin{longtable}{@{}lrrrr@{}}
\caption{Symmetric-KL gain minus the empirical frontier without symmetric-KL rules, including K-Margin in the alternative library. Every one of the 3,125 ordered stream resamples refits the frontier. Bounds are the 2.5th and 97.5th descriptive percentiles, in percentage points.}\label{tab:emp-obs-frontier}\\
\toprule
Cost & $k$ & Point gap & 2.5\% & 97.5\% \\
\midrule
\endfirsthead
\caption[]{Symmetric-KL gain minus the empirical frontier without symmetric-KL rules, including K-Margin in the alternative library (continued).}\\
\toprule
Cost & $k$ & Point gap & 2.5\% & 97.5\% \\
\midrule
\endhead
\bottomrule
\endfoot
Batch net cost & 1 & -0.4004 & -0.5453 & -0.2555 \\
Batch net cost & 2 & -0.7131 & -1.5901 & 0.7030 \\
Batch net cost & 3 & 0.1090 & -2.5777 & 1.2591 \\
New errors & 1 & -0.0717 & -0.2560 & 0.1362 \\
New errors & 2 & 0.5973 & -0.1205 & 1.1252 \\
New errors & 3 & 0.7927 & 0.0611 & 1.2920 \\
\end{longtable}
}
\space
% END inlined appendix_materials/empirical/tables/obs-frontier.tex

\subsection{Scoring partitions and natural groups}

Table~\ref{tab:emp-obs-grouping} changes the unit at which negative contributions are accumulated. Every image keeps its prediction and every batch keeps its decision. Coarsening the scoring partition allows positive and negative contributions to cancel, reducing $H_B$ while leaving G and S unchanged. This control isolates the scoring definition from a change in the network's input batch size.

% BEGIN inlined appendix_materials/empirical/tables/obs-grouping.tex
{
\small
\setlength{\tabcolsep}{2pt}\setlength{\LTpre}{3pt}\setlength{\LTpost}{4pt}
\setlength{\LTcapwidth}{\linewidth}
\begin{longtable}{@{}lrrrrrrr@{}}
\caption{Negative contribution when the same selected outputs are scored at different grouping scales. Only the scoring partition changes. Individual-image groups yield S, and original-batch groups yield $H_B$.}\label{tab:emp-obs-grouping}\\
\toprule
Output & Image & Batch & 2 batches & 4 batches & 8 batches & Domain & Stream \\
\midrule
\endfirsthead
\caption[]{Negative contribution when the same selected outputs are scored at different grouping scales (continued).}\\
\toprule
Output & Image & Batch & 2 batches & 4 batches & 8 batches & Domain & Stream \\
\midrule
\endhead
\bottomrule
\endfoot
Always BF & 11.1817 & 1.9976 & 1.2130 & 0.7220 & 0.3995 & 0.0000 & 0.0000 \\
Disagreement & 5.1649 & 0.6883 & 0.5247 & 0.3418 & 0.1877 & 0.0000 & 0.0000 \\
Symmetric KL & 6.1468 & 1.1456 & 0.8857 & 0.6691 & 0.4043 & 0.0193 & 0.0000 \\
\end{longtable}
}
\space
% END inlined appendix_materials/empirical/tables/obs-grouping.tex

Table~\ref{tab:emp-obs-natural-groups} retains all 21 domain-by-class groups. The same frozen BF intervention that improves overall accuracy can worsen a particular group's accuracy. Matching identities across the five orders also identifies repeated errors: fixed BF harms 330 evaluation images in all five orders, compared with 35 under symmetric-KL selection and 15 under the reference disagreement rule. These counts describe recurring outcomes on the shared evaluation images.

% BEGIN inlined appendix_materials/empirical/tables/obs-natural-groups.tex
{
\small
\setlength{\tabcolsep}{2pt}\setlength{\LTpre}{3pt}\setlength{\LTpost}{4pt}
\setlength{\LTcapwidth}{\linewidth}
\begin{longtable}{@{}llrrrr@{}}
\caption{All domain-by-class gains on the permanent PACS evaluation population. Each group denominator is fixed across the five orders. The two query rules use three labels and k=1.}\label{tab:emp-obs-natural-groups}\\
\toprule
Domain & Class & $n_E$ & Always BF & Disagreement & Symmetric KL \\
\midrule
\endfirsthead
\caption[]{All domain-by-class gains on the permanent PACS evaluation population (continued).}\\
\toprule
Domain & Class & $n_E$ & Always BF & Disagreement & Symmetric KL \\
\midrule
\endhead
\bottomrule
\endfoot
Art & dog & 183 & -6.23 & -3.72 & -2.40 \\
Art & elephant & 123 & 9.43 & 5.20 & 4.88 \\
Art & giraffe & 131 & -10.08 & -5.19 & -3.66 \\
Art & guitar & 91 & -4.84 & -2.86 & -2.42 \\
Art & horse & 92 & -3.26 & -2.17 & -2.17 \\
Art & house & 149 & 1.48 & 0.94 & 0.81 \\
Art & person & 238 & 29.83 & 17.56 & 13.61 \\
Cartoon & dog & 218 & 8.62 & 6.88 & 8.35 \\
Cartoon & elephant & 222 & 31.17 & 25.68 & 30.09 \\
Cartoon & giraffe & 172 & 28.84 & 22.67 & 27.21 \\
Cartoon & guitar & 65 & -3.38 & -3.38 & -3.38 \\
Cartoon & horse & 168 & 10.36 & 8.33 & 9.40 \\
Cartoon & house & 137 & 5.84 & 4.53 & 5.69 \\
Cartoon & person & 193 & 60.41 & 50.88 & 56.06 \\
Sketch & dog & 381 & 14.17 & 5.98 & 6.61 \\
Sketch & elephant & 371 & 43.56 & 19.57 & 22.59 \\
Sketch & giraffe & 392 & -34.23 & -12.14 & -16.73 \\
Sketch & guitar & 313 & -20.45 & -8.31 & -11.69 \\
Sketch & horse & 393 & -8.85 & -2.85 & -4.99 \\
Sketch & house & 35 & 78.29 & 34.29 & 36.00 \\
Sketch & person & 88 & 28.86 & 12.50 & 15.68 \\
\end{longtable}
}
\space
% END inlined appendix_materials/empirical/tables/obs-natural-groups.tex
\space
% END inlined sections/expanded_observational_controls.tex
\newcommand{\ExpandedVLMRoot}{appendix_materials/expanded_vlm}
% BEGIN inlined sections/expanded_vlm_appendix.tex
\providecommand{\ExpandedVLMRoot}{.}
\providecommand{\EmpJoint}{Joint selection}

\clearpage
\section{Gaussian readout adaptation and portfolio expansion}
\label{app:expanded-vlm}

\subsection{Episodic StatA on shared CLIP representations}

StatA~\citep{portal2025_realistic_vlm} adapts a Gaussian classifier initialized from text prototypes. It combines class assignments, a feature-neighbor graph, and updates to class means and covariances. We use the authors' solver with original-image CLIP ViT-L/14 features and the same eight-template prototypes as the fixed readouts. Each batch starts from these prototypes. The data include all 9,991 PACS and 7,510 VLCS instances, in five stream orders with batches of 64. StatA sees the full batch's unlabeled features and uses zero target labels. We score both the full data and the permanent evaluation set $E$.
% Author solver revision used for reproducibility: 2635f935.

StatA's anchor strength $\alpha$, initial-assignment weight $\lambda_{\hat y}$, and graph regularization weight $\lambda_{\mathrm{Laplacian}}$ are all 1. Its mixing coefficient $\beta$ balances batch-estimated statistics against the text anchor and is updated from hard class counts. We run ten mean/covariance iterations with three neighbors. The final three-image LABELME batch uses two neighbors. Table~\ref{tab:expanded-stata-summary} gives full-data and common-$E$ performance; Tables~\ref{tab:expanded-stata-domain-full} and~\ref{tab:expanded-stata-domain-e} give every domain.

On VLCS, StatA raises full-data equal-domain accuracy from the fixed U8V readout's $83.793\%$ to $84.754\%$. On common $E$, the corresponding scores are $84.097\%$ and $85.106\%$. The LABELME improvement contributes strongly to this result. PACS exhibits a different ordering: StatA attains $98.445\%$ on the full data and $98.404\%$ on $E$, compared with U8V's $98.517\%$ and $98.437\%$. This variation motivates retaining both readouts in the portfolio.

% BEGIN inlined appendix_materials/expanded_vlm/tables/stata-summary.tex
\begin{table}[h]
\centering\small
\caption{Episodic StatA and fixed CLIP readouts. Accuracy is equal-domain mean $\pm$ stream-order SD. Fixed readouts have no order variation; all rows use zero target labels.}
\label{tab:expanded-stata-summary}
\begin{tabular}{llrrr}
\toprule
Dataset & Readout & Full data (\%) & Common $E$ (\%) & Labels \\
\midrule
PACS & Prompt average & 98.393 $\pm$ 0.000 & 98.279 $\pm$ 0.000 & 0 \\
PACS & Eight views & 98.517 $\pm$ 0.000 & 98.437 $\pm$ 0.000 & 0 \\
PACS & Mean logits & 98.502 $\pm$ 0.000 & 98.390 $\pm$ 0.000 & 0 \\
PACS & StatA & 98.445 $\pm$ 0.051 & 98.404 $\pm$ 0.063 & 0 \\
VLCS & Prompt average & 83.099 $\pm$ 0.000 & 83.523 $\pm$ 0.000 & 0 \\
VLCS & Eight views & 83.793 $\pm$ 0.000 & 84.097 $\pm$ 0.000 & 0 \\
VLCS & Mean logits & 83.176 $\pm$ 0.000 & 83.650 $\pm$ 0.000 & 0 \\
VLCS & StatA & 84.754 $\pm$ 0.081 & 85.106 $\pm$ 0.137 & 0 \\
\bottomrule
\end{tabular}
\end{table}
\space
% END inlined appendix_materials/expanded_vlm/tables/stata-summary.tex
% BEGIN inlined appendix_materials/expanded_vlm/tables/stata-domain-full.tex
\begin{table}[h]
\centering\small
\caption{Domain-level StatA comparison on full data. Values are accuracy means $\pm$ stream-order SD over five orders.}
\label{tab:expanded-stata-domain-full}
\begin{tabular}{lllr}
\toprule
Dataset & Domain & Readout & Accuracy (\%) \\
\midrule
PACS & Art & Eight views & 99.170 $\pm$ 0.000 \\
PACS & Art & StatA & 99.219 $\pm$ 0.035 \\
PACS & Cartoon & Eight views & 99.616 $\pm$ 0.000 \\
PACS & Cartoon & StatA & 99.727 $\pm$ 0.038 \\
PACS & Photo & Eight views & 99.940 $\pm$ 0.000 \\
PACS & Photo & StatA & 100.000 $\pm$ 0.000 \\
PACS & Sketch & Eight views & 95.342 $\pm$ 0.000 \\
PACS & Sketch & StatA & 94.833 $\pm$ 0.172 \\
VLCS & Caltech & Eight views & 100.000 $\pm$ 0.000 \\
VLCS & Caltech & StatA & 100.000 $\pm$ 0.000 \\
VLCS & LabelMe & Eight views & 71.382 $\pm$ 0.000 \\
VLCS & LabelMe & StatA & 75.148 $\pm$ 0.101 \\
VLCS & Pascal & Eight views & 86.035 $\pm$ 0.000 \\
VLCS & Pascal & StatA & 85.722 $\pm$ 0.231 \\
VLCS & SUN & Eight views & 77.754 $\pm$ 0.000 \\
VLCS & SUN & StatA & 78.145 $\pm$ 0.392 \\
\bottomrule
\end{tabular}
\end{table}
\space
% END inlined appendix_materials/expanded_vlm/tables/stata-domain-full.tex
% BEGIN inlined appendix_materials/expanded_vlm/tables/stata-domain-e.tex
\begin{table}[h]
\centering\small
\caption{Domain-level StatA comparison on common $E$. Values are accuracy means $\pm$ stream-order SD over five orders.}
\label{tab:expanded-stata-domain-e}
\begin{tabular}{lllr}
\toprule
Dataset & Domain & Readout & Accuracy (\%) \\
\midrule
PACS & Art & Eight views & 99.069 $\pm$ 0.000 \\
PACS & Art & StatA & 99.255 $\pm$ 0.046 \\
PACS & Cartoon & Eight views & 99.570 $\pm$ 0.000 \\
PACS & Cartoon & StatA & 99.725 $\pm$ 0.094 \\
PACS & Photo & Eight views & 100.000 $\pm$ 0.000 \\
PACS & Photo & StatA & 100.000 $\pm$ 0.000 \\
PACS & Sketch & Eight views & 95.108 $\pm$ 0.000 \\
PACS & Sketch & StatA & 94.634 $\pm$ 0.170 \\
VLCS & Caltech & Eight views & 100.000 $\pm$ 0.000 \\
VLCS & Caltech & StatA & 100.000 $\pm$ 0.000 \\
VLCS & LabelMe & Eight views & 71.354 $\pm$ 0.000 \\
VLCS & LabelMe & StatA & 75.803 $\pm$ 0.410 \\
VLCS & Pascal & Eight views & 86.356 $\pm$ 0.000 \\
VLCS & Pascal & StatA & 85.630 $\pm$ 0.384 \\
VLCS & SUN & Eight views & 78.677 $\pm$ 0.000 \\
VLCS & SUN & StatA & 78.990 $\pm$ 0.475 \\
\bottomrule
\end{tabular}
\end{table}
\space
% END inlined appendix_materials/expanded_vlm/tables/stata-domain-e.tex
% BEGIN inlined appendix_materials/expanded_vlm/tables/stata-resources.tex
\begin{table}[h]
\centering\small
\caption{StatA resource measurements from an isolated rerun. Time is seconds per four-domain stream; memory is peak allocated GPU memory. The final three-image batch uses two neighbors. Initialization $|\Delta p|$ is the largest difference between the solver's initial probabilities and the stored prompt-average probabilities.}
\label{tab:expanded-stata-resources}
\begin{tabular}{lrrrr}
\toprule
Dataset & Time (s) & Peak MiB & Final-batch adjustments & Initialization $|\Delta p|$ \\
\midrule
PACS & 7.95 $\pm$ 0.24 & 24.51 & 0 & 1.75e-06 \\
VLCS & 5.92 $\pm$ 0.03 & 20.33 & 5 & 1.83e-06 \\
\bottomrule
\end{tabular}
\end{table}
\space
% END inlined appendix_materials/expanded_vlm/tables/stata-resources.tex

The resource measurements in Table~\ref{tab:expanded-stata-resources} come from a dedicated rerun launched without another GPU process. They cover synchronized solver calls, including feature transfers and final CPU outputs; CLIP feature extraction is outside the timing. PACS and VLCS take $7.95\pm0.24$ and $5.92\pm0.03$ seconds per complete four-domain stream. The outputs and probabilities exactly match the original full run. The reported anchor difference is the maximum absolute gap between the solver's initial probabilities and the cached universal-template probabilities. StatA consumes one original-image embedding per instance, whereas U8V combines eight image views. The Gaussian readout is reset after each batch, so its cost and output depend on batch composition without creating a carried adaptation state.

\subsection{Adding the Gaussian readout as a fifth candidate}

We append StatA to the original single-template, universal-template, U2V, and U8V readouts. Each StatA prediction comes from the exact stream order and batch in which the five-candidate comparison is evaluated. The original four logits are normalized with log-softmax. StatA probabilities are floored at $10^{-12}$ and converted to normalized log probabilities; their hard predictions are preserved. Mean-logit prediction averages these normalized score vectors. U8V remains the declared default, including empirical-error ties, in both portfolios.

Uniform, disagreement-first, and maximum pairwise symmetric-KL selection each use up to three current-batch $Q$ labels. Uniform uses identical query sets for the four- and five-candidate comparisons. The other rules can change their queries when StatA changes the disagreement pattern. \EmpJoint{} keeps prior mass 1, predicted-class mass 0.75, and label price 0.002. Three query credits arrive per batch; unused credits carry forward, with at most twelve queries in one batch. Each candidate set has its own history. The batch oracle uses current $E$ labels to choose one candidate for the whole batch. We rerun the empirical query rules in paired four- and five-candidate comparisons. Fixed outputs and \EmpJoint{} reproduce their original scores.

Tables~\ref{tab:expanded-portfolio-pacs} and~\ref{tab:expanded-portfolio-vlcs} show every fixed readout, the zero-label mean, and the query-based rules. On VLCS, adding StatA raises batch-oracle accuracy from $85.637\%$ to $86.769\%$. \EmpJoint{} rises from $83.919\%$ to $84.576\%$, with mean query counts increasing from 89.4 to 96.0. The candidate adds $1.132$ percentage points of attainable batch accuracy, of which this selection rule realizes $0.657$ points. The new mean-logit rule reaches $84.732\%$, and fixed StatA reaches $85.106\%$, both with zero labels. These comparisons expose the remaining selection loss alongside the additional capability.

On PACS, the batch oracle rises from $98.646\%$ to $98.774\%$, while \EmpJoint{} stays at $98.442\%$. Mean-logit accuracy improves from $98.390\%$ to $98.472\%$. StatA helps through averaging even when selection changes little. Its additional computation is measured separately. Its predictions are recomputed for each order and batch, then stored for the selection experiment.

% BEGIN inlined appendix_materials/expanded_vlm/tables/portfolio-pacs.tex
\begin{table}[h]
\centering\small
\caption{PACS portfolio expansion on common $E$. Accuracy is equal-domain mean $\pm$ stream-order SD. The four-action portfolio contains the original, prompt-average, two-view and eight-view readouts; the fifth action is episodic StatA. Empirical-query rows are paired reruns for this expansion.}
\label{tab:expanded-portfolio-pacs}
\begin{tabular}{lrrrr}
\toprule
& \multicolumn{2}{c}{Four readouts} & \multicolumn{2}{c}{Add StatA} \\
\cmidrule(lr){2-3}\cmidrule(l){4-5}
Rule & Accuracy (\%) & Labels & Accuracy (\%) & Labels \\
\midrule
Original view & 98.191 $\pm$ 0.000 & 0.0 & 98.191 $\pm$ 0.000 & 0.0 \\
Prompt average & 98.279 $\pm$ 0.000 & 0.0 & 98.279 $\pm$ 0.000 & 0.0 \\
Two views & 98.425 $\pm$ 0.000 & 0.0 & 98.425 $\pm$ 0.000 & 0.0 \\
Eight views & 98.437 $\pm$ 0.000 & 0.0 & 98.437 $\pm$ 0.000 & 0.0 \\
StatA & --- & --- & 98.404 $\pm$ 0.063 & 0.0 \\
Mean logits & 98.390 $\pm$ 0.000 & 0.0 & 98.472 $\pm$ 0.055 & 0.0 \\
Uniform ERM & 98.429 $\pm$ 0.007 & 473.6 & 98.424 $\pm$ 0.009 & 473.6 \\
Disagreement ERM & 98.434 $\pm$ 0.031 & 473.6 & 98.419 $\pm$ 0.034 & 473.6 \\
Maximum-KL ERM & 98.439 $\pm$ 0.029 & 473.6 & 98.422 $\pm$ 0.033 & 473.6 \\
\PORTAL{} & 98.442 $\pm$ 0.007 & 1.4 & 98.442 $\pm$ 0.017 & 2.4 \\
Batch oracle & 98.646 $\pm$ 0.016 & --- & 98.774 $\pm$ 0.035 & --- \\
\bottomrule
\end{tabular}
\end{table}
\space
% END inlined appendix_materials/expanded_vlm/tables/portfolio-pacs.tex
% BEGIN inlined appendix_materials/expanded_vlm/tables/portfolio-vlcs.tex
\begin{table}[h]
\centering\small
\caption{VLCS portfolio expansion on common $E$. Accuracy is equal-domain mean $\pm$ stream-order SD. The four-action portfolio contains the original, prompt-average, two-view and eight-view readouts; the fifth action is episodic StatA. Empirical-query rows are paired reruns for this expansion.}
\label{tab:expanded-portfolio-vlcs}
\begin{tabular}{lrrrr}
\toprule
& \multicolumn{2}{c}{Four readouts} & \multicolumn{2}{c}{Add StatA} \\
\cmidrule(lr){2-3}\cmidrule(l){4-5}
Rule & Accuracy (\%) & Labels & Accuracy (\%) & Labels \\
\midrule
Original view & 80.845 $\pm$ 0.000 & 0.0 & 80.845 $\pm$ 0.000 & 0.0 \\
Prompt average & 83.523 $\pm$ 0.000 & 0.0 & 83.523 $\pm$ 0.000 & 0.0 \\
Two views & 83.861 $\pm$ 0.000 & 0.0 & 83.861 $\pm$ 0.000 & 0.0 \\
Eight views & 84.097 $\pm$ 0.000 & 0.0 & 84.097 $\pm$ 0.000 & 0.0 \\
StatA & --- & --- & 85.106 $\pm$ 0.137 & 0.0 \\
Mean logits & 83.650 $\pm$ 0.000 & 0.0 & 84.732 $\pm$ 0.140 & 0.0 \\
Uniform ERM & 83.907 $\pm$ 0.202 & 354.8 & 83.962 $\pm$ 0.224 & 354.8 \\
Disagreement ERM & 83.869 $\pm$ 0.126 & 354.8 & 84.109 $\pm$ 0.095 & 354.8 \\
Maximum-KL ERM & 83.918 $\pm$ 0.158 & 354.8 & 84.258 $\pm$ 0.160 & 354.8 \\
\PORTAL{} & 83.919 $\pm$ 0.170 & 89.4 & 84.576 $\pm$ 0.174 & 96.0 \\
Batch oracle & 85.637 $\pm$ 0.148 & --- & 86.769 $\pm$ 0.073 & --- \\
\bottomrule
\end{tabular}
\end{table}
\space
% END inlined appendix_materials/expanded_vlm/tables/portfolio-vlcs.tex

\section{Supervised convex combinations of readout logits}
\label{app:convex-readout}

\subsection{A learnable readout and target-aware acquisition}

We also fit a weighted average of the original four readouts on PACS, OfficeHome, and VLCS. Let $K$ be the candidate count and $C$ the class count. Matrix $Z_x\in\mathbb R^{K\times C}$ contains their logits for input $x$, with class column $z_{x,c}$. The weights $w\in\Delta_K$ are nonnegative and sum to one, giving combined logits $Z_x^\top w$. Let $Q_{\mathrm{obs}}$ be the acquired inputs and $y_q$ their labels. We regularize toward uniform weights $u$ with strength $\lambda>0$ by minimizing
\begin{equation}
 \sum_{q\in Q_{\mathrm{obs}}}
 \left[\log\sum_c\exp(w^\top z_{q,c})-w^\top z_{q,y_q}\right]
 +\frac{\lambda}{2}\|w-u\|_2^2,
 \qquad u=\frac{1}{K}\mathbf 1.
 \label{eq:expanded-convex-objective}
\end{equation}
The regularizer favors uniform averaging when evidence is limited. We subtract each candidate's mean class logit, which leaves the mixture probabilities unchanged. Sequential least-squares programming (SLSQP) uses analytic gradients, tolerance $10^{-9}$, and at most 200 iterations. Weights reset to $u$ each batch and are refitted after each query, up to three labels. The primary regularization strength is $\lambda=1$; sensitivity comparisons use $0.1$ and $10$.

Let $\ell_x(w,y)$ be the cross-entropy loss of the mixture for label $y$, and let $p_x=\operatorname{softmax}(Z_x^\top w)$. Its gradient and Hessian are
\begin{align}
 \nabla_w\ell_x(w,y)&=Z_x(p_x-e_y),\\
 H_x=\nabla_w^2\ell_x(w,y)
 &=Z_x\left[\operatorname{diag}(p_x)-p_xp_x^\top\right]Z_x^\top.
 \label{eq:expanded-logit-hessian}
\end{align}
The quadratic form of $H_x$ is a predictive variance, so $H_x$ is positive semidefinite. The positive regularizer makes the objective strictly convex on the simplex. Under the current categorical model $Y\sim p_x$, the gradient covariance also equals $H_x$, connecting curvature to paired candidate-gradient uncertainty.

We rank queries by a local interior curvature approximation in the sum-zero simplex subspace. Let $B$ have orthonormal columns spanning that subspace, $\bar H_x=B^\top H_xB$, and
\begin{equation}
 A=\lambda I+\sum_{q\in Q_{\mathrm{obs}}}\bar H_q,
 \qquad
 \bar H_E=\frac{1}{|E|}\sum_{e\in E}\bar H_e.
\end{equation}
We call the following acquisition score \emph{target-Fisher}, because it uses the curvature of the target prediction loss:
\begin{equation}
 s(x)=\frac12\operatorname{tr}\!\left[
 \bar H_E\left(A^{-1}-(A+\bar H_x)^{-1}\right)\right].
 \label{eq:expanded-target-fisher}
\end{equation}
This score measures prospective covariance reduction weighted by target curvature. Its $E$ dependence uses unlabeled prediction vectors. After each query, the fitted weights and curvature are updated before ranking the remaining $Q$ examples.

\subsection{Readout fitting and acquisition outcomes}

Table~\ref{tab:expanded-convex-primary} compares uniform, disagreement-first, maximum symmetric-KL, and target-Fisher queries at $\lambda=1$. Each fitted mixture is paired with ERM on exactly the same queries, using U8V on ties. Fixed U8V and mean logits provide zero-label references. \EmpJoint{} uses its accumulated evidence and actual query budget. All OfficeHome instances remain included. The convex experiment skips batches without evaluation inputs, using 737.4 labels on average; the discrete-selector experiment uses 738.0.

Fitting improves on uniform averaging for OfficeHome and VLCS. At the primary regularizer, target-Fisher reaches $88.735\%$ and $84.030\%$, respectively, compared with mean logits at $88.237\%$ and $83.650\%$. Maximum symmetric-KL reaches $88.757\%$ and $84.115\%$, so the curvature score does not improve this acquisition comparison. Fixed U8V reaches $89.008\%$ and $84.097\%$. On PACS, the fitted target-Fisher mixture gives $98.387\%$, close to mean logits at $98.390\%$ and below U8V at $98.437\%$. Cross-entropy and Brier score measure probability quality; the Brier score is the squared distance to the one-hot label. These metrics and all 36 regularizer/acquisition settings are retained in Tables~\ref{tab:expanded-convex-pacs}, \ref{tab:expanded-convex-officehome}, and~\ref{tab:expanded-convex-vlcs}. The readout fit supplies an additional learnable candidate, while its advantages over the strongest fixed readout remain setting-dependent.

% BEGIN inlined appendix_materials/expanded_vlm/tables/convex-primary.tex
\begin{table}[H]
\centering\scriptsize
\caption{Primary convex-readout results ($\lambda=1$) on common $E$. Accuracy is equal-domain mean $\pm$ stream-order SD; labels average five orders. Discrete ERM rows use the same queries as the corresponding mixture.}
\label{tab:expanded-convex-primary}
\begin{tabular}{lrrrrrr}
\toprule
& \multicolumn{2}{c}{PACS} & \multicolumn{2}{c}{OfficeHome} & \multicolumn{2}{c}{VLCS} \\
\cmidrule(lr){2-3}\cmidrule(lr){4-5}\cmidrule(l){6-7}
Rule & Acc. & Labels & Acc. & Labels & Acc. & Labels \\
\midrule
Eight views & 98.437 $\pm$ 0.000 & 0.0 & 89.008 $\pm$ 0.000 & 0.0 & 84.097 $\pm$ 0.000 & 0.0 \\
Mean logits & 98.390 $\pm$ 0.000 & 0.0 & 88.237 $\pm$ 0.000 & 0.0 & 83.650 $\pm$ 0.000 & 0.0 \\
\PORTAL{} & 98.442 $\pm$ 0.007 & 1.4 & 89.010 $\pm$ 0.012 & 7.4 & 83.919 $\pm$ 0.170 & 89.4 \\
Uniform / mixture & 98.398 $\pm$ 0.014 & 473.6 & 88.415 $\pm$ 0.098 & 737.4 & 83.868 $\pm$ 0.228 & 354.8 \\
Disagreement / mixture & 98.415 $\pm$ 0.009 & 473.6 & 88.636 $\pm$ 0.068 & 737.4 & 84.005 $\pm$ 0.139 & 354.8 \\
Maximum KL / mixture & 98.393 $\pm$ 0.014 & 473.6 & 88.757 $\pm$ 0.087 & 737.4 & 84.115 $\pm$ 0.083 & 354.8 \\
Target Fisher / mixture & 98.387 $\pm$ 0.014 & 473.6 & 88.735 $\pm$ 0.109 & 737.4 & 84.030 $\pm$ 0.137 & 354.8 \\
Uniform / discrete ERM & 98.429 $\pm$ 0.007 & 473.6 & 88.863 $\pm$ 0.121 & 737.4 & 83.907 $\pm$ 0.202 & 354.8 \\
Disagreement / discrete ERM & 98.434 $\pm$ 0.031 & 473.6 & 88.637 $\pm$ 0.118 & 737.4 & 83.869 $\pm$ 0.126 & 354.8 \\
Maximum KL / discrete ERM & 98.439 $\pm$ 0.029 & 473.6 & 88.688 $\pm$ 0.122 & 737.4 & 83.918 $\pm$ 0.158 & 354.8 \\
Target Fisher / discrete ERM & 98.439 $\pm$ 0.029 & 473.6 & 88.685 $\pm$ 0.127 & 737.4 & 83.850 $\pm$ 0.052 & 354.8 \\
\bottomrule
\end{tabular}
\end{table}
\space
% END inlined appendix_materials/expanded_vlm/tables/convex-primary.tex
\FloatBarrier
% BEGIN inlined appendix_materials/expanded_vlm/tables/convex-pacs.tex
\begingroup\scriptsize
\setlength{\tabcolsep}{2.4pt}
\begin{longtable}{rllrrrrrrr}
\caption{Complete convex-readout grid on PACS common $E$. Every acquisition uses up to three current-batch labels; means and SDs cover five stream orders. Accuracy, cross-entropy and Brier columns pair discrete ERM (D) with the fitted mixture (M). Runtime is CPU mixture-fitting/replay time per stream.}\label{tab:expanded-convex-pacs}\\
\toprule
& & & \multicolumn{2}{c}{Accuracy (\%)} & \multicolumn{2}{c}{CE} & \multicolumn{2}{c}{Brier} & \\
$\lambda$ & Acquisition & Labels & D & M & D & M & D & M & Fit (s) \\
\midrule
\endfirsthead
\toprule
& & & \multicolumn{2}{c}{Accuracy (\%)} & \multicolumn{2}{c}{CE} & \multicolumn{2}{c}{Brier} & \\
$\lambda$ & Acquisition & Labels & D & M & D & M & D & M & Fit (s) \\
\midrule
\endhead
0.1 & Uniform & 473.6 & 98.429 $\pm$ 0.007 & 98.382 $\pm$ 0.021 & 0.0580 & 0.0604 & 0.0252 & 0.0264 & 0.821 \\
0.1 & Disagreement & 473.6 & 98.434 $\pm$ 0.031 & 98.393 $\pm$ 0.033 & 0.0584 & 0.0603 & 0.0253 & 0.0262 & 0.825 \\
0.1 & Maximum KL & 473.6 & 98.439 $\pm$ 0.029 & 98.407 $\pm$ 0.049 & 0.0584 & 0.0597 & 0.0253 & 0.0259 & 1.028 \\
0.1 & Target Fisher & 473.6 & 98.434 $\pm$ 0.034 & 98.417 $\pm$ 0.069 & 0.0583 & 0.0595 & 0.0253 & 0.0258 & 1.167 \\
\addlinespace[2pt]
1 & Uniform & 473.6 & 98.429 $\pm$ 0.007 & 98.398 $\pm$ 0.014 & 0.0580 & 0.0604 & 0.0252 & 0.0264 & 0.503 \\
1 & Disagreement & 473.6 & 98.434 $\pm$ 0.031 & 98.415 $\pm$ 0.009 & 0.0584 & 0.0603 & 0.0253 & 0.0263 & 0.560 \\
1 & Maximum KL & 473.6 & 98.439 $\pm$ 0.029 & 98.393 $\pm$ 0.014 & 0.0584 & 0.0597 & 0.0253 & 0.0260 & 0.706 \\
1 & Target Fisher & 473.6 & 98.439 $\pm$ 0.029 & 98.387 $\pm$ 0.014 & 0.0584 & 0.0597 & 0.0253 & 0.0261 & 0.825 \\
\addlinespace[2pt]
10 & Uniform & 473.6 & 98.429 $\pm$ 0.007 & 98.388 $\pm$ 0.011 & 0.0580 & 0.0604 & 0.0252 & 0.0265 & 0.597 \\
10 & Disagreement & 473.6 & 98.434 $\pm$ 0.031 & 98.393 $\pm$ 0.006 & 0.0584 & 0.0604 & 0.0253 & 0.0264 & 0.705 \\
10 & Maximum KL & 473.6 & 98.439 $\pm$ 0.029 & 98.388 $\pm$ 0.011 & 0.0584 & 0.0603 & 0.0253 & 0.0264 & 0.987 \\
10 & Target Fisher & 473.6 & 98.439 $\pm$ 0.029 & 98.388 $\pm$ 0.011 & 0.0584 & 0.0603 & 0.0253 & 0.0264 & 1.101 \\
\addlinespace[2pt]
\bottomrule
\end{longtable}
\endgroup
\space
% END inlined appendix_materials/expanded_vlm/tables/convex-pacs.tex
% BEGIN inlined appendix_materials/expanded_vlm/tables/convex-officehome.tex
\begingroup\scriptsize
\setlength{\tabcolsep}{2.4pt}
\begin{longtable}{rllrrrrrrr}
\caption{Complete convex-readout grid on OFFICEHOME common $E$. Every acquisition uses up to three current-batch labels; means and SDs cover five stream orders. Accuracy, cross-entropy and Brier columns pair discrete ERM (D) with the fitted mixture (M). Runtime is CPU mixture-fitting/replay time per stream.}\label{tab:expanded-convex-officehome}\\
\toprule
& & & \multicolumn{2}{c}{Accuracy (\%)} & \multicolumn{2}{c}{CE} & \multicolumn{2}{c}{Brier} & \\
$\lambda$ & Acquisition & Labels & D & M & D & M & D & M & Fit (s) \\
\midrule
\endfirsthead
\toprule
& & & \multicolumn{2}{c}{Accuracy (\%)} & \multicolumn{2}{c}{CE} & \multicolumn{2}{c}{Brier} & \\
$\lambda$ & Acquisition & Labels & D & M & D & M & D & M & Fit (s) \\
\midrule
\endhead
0.1 & Uniform & 737.4 & 88.863 $\pm$ 0.121 & 88.523 $\pm$ 0.144 & 0.3952 & 0.4105 & 0.1580 & 0.1623 & 1.464 \\
0.1 & Disagreement & 737.4 & 88.637 $\pm$ 0.118 & 88.579 $\pm$ 0.094 & 0.4080 & 0.4060 & 0.1617 & 0.1609 & 1.283 \\
0.1 & Maximum KL & 737.4 & 88.688 $\pm$ 0.122 & 88.726 $\pm$ 0.112 & 0.4018 & 0.4000 & 0.1601 & 0.1592 & 1.401 \\
0.1 & Target Fisher & 737.4 & 88.687 $\pm$ 0.088 & 88.724 $\pm$ 0.077 & 0.4029 & 0.4006 & 0.1604 & 0.1595 & 1.570 \\
\addlinespace[2pt]
1 & Uniform & 737.4 & 88.863 $\pm$ 0.121 & 88.415 $\pm$ 0.098 & 0.3952 & 0.4084 & 0.1580 & 0.1619 & 0.938 \\
1 & Disagreement & 737.4 & 88.637 $\pm$ 0.118 & 88.636 $\pm$ 0.068 & 0.4080 & 0.4030 & 0.1617 & 0.1599 & 1.151 \\
1 & Maximum KL & 737.4 & 88.688 $\pm$ 0.122 & 88.757 $\pm$ 0.087 & 0.4018 & 0.3982 & 0.1601 & 0.1586 & 1.377 \\
1 & Target Fisher & 737.4 & 88.685 $\pm$ 0.127 & 88.735 $\pm$ 0.109 & 0.4020 & 0.3984 & 0.1603 & 0.1587 & 1.544 \\
\addlinespace[2pt]
10 & Uniform & 737.4 & 88.863 $\pm$ 0.121 & 88.298 $\pm$ 0.040 & 0.3952 & 0.4107 & 0.1580 & 0.1627 & 1.272 \\
10 & Disagreement & 737.4 & 88.637 $\pm$ 0.118 & 88.433 $\pm$ 0.076 & 0.4080 & 0.4071 & 0.1617 & 0.1614 & 1.684 \\
10 & Maximum KL & 737.4 & 88.688 $\pm$ 0.122 & 88.555 $\pm$ 0.031 & 0.4018 & 0.4039 & 0.1601 & 0.1605 & 1.971 \\
10 & Target Fisher & 737.4 & 88.637 $\pm$ 0.112 & 88.561 $\pm$ 0.027 & 0.4027 & 0.4036 & 0.1605 & 0.1604 & 2.135 \\
\addlinespace[2pt]
\bottomrule
\end{longtable}
\endgroup
\space
% END inlined appendix_materials/expanded_vlm/tables/convex-officehome.tex
% BEGIN inlined appendix_materials/expanded_vlm/tables/convex-vlcs.tex
\begingroup\scriptsize
\setlength{\tabcolsep}{2.4pt}
\begin{longtable}{rllrrrrrrr}
\caption{Complete convex-readout grid on VLCS common $E$. Every acquisition uses up to three current-batch labels; means and SDs cover five stream orders. Accuracy, cross-entropy and Brier columns pair discrete ERM (D) with the fitted mixture (M). Runtime is CPU mixture-fitting/replay time per stream.}\label{tab:expanded-convex-vlcs}\\
\toprule
& & & \multicolumn{2}{c}{Accuracy (\%)} & \multicolumn{2}{c}{CE} & \multicolumn{2}{c}{Brier} & \\
$\lambda$ & Acquisition & Labels & D & M & D & M & D & M & Fit (s) \\
\midrule
\endfirsthead
\toprule
& & & \multicolumn{2}{c}{Accuracy (\%)} & \multicolumn{2}{c}{CE} & \multicolumn{2}{c}{Brier} & \\
$\lambda$ & Acquisition & Labels & D & M & D & M & D & M & Fit (s) \\
\midrule
\endhead
0.1 & Uniform & 354.8 & 83.907 $\pm$ 0.202 & 83.819 $\pm$ 0.248 & 0.4629 & 0.4648 & 0.2402 & 0.2392 & 0.681 \\
0.1 & Disagreement & 354.8 & 83.869 $\pm$ 0.126 & 83.956 $\pm$ 0.164 & 0.4651 & 0.4625 & 0.2404 & 0.2390 & 0.630 \\
0.1 & Maximum KL & 354.8 & 83.918 $\pm$ 0.158 & 84.068 $\pm$ 0.151 & 0.4661 & 0.4603 & 0.2408 & 0.2371 & 0.666 \\
0.1 & Target Fisher & 354.8 & 83.887 $\pm$ 0.181 & 83.901 $\pm$ 0.241 & 0.4659 & 0.4613 & 0.2409 & 0.2381 & 0.775 \\
\addlinespace[2pt]
1 & Uniform & 354.8 & 83.907 $\pm$ 0.202 & 83.868 $\pm$ 0.228 & 0.4629 & 0.4615 & 0.2402 & 0.2382 & 0.462 \\
1 & Disagreement & 354.8 & 83.869 $\pm$ 0.126 & 84.005 $\pm$ 0.139 & 0.4651 & 0.4610 & 0.2404 & 0.2386 & 0.517 \\
1 & Maximum KL & 354.8 & 83.918 $\pm$ 0.158 & 84.115 $\pm$ 0.083 & 0.4661 & 0.4584 & 0.2408 & 0.2366 & 0.562 \\
1 & Target Fisher & 354.8 & 83.850 $\pm$ 0.052 & 84.030 $\pm$ 0.137 & 0.4661 & 0.4585 & 0.2407 & 0.2370 & 0.651 \\
\addlinespace[2pt]
10 & Uniform & 354.8 & 83.907 $\pm$ 0.202 & 83.757 $\pm$ 0.068 & 0.4629 & 0.4641 & 0.2402 & 0.2408 & 0.653 \\
10 & Disagreement & 354.8 & 83.869 $\pm$ 0.126 & 83.850 $\pm$ 0.144 & 0.4651 & 0.4630 & 0.2404 & 0.2404 & 0.755 \\
10 & Maximum KL & 354.8 & 83.918 $\pm$ 0.158 & 83.842 $\pm$ 0.125 & 0.4661 & 0.4609 & 0.2408 & 0.2390 & 0.828 \\
10 & Target Fisher & 354.8 & 83.876 $\pm$ 0.080 & 83.851 $\pm$ 0.095 & 0.4658 & 0.4607 & 0.2406 & 0.2390 & 0.926 \\
\addlinespace[2pt]
\bottomrule
\end{longtable}
\endgroup
\space
% END inlined appendix_materials/expanded_vlm/tables/convex-vlcs.tex
\space
% END inlined sections/expanded_vlm_appendix.tex
\FloatBarrier
% BEGIN inlined sections/expanded_dynamic_appendix.tex
\section{Dynamic conditions, source changes and class-label noise}
\label{app:expanded_dynamic}
Batch composition changes normalization-based predictions. We regenerate all four ResNet outputs for each condition and keep every target image once per stream. Three arrival orders are tested with four source checkpoints. For the Photo source, batch sizes are 16,32,64,128. At batch size 64, we test mixed, reversed, gradually overlapping, and class-block orders. Art, Cartoon, and Sketch source models are each evaluated on the remaining three domains. This gives 33 complete runs. The original query/evaluation roles remain fixed. Photo images, used as targets for the other source models, receive a content-based split before evaluation.
% Stream seeds: 301--303.

\paragraph{Arrival construction.}
Domain-block streams shuffle images within each domain. Reverse streams reverse the domain order while preserving their within-domain permutations. Mixed streams uniformly permute all target images. For gradual transitions, the within-domain normalized rank receives Gaussian jitter of standard deviation $0.45$ before sorting across domains; overlapping regions retain every image once. Class-block streams group the benchmark's labels within each domain, constructing temporally imbalanced arrivals. Labels define this offline evaluation arrangement and remain hidden from querying and deployment. Batches ending a block retain their original smaller size.

Every model uses the same released architecture and preprocessing as the primary PACS experiment. Source and batch-statistic branches are frozen before querying; batch-statistic normalization uses the actual current inputs and preserves running buffers. Selection history carries across the stream, and a new selector begins each condition/order. Three-label empirical rules and joint selection see the same candidate predictions and evaluation images. Small query pools reduce the attainable fixed-three budget, so tables report actual label counts.

\begin{figure}[t]
\centering
\includegraphics[width=\linewidth]{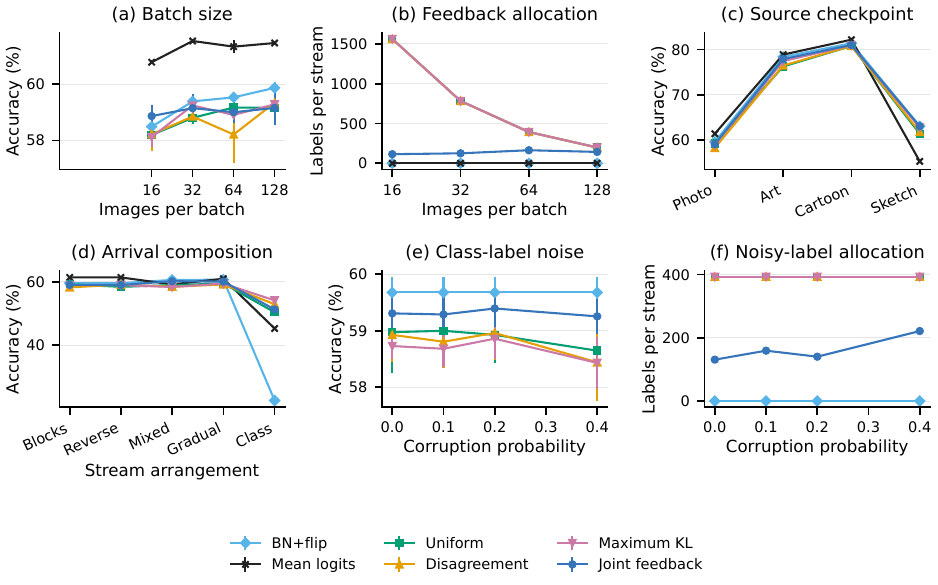}
\caption{Actual batch/source changes and label-channel stress tests. Panels (a)--(d) use regenerated predictions under three stream orders per condition; (e)--(f) use the primary five frozen-prediction orders with multiclass annotation replacement. Accuracy error bars are order standard deviations; label curves show mean actual acquisitions. Fixed output and mean-logit references retain zero feedback cost. Class-block arrivals change the normalization predictions, while label corruption leaves every candidate prediction fixed.}
\label{fig:expanded_dynamic_grid}
\end{figure}

\paragraph{Batch and domain effects.}
For Photo source, smaller batches reduce the fixed normalization-and-flip accuracy: the batch-16 mean is 58.49\%, compared with 59.87\% at batch 128. The corresponding batch oracle changes its action granularity, so its accuracy is interpreted within each condition. Uniform mean logits remain a strong reference. Joint selection at batch 16 reaches 58.87\% using 114.3 labels per stream, compared with 58.12\% for maximum-KL selection using 1,561.7 labels. Across mixed-domain arrivals, joint selection reaches 60.16\% with 114.7 labels and maximum-KL selection reaches 58.34\% with 390.3. Fixed BN+flip reaches 60.56\% in that same condition.

\paragraph{Temporally concentrated classes.}
Class-block arrivals materially alter batch statistics. The fixed BN+flip branch reaches 22.58\%, while source accuracy is about 50.88\%. Joint selection reaches 51.28\% with 187.7 labels; maximum-KL selection reaches 54.11\% with 393. Their batch-net harm values are 0.34 and 0.86 percentage points, respectively, versus 32.41 for fixed BN+flip. The result exposes a different evidence--accuracy tradeoff from mixed arrivals. It also demonstrates that a candidate strong under ordinary domain-block evaluation can require frequent local rejection under a changed temporal class distribution.

\paragraph{Source-model variation.}
The Art, Cartoon and Sketch checkpoints alter both the learned source representation and target domains. Joint accuracy is 77.91\%,81.02\%,63.00\%, with 69.0,64.0,54.3 labels per target stream. The corresponding maximum-KL accuracies are 77.45\%,80.97\%,62.71\%. Fixed BN+flip reaches 78.39\%,81.41\%,63.02\%. Each released source checkpoint is repeated over three target stream orders.

\begin{table}[!tbp]
\centering
\caption{PACS under changes in source checkpoint, batch size and arrival order. Every entry averages three regenerated stream orders.}
\label{tab:dynamic-matrix-performance}
\textit{Accuracy (\%).}\par\vspace{2pt}
% BEGIN inlined layout_kit/PORTAL_layout_kit/tables/dynamic_0.tex
\begingroup\fontsize{7.8}{8.6}\selectfont\setlength{\tabcolsep}{1.5pt}\renewcommand{\arraystretch}{1.0}
\setlength{\PortalTableArea}{\dimexpr(\linewidth-22\tabcolsep-20mm)/11\relax}
\begin{tabular}{@{}L{20mm}*{11}{R{\PortalTableArea}}@{}}
\toprule
&\multicolumn{3}{c}{Source checkpoint}&\multicolumn{4}{c}{Photo source: batch size}&\multicolumn{4}{c}{Photo source: arrival order}\\
\cmidrule(lr){2-4}\cmidrule(lr){5-8}\cmidrule(l){9-12}
Rule & Art & Cart. & Sketch & 16 & 32 & 64 & 128 & Rev. & Mixed & Grad. & Class\\
\midrule
Source & 68.42 & 76.91 & 38.32 & 50.85 & 50.85 & 50.88 & 50.88 & 50.88 & 50.88 & 50.88 & 50.88\\
BN & 77.33 & 80.35 & 62.45 & 57.18 & 58.07 & 58.17 & 58.80 & 58.17 & 59.67 & 59.08 & 22.65\\
Flip & 68.94 & 77.51 & 38.02 & 50.90 & 50.90 & 50.90 & 50.90 & 50.90 & 50.90 & 50.90 & 50.90\\
BN+flip & 78.39 & 81.41 & 63.02 & 58.49 & 59.39 & 59.53 & 59.87 & 59.53 & 60.56 & 60.63 & 22.58\\
Mean logits & 78.89 & 82.20 & 55.22 & 60.79 & 61.54 & 61.34 & 61.47 & 61.34 & 59.15 & 60.92 & 45.25\\
\addlinespace[1.5pt]
Uniform & 76.25 & 80.81 & 61.34 & 58.19 & 58.80 & 59.17 & 59.16 & 58.38 & 58.92 & 59.43 & 50.39\\
Disagreement & 76.47 & 80.82 & 61.72 & 58.22 & 58.85 & 58.21 & 59.34 & 59.09 & 58.41 & 59.13 & 52.92\\
Max. KL & 77.45 & 80.97 & 62.71 & 58.12 & 59.25 & 58.91 & 59.29 & 58.91 & 58.34 & 59.40 & 54.11\\
\rowcolor{PortalRow}
\PORTAL{} & 77.91 & 81.02 & 63.00 & 58.87 & 59.15 & 59.01 & 59.16 & 58.96 & 60.16 & 60.02 & 51.28\\
\midrule
Oracle & 79.94 & 83.66 & 64.00 & 66.05 & 64.28 & 62.81 & 61.93 & 62.81 & 62.58 & 62.87 & 56.52\\
\bottomrule\end{tabular}\endgroup
\space
% END inlined layout_kit/PORTAL_layout_kit/tables/dynamic_0.tex
\par\vspace{6pt}
\textit{Order standard deviation (percentage points).}\par\vspace{2pt}
% BEGIN inlined layout_kit/PORTAL_layout_kit/tables/dynamic_1.tex
\begingroup\fontsize{7.8}{8.6}\selectfont\setlength{\tabcolsep}{1.5pt}\renewcommand{\arraystretch}{1.0}
\setlength{\PortalTableArea}{\dimexpr(\linewidth-22\tabcolsep-20mm)/11\relax}
\begin{tabular}{@{}L{20mm}*{11}{R{\PortalTableArea}}@{}}
\toprule
&\multicolumn{3}{c}{Source checkpoint}&\multicolumn{4}{c}{Photo source: batch size}&\multicolumn{4}{c}{Photo source: arrival order}\\
\cmidrule(lr){2-4}\cmidrule(lr){5-8}\cmidrule(l){9-12}
Rule & Art & Cart. & Sketch & 16 & 32 & 64 & 128 & Rev. & Mixed & Grad. & Class\\
\midrule
Source & 0.00 & 0.00 & 0.00 & 0.00 & 0.00 & 0.00 & 0.00 & 0.00 & 0.00 & 0.00 & 0.00\\
BN & 0.08 & 0.17 & 0.15 & 0.10 & 0.28 & 0.09 & 0.18 & 0.09 & 0.16 & 0.16 & 0.35\\
Flip & 0.00 & 0.00 & 0.00 & 0.00 & 0.00 & 0.00 & 0.00 & 0.00 & 0.00 & 0.00 & 0.00\\
BN+flip & 0.25 & 0.04 & 0.13 & 0.21 & 0.09 & 0.18 & 0.22 & 0.18 & 0.41 & 0.10 & 0.17\\
Mean logits & 0.40 & 0.16 & 0.43 & 0.07 & 0.07 & 0.23 & 0.09 & 0.23 & 0.22 & 0.18 & 0.05\\
\addlinespace[1.5pt]
Uniform & 0.46 & 0.34 & 0.23 & 0.28 & 0.23 & 0.29 & 0.15 & 0.47 & 0.26 & 0.24 & 0.86\\
Disagreement & 0.70 & 0.50 & 0.81 & 0.61 & 0.18 & 1.02 & 0.20 & 0.29 & 0.41 & 0.54 & 0.13\\
Max. KL & 0.12 & 0.12 & 0.21 & 0.35 & 0.38 & 0.26 & 0.09 & 0.26 & 0.61 & 0.58 & 0.31\\
\rowcolor{PortalRow}
\PORTAL{} & 0.07 & 0.54 & 0.16 & 0.40 & 0.50 & 0.38 & 0.60 & 0.30 & 0.53 & 0.72 & 0.51\\
\midrule
Oracle & 0.18 & 0.20 & 0.06 & 0.28 & 0.25 & 0.33 & 0.08 & 0.33 & 0.15 & 0.20 & 0.07\\
\bottomrule\end{tabular}\endgroup
\space
% END inlined layout_kit/PORTAL_layout_kit/tables/dynamic_1.tex
\par\vspace{6pt}
\textit{Acquired labels per stream.}\par\vspace{2pt}
% BEGIN inlined layout_kit/PORTAL_layout_kit/tables/dynamic_2.tex
\begingroup\fontsize{7.8}{8.6}\selectfont\setlength{\tabcolsep}{1.5pt}\renewcommand{\arraystretch}{1.0}
\setlength{\PortalTableArea}{\dimexpr(\linewidth-22\tabcolsep-20mm)/11\relax}
\begin{tabular}{@{}L{20mm}*{11}{R{\PortalTableArea}}@{}}
\toprule
&\multicolumn{3}{c}{Source checkpoint}&\multicolumn{4}{c}{Photo source: batch size}&\multicolumn{4}{c}{Photo source: arrival order}\\
\cmidrule(lr){2-4}\cmidrule(lr){5-8}\cmidrule(l){9-12}
Rule & Art & Cart. & Sketch & 16 & 32 & 64 & 128 & Rev. & Mixed & Grad. & Class\\
\midrule
Source & 0 & 0 & 0 & 0 & 0 & 0 & 0 & 0 & 0 & 0 & 0\\
BN & 0 & 0 & 0 & 0 & 0 & 0 & 0 & 0 & 0 & 0 & 0\\
Flip & 0 & 0 & 0 & 0 & 0 & 0 & 0 & 0 & 0 & 0 & 0\\
BN+flip & 0 & 0 & 0 & 0 & 0 & 0 & 0 & 0 & 0 & 0 & 0\\
Mean logits & 0 & 0 & 0 & 0 & 0 & 0 & 0 & 0 & 0 & 0 & 0\\
\addlinespace[1.5pt]
Uniform & 377.3 & 362.7 & 287.7 & 1,561.7 & 783 & 393 & 198 & 393 & 390.3 & 390 & 393\\
Disagreement & 377.3 & 362.7 & 287.7 & 1,561.7 & 783 & 393 & 198 & 393 & 390.3 & 390 & 393\\
Max. KL & 377.3 & 362.7 & 287.7 & 1,561.7 & 783 & 393 & 198 & 393 & 390.3 & 390 & 393\\
\rowcolor{PortalRow}
\PORTAL{} & 69 & 64 & 54.3 & 114.3 & 125.3 & 164 & 143 & 189.3 & 114.7 & 110 & 187.7\\
\midrule
Oracle & \dash & \dash & \dash & \dash & \dash & \dash & \dash & \dash & \dash & \dash & \dash\\
\bottomrule\end{tabular}\endgroup
\space
% END inlined layout_kit/PORTAL_layout_kit/tables/dynamic_2.tex
\end{table}

\begin{table}[!tbp]
\centering
\caption{Net batch harm and sample-level new errors under the same PACS conditions as Table~\ref{tab:dynamic-matrix-performance}.}
\label{tab:dynamic-matrix-errors}
\textit{Net batch harm (percentage points).}\par\vspace{2pt}
% BEGIN inlined layout_kit/PORTAL_layout_kit/tables/dynamic_3.tex
\begingroup\fontsize{7.8}{8.6}\selectfont\setlength{\tabcolsep}{1.5pt}\renewcommand{\arraystretch}{1.0}
\setlength{\PortalTableArea}{\dimexpr(\linewidth-22\tabcolsep-20mm)/11\relax}
\begin{tabular}{@{}L{20mm}*{11}{R{\PortalTableArea}}@{}}
\toprule
&\multicolumn{3}{c}{Source checkpoint}&\multicolumn{4}{c}{Photo source: batch size}&\multicolumn{4}{c}{Photo source: arrival order}\\
\cmidrule(lr){2-4}\cmidrule(lr){5-8}\cmidrule(l){9-12}
Rule & Art & Cart. & Sketch & 16 & 32 & 64 & 128 & Rev. & Mixed & Grad. & Class\\
\midrule
Source & 0.00 & 0.00 & 0.00 & 0.00 & 0.00 & 0.00 & 0.00 & 0.00 & 0.00 & 0.00 & 0.00\\
BN & 0.59 & 1.73 & 0.01 & 5.97 & 3.83 & 2.58 & 1.48 & 2.58 & 0.87 & 1.53 & 32.41\\
Flip & 1.00 & 1.13 & 1.46 & 1.75 & 1.52 & 1.21 & 0.91 & 1.21 & 1.28 & 1.29 & 1.12\\
BN+flip & 0.55 & 1.21 & 0.00 & 5.21 & 3.28 & 1.93 & 1.12 & 1.93 & 0.70 & 1.08 & 32.41\\
Mean logits & 0.12 & 0.37 & 0.01 & 1.14 & 0.58 & 0.19 & 0.04 & 0.19 & 0.14 & 0.13 & 8.10\\
\addlinespace[1.5pt]
Uniform & 0.51 & 1.11 & 0.01 & 3.73 & 2.35 & 1.38 & 0.71 & 1.40 & 0.63 & 0.88 & 4.62\\
Disagreement & 0.42 & 0.99 & 0.00 & 3.62 & 2.35 & 1.45 & 0.64 & 1.29 & 0.61 & 0.84 & 1.94\\
Max. KL & 0.46 & 0.84 & 0.00 & 3.67 & 2.10 & 1.37 & 0.86 & 1.37 & 0.52 & 0.91 & 0.86\\
\rowcolor{PortalRow}
\PORTAL{} & 0.52 & 0.80 & 0.00 & 3.56 & 2.89 & 1.72 & 1.16 & 1.57 & 0.60 & 1.15 & 0.34\\
\midrule
Oracle & 0.00 & 0.00 & 0.00 & 0.00 & 0.00 & 0.00 & 0.00 & 0.00 & 0.00 & 0.00 & 0.00\\
\bottomrule\end{tabular}\endgroup
\space
% END inlined layout_kit/PORTAL_layout_kit/tables/dynamic_3.tex
\par\vspace{6pt}
\textit{New errors (\%).}\par\vspace{2pt}
% BEGIN inlined layout_kit/PORTAL_layout_kit/tables/dynamic_4.tex
\begingroup\fontsize{7.8}{8.6}\selectfont\setlength{\tabcolsep}{1.5pt}\renewcommand{\arraystretch}{1.0}
\setlength{\PortalTableArea}{\dimexpr(\linewidth-22\tabcolsep-20mm)/11\relax}
\begin{tabular}{@{}L{20mm}*{11}{R{\PortalTableArea}}@{}}
\toprule
&\multicolumn{3}{c}{Source checkpoint}&\multicolumn{4}{c}{Photo source: batch size}&\multicolumn{4}{c}{Photo source: arrival order}\\
\cmidrule(lr){2-4}\cmidrule(lr){5-8}\cmidrule(l){9-12}
Rule & Art & Cart. & Sketch & 16 & 32 & 64 & 128 & Rev. & Mixed & Grad. & Class\\
\midrule
Source & 0.00 & 0.00 & 0.00 & 0.00 & 0.00 & 0.00 & 0.00 & 0.00 & 0.00 & 0.00 & 0.00\\
BN & 7.13 & 7.16 & 3.24 & 12.96 & 12.27 & 12.07 & 11.62 & 12.07 & 9.48 & 10.42 & 34.79\\
Flip & 1.78 & 2.04 & 2.25 & 2.07 & 2.07 & 2.07 & 2.07 & 2.07 & 2.07 & 2.07 & 2.07\\
BN+flip & 6.74 & 6.36 & 3.25 & 12.24 & 11.58 & 11.30 & 11.05 & 11.30 & 8.85 & 9.68 & 34.88\\
Mean logits & 1.40 & 1.90 & 1.31 & 2.51 & 2.34 & 2.49 & 2.47 & 2.49 & 2.57 & 2.49 & 10.24\\
\addlinespace[1.5pt]
Uniform & 5.53 & 5.53 & 3.13 & 8.95 & 8.31 & 8.42 & 7.61 & 8.22 & 7.47 & 8.05 & 6.38\\
Disagreement & 5.39 & 5.09 & 3.10 & 8.49 & 8.50 & 7.90 & 7.93 & 7.91 & 6.99 & 7.82 & 3.35\\
Max. KL & 6.28 & 4.57 & 3.20 & 8.58 & 8.18 & 8.20 & 8.33 & 8.20 & 6.58 & 8.15 & 1.96\\
\rowcolor{PortalRow}
\PORTAL{} & 6.58 & 5.17 & 3.22 & 9.10 & 10.35 & 10.01 & 10.49 & 9.63 & 8.34 & 9.34 & 0.66\\
\midrule
Oracle & 4.56 & 2.78 & 2.86 & 2.73 & 4.20 & 5.17 & 6.27 & 5.17 & 5.86 & 5.70 & 0.71\\
\bottomrule\end{tabular}\endgroup
\space
% END inlined layout_kit/PORTAL_layout_kit/tables/dynamic_4.tex
\end{table}

\begin{figure}[p]
\centering
\includegraphics[width=\linewidth]{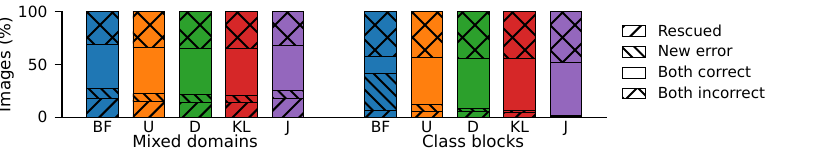}
\caption{Outcome decomposition under dynamic arrival conditions. Color identifies the selected rule and texture identifies whether a sample is rescued, newly harmed, correct under both outputs, or incorrect under both. The components are a linear decomposition of the reported source accuracy, selected accuracy and sample-level new errors on the common evaluation population.}
\label{fig:dynamic-outcomes}
\end{figure}

\subsection{Multiclass annotation noise}
To test the annotation channel, every requested class label is replaced by a uniformly chosen different class with probability $\eta\in\{0,0.1,0.2,0.4\}$. A corrupted class is shared by every candidate's correctness evaluation. This can move rescue, harm and neutral outcomes in either direction, as occurs in multiclass annotation. Candidate logits and evaluation labels remain unchanged. Image-keyed randomization couples noise across methods and noise levels; each acquired instance still costs one label.

The complete primary four-action PACS banks and all five stream orders are retained. Joint selection uses the same prior, label price, default, carryover rule and query cap as the uncorrupted experiment. Every zero-noise result reproduces its primary counterpart. The independently checked traces reconstruct acquired-label counts, actual class replacements and all evaluation errors from the prediction banks.

Joint accuracy varies from 59.31\% at zero corruption to 59.25\% at corruption probability $0.4$, while mean queried labels increase from 130.8 to 221.2. Three-label maximum-KL selection varies from 58.73\% to 58.43\% at its constant 393 labels. Fixed BN+flip remains at 59.68\% with zero labels throughout. The experiment shows how the implemented posterior rule responds to corrupted class evidence; Table~\ref{tab:class_noise} retains all tested rates and simple selection references.

% BEGIN inlined tables/expanded/class_noise.tex
\begin{table}[t]
\centering\small
\caption{Class-label corruption on the complete frozen PACS portfolio, five stream orders. A queried class is replaced uniformly by a different class with probability $\eta$. Accuracy is mean $\pm$ order SD; labels and newly introduced errors are means. Predictions and evaluation labels are unchanged.}\label{tab:class_noise}\begin{tabular}{llrrr}\toprule
$\eta$ & Rule & Accuracy (\%) & Labels & New errors (pp)\\\midrule
0 & Uniform & 58.97 $\pm$ 0.72 & 393.0 & 8.53 \\
0 & Disagreement & 58.93 $\pm$ 0.47 & 393.0 & 8.66 \\
0 & Maximum KL & 58.73 $\pm$ 0.23 & 393.0 & 8.55 \\
0 & \PORTAL{} & 59.31 $\pm$ 0.36 & 130.8 & 10.53 \\
\midrule
0.1 & Uniform & 59.00 $\pm$ 0.65 & 393.0 & 8.52 \\
0.1 & Disagreement & 58.81 $\pm$ 0.47 & 393.0 & 8.45 \\
0.1 & Maximum KL & 58.68 $\pm$ 0.32 & 393.0 & 8.53 \\
0.1 & \PORTAL{} & 59.29 $\pm$ 0.32 & 159.2 & 10.33 \\
\midrule
0.2 & Uniform & 58.93 $\pm$ 0.49 & 393.0 & 8.48 \\
0.2 & Disagreement & 58.96 $\pm$ 0.47 & 393.0 & 8.49 \\
0.2 & Maximum KL & 58.86 $\pm$ 0.36 & 393.0 & 8.34 \\
0.2 & \PORTAL{} & 59.39 $\pm$ 0.32 & 140.2 & 10.48 \\
\midrule
0.4 & Uniform & 58.65 $\pm$ 0.46 & 393.0 & 8.54 \\
0.4 & Disagreement & 58.44 $\pm$ 0.68 & 393.0 & 8.22 \\
0.4 & Maximum KL & 58.43 $\pm$ 0.46 & 393.0 & 8.29 \\
0.4 & \PORTAL{} & 59.25 $\pm$ 0.30 & 221.2 & 9.97 \\
\midrule
\bottomrule\end{tabular}
\end{table}
\space
% END inlined tables/expanded/class_noise.tex
\space
% END inlined sections/expanded_dynamic_appendix.tex
\FloatBarrier

\FloatBarrier
\section{Historical controller comparisons}
\label{app:controller_stress}
\subsection{Response to candidate-set expansion}
This experiment asks how the benefit of adding candidates changes with the controller. The locked confirmation uses one RobustBench CIFAR-10 Standard checkpoint, CIFAR-10/CIFAR-10-C and three corruption families. Ten new seeds (10--19) produce 120 episodes and 192,000 scored examples. The HILTTA-style controller follows a harmonized selection rule based on \citet{li2024hiltta}. In this experiment, \PORTAL{} evaluates candidates reset to the pretrained model through queried errors. This configuration isolates the controller--portfolio interaction. The main experiments additionally evaluate the shared-posterior formulation in Section~\ref{sec:expanded_method}.

Define the expansion effect as expanded-set error minus narrow-set error. It is $+14.3203$ points for HILTTA-style and $-1.2422$ points for \PORTAL{}. Their difference is the $15.5625$-point controller--portfolio interaction. This comparison follows each controller's own trajectory. It measures how expansion behaves under the two policies, rather than isolating selection on one common prediction pool.

\begin{table}[!htb]
\centering\small
\caption{Locked controller comparison across ten confirmation seeds. Errors include every scored episode input. Positive expansion effects mean that adding candidates increases error. A dash indicates that no expansion contrast is reported for that comparator.}
\label{tab:controller_stress}
\begin{tabular*}{\linewidth}{@{\extracolsep{\fill}}lrr@{}}
\toprule
Controller & Expanded-set error (\%) & Expansion effect (pp) \\
\midrule
HILTTA-style & 36.1698 & +14.3203 \\
\PORTAL{} & 18.9375 & -1.2422 \\
Harmonized SimATTA-style & 15.2719 & --- \\
Committee & 18.3438 & --- \\
\bottomrule
\end{tabular*}
\end{table}

\PORTAL{} reduces expanded-set error by 17.2323 points relative to HILTTA-style and wins in all ten seeds. The exact two-sided sign-test probability is $0.001953$. Table~\ref{tab:controller_stress} also reports the harmonized SimATTA-style and Committee comparators. SimATTA-style has the lowest error in this experiment. The interaction answers the question about expansion; these complete scores show the accompanying controller ranking.

All expanded-set controllers use eight labels per episode. The narrow-portfolio \PORTAL{} cell uses at most eight: 748 of 960 available labels, with all eight used in 70 of 120 episodes. Primary error includes queried inputs. Their labels do not enter current candidate generation or gradients. These details distinguish this experiment from the permanent query/deployment split used for the main shared-feedback comparisons.

% Source record: findings.md, 2026-07-19 R056 confirmation and audit.
% SHA256: adb8b0089d115fa02adee2b25706be2c313f8965bdb21fbccc86cd55eb8916d4

\subsection{Fixed-challenger admission on PACS and VLCS}
\label{app:paired_admission}
The PACS replay fixes a SimATTA incumbent trajectory, a horizontal-view challenger called U02, and 816 batches. U02 averages original and horizontally flipped predictions. Every controller in Table~\ref{tab:paired_admission_replay} uses the same recorded outputs. Positive gains mean higher accuracy than the incumbent. Paired admission selects U02 only when it makes strictly fewer queried errors; ties keep the incumbent. The \PORTAL{} row queries the three largest symmetric-KL disagreements. Its admission decisions coincide with queried ERM in all 816 batches.

Always deployment uses U02 throughout. The random gate admits with probability 29.41\%, matching the targeted controller's rate. Uniform admission samples three query positions and applies the same error comparison. The oracle selects from full-batch evaluation labels. It is an evaluation reference. Keeping these predictions fixed isolates the effect of acquiring and using evidence on the recorded trajectory.

\begin{table}[!htb]
\centering\small
\caption{PACS gains in percentage points over the matched incumbent. All inputs include current queried examples. Query-excluded removes current query positions, not historical labels or source-domain reuse. Brackets give 95\% random-policy replay intervals where available.}
\label{tab:paired_admission_replay}
\begin{tabularx}{\linewidth}{@{}Xrr@{}}
\toprule
Controller & All inputs & Current-query-excluded \\
\midrule
Always challenger & -5.7965 & -5.4337 \\
Random gate & $-1.6997\ [-2.2946,-1.1256]$ & $-1.5937\ [-2.1556,-1.0389]$ \\
Uniform three-query admission & $+1.1508\ [0.8320,1.4689]$ & $+0.9441\ [0.6104,1.2663]$ \\
\PORTAL{} paired admission & +2.2583 & +1.7863 \\
Full-batch oracle & +3.3280 & +3.1247 \\
\bottomrule
\end{tabularx}
\end{table}

The 79,928 PACS observations are eight readouts of 9,991 images, across 32 phases. They include the Photo source domain and images used in earlier adaptation. Current-query exclusion therefore does not produce a never-labeled population. A later replay first separates each batch into a common query pool and evaluation pool. Across five split/query seeds, symmetric-KL admission gains $1.8250\pm0.1222$ points, uniform admission $0.9428\pm0.0382$, always deployment $-5.7861\pm0.0974$, and the evaluation-pool oracle $3.6789\pm0.1470$. These repeats vary the split and queries on one recorded trajectory.

The VLCS transfer uses 872 batches and 85,832 observations over 32 phases. Always using U02 loses 11.2580 points; three-query symmetric-KL admission loses 0.3530. The avoided fraction of net accuracy loss is $1-0.3530/11.2580=96.86\%$. A ten-batch moving-block bootstrap gives an admission-gain interval of $[-0.5632,-0.1746]$ points. This measures the suppression of a harmful intervention; the remaining loss is reported in full.

% Sources: findings.md, R098/R099, R478, R482 and the 2026-09-14 replay audit.

% BEGIN inlined sections/expanded_algorithm.tex
\section{Joint selection algorithm and implementation}
\label{app:expanded_algorithm}
The selector receives candidate logits, query and deployment indices, image identifiers, and the remaining label allowance. It can request query labels; the evaluator alone holds deployment labels. This section specifies the finite-batch procedure. Appendix~\ref{app:joint_feedback} gives the separate fixed-confidence procedure and its sampling rules.

\subsection{Candidate representation and prior}
Candidates with identical hard outputs on every current input are merged. If a group contains the designated fixed default, it is retained as the representative; otherwise the first action is retained. A joint prediction pattern $p_s=(p_{s1},\ldots,p_{sK})$ groups current inputs into strata. Patterns occurring only in the query pool have target weight zero. For $C$ classes, let $\mathcal C_s$ contain the distinct predicted classes and let $\nu$ be the prior mass assigned to them. The initial parameters are
\begin{equation}
 \alpha^0_{sc}=\frac{1-\nu}{C}
       +\frac{\nu}{|\mathcal C_s|}\mathbf1\{c\in\mathcal C_s\},
 \qquad\nu=0.75.
 \label{eq:app_prior}
\end{equation}
The total prior mass is one. Each acquired label adds one class count. We store historical counts by candidate identity and prediction pattern, and count each acquired image once per deployment. When the same pattern reappears, its accumulated counts initialize the posterior. This reuse assumes that the pattern's class probabilities stay stable. The no-history comparison starts from Equation~\eqref{eq:app_prior} at every batch.

Target weights are empirical frequencies of unlabeled deployment inputs. The selector recomputes posterior means and covariances after every acquired class label using Equations~\eqref{eq:expanded_mu}--\eqref{eq:expanded_cov}. The leader is the default if its posterior mean is within $10^{-12}$ of the maximum; otherwise the first maximizer is used. Class labels supply complete candidate correctness vectors, so the annotation count grows by one regardless of the number of candidates.

\subsection{Reachability and query scores}
For each rival, Equation~\eqref{eq:expanded_reachable} bounds its best possible posterior advantage after the remaining queries. Write $d_{sa}=\pi_{s,p_{sa}}-\pi_{s,p_{sj}}$. The improvement from assigning the $k$th favorable label to a disagreement stratum is
\begin{equation}
 w_s(1-d_{sa})\frac{n_s}{(n_s+k)(n_s+k-1)}.
 \label{eq:app_reach_marginal}
\end{equation}
This quantity is nonnegative and decreases with $k$. Collecting the largest available increments, up to the remaining label count, solves the capped integer allocation. The rival remains competitive if its optimistic advantage exceeds $10^{-12}$. The leader is always retained. The full posterior perfect-information bound still includes every distinct action.

For each eligible target stratum, Equation~\eqref{eq:expanded_voi} computes one-label value on the competitive set. If all values are at most $10^{-14}$, the selector uses paired posterior variance reduction. For rival $a$ and leader $j$ that predict different classes in stratum $s$, the expected reduction is
\begin{equation}
 B_{s,aj}=\frac{w_s^2}{(n_s+1)^2}
 \left[\pi_{s,p_{sa}}+\pi_{s,p_{sj}}
       -(\pi_{s,p_{sa}}-\pi_{s,p_{sj}})^2\right].
\end{equation}
The fallback stratum score is
\begin{equation}
 T_s=\max_{a\ne j\,\mathrm{competitive}}
 \frac{B_{s,aj}}{(\mu_j-\mu_a)^2+
              \Sigma_{aa}+\Sigma_{jj}-2\Sigma_{aj}+10^{-12}}.
 \label{eq:app_variance_fallback}
\end{equation}
The numerator is the expected squared change in the pair's estimated accuracy difference after one label. The denominator gives priority to uncertain comparisons with small estimated gaps. This score can gather evidence when several labels are needed before the winner changes. Appendix~\ref{app:joint_feedback} analyzes a separate frequentist allocation procedure.

The highest-scoring eligible stratum is selected, with deterministic pattern-order ties. Its first unqueried input in stream order is labeled. No score uses evaluation labels. If both one-label value and the fallback vanish, or no target stratum contains an available query, the current batch decision stops. The mandatory-spending ablation instead fills its remaining budget from available query strata, including strata absent from the target batch when necessary.

\subsection{Execution sequence}
\begin{table}[H]
\centering\small
\caption{Finite-batch joint selection. Predictions, current target composition and available query locations are inputs; acquired class labels are the feedback.}
\label{tab:expanded_algorithm}
\begin{tabularx}{\linewidth}{rX}
\toprule
Step & Operation\\
\midrule
1 & Freeze all candidate logits and merge exact hard-output copies. Form joint prediction strata and target frequencies.\\
2 & Initialize each stratum from Equation~\eqref{eq:app_prior}, adding eligible historical acquired-label counts.\\
3 & Compute the posterior means/covariance, default-tied leader and all-candidate bound $U$.\\
4 & Stop if the local query allowance is exhausted or $U\le0.002$. Stop if no available target-stratum query remains.\\
5 & Compute capped rival reachability. Stop if every rival is unreachable within the remaining local allowance.\\
6 & Compute one-label decision values. On a zero-value plateau, use Equation~\eqref{eq:app_variance_fallback}. Stop if no available query has positive score.\\
7 & Acquire one class label in the chosen stratum. Update its posterior and historical count once, record its identifier and decrement the annotation allowance. Return to step 3.\\
8 & Deploy the posterior leader's frozen predictions on the current target inputs. Apply the experiment's prescribed subsequent model update.\\
\bottomrule
\end{tabularx}
\end{table}

\paragraph{Budget handling.}
The fixed-prediction PACS and VLM experiments accrue three labels per batch, carry unspent credits, and permit at most twelve acquisitions within a batch. The complete CIFAR frozen-selection experiment caps the stored balance at twelve as well as the within-batch count. Fixed-three baselines acquire three eligible labels, including on agreement inputs when needed. Their budget is reduced only when the eligible pool has fewer than three examples. The learned-readout experiments use three total labels in each batch: the selector can spend part of this budget on output evidence and uniform sampling supplies the remainder for the subsequent learned state. That total includes all labels used for either purpose.

\paragraph{Computational structure.}
Posterior means require $O(SK)$ operations and the full covariance $O(SK^2)$. The finite-horizon reachability calculation processes at most $SKB$ marginal increments for local allowance $B$, with sorting in the current implementation. One-label value evaluates $C$ possible outcomes per eligible stratum and competitive action. The matrices have candidate dimension, independent of the feature encoder's parameter count. Candidate generation remains a separate cost: four ResNet actions require original/flip forwards under source and batch statistics; eight-view CLIP requires all specified image encodings. Resource tables report these operations separately from selection arithmetic.

\subsection{Convex-readout implementation}
Section~\ref{app:convex-readout} defines the supervised simplex objective, its
Hessian and target-Fisher query score. The implementation centers every
candidate's class logits, initializes the weights uniformly, and solves the
strictly convex objective with analytic-gradient SLSQP at tolerance $10^{-9}$.
Weights reset at every batch and are refitted after each of at most three
acquired labels. The primary regularizer is $\lambda=1$; the complete grid also
retains $0.1$ and $10$.

The comparison uses the same queried examples for each fitted mixture and its
discrete ERM control. Uniform, disagreement and maximum-pair-KL acquisition use
their declared score, while target Fisher recomputes Equation~\eqref{eq:expanded-target-fisher}
after every label. Evaluation inputs contribute unlabeled logits to the target
curvature and supply labels only to the scorer. Tables~\ref{tab:expanded-convex-pacs}--\ref{tab:expanded-convex-vlcs}
contain every regularizer, acquisition, output rule and probability metric.
\space
% END inlined sections/expanded_algorithm.tex
\FloatBarrier
% BEGIN inlined sections/expanded_acquisition_controls.tex
\section{Acquisition and stopping controls}
\label{app:acquisition-controls}

\subsection{Changing only the current query locations}
The query comparison uses four fixed PACS candidates and five arrival orders. Each policy starts every batch from the same history, acquired labels, prior, deployment weights, and stopping rules. It can buy at most three labels, for a stream ceiling of 393. Only the choice of query inputs differs.

Uniform acquisition samples an eligible input. Disagreement acquisition prioritizes inputs where the active candidates predict different classes. The KL rule maximizes pairwise symmetric divergence. \PORTAL{} uses one-label decision value, with paired variance reduction when that value is zero. Each rule uses a fixed random seed for ties. After the comparison, \PORTAL{} supplies the next batch's shared history. This gives all query rules the same starting evidence.

% BEGIN inlined tables/expanded/acquisition_table.tex
\begin{table}[H]
\centering\small
\caption{Acquisition with a common posterior, history, available budget and stopping rule. Only the current query locations change. Means $\pm$ SD cover five PACS stream orders.}
\label{tab:common-posterior-acquisition}
\begin{tabular}{lrrr}
\toprule
Acquisition & Accuracy (\%) & Labels & Batch harm (pp) \\
\midrule
Uniform & 59.35 $\pm$ 0.13 & 132.8 & 1.78 \\
Disagreement & 59.31 $\pm$ 0.23 & 96.8 & 1.86 \\
Maximum-pair KL & 59.33 $\pm$ 0.26 & 118.8 & 1.85 \\
PORTAL value & 59.32 $\pm$ 0.34 & 85.6 & 1.86 \\
\bottomrule
\end{tabular}
\end{table}
\space
% END inlined tables/expanded/acquisition_table.tex

Table~\ref{tab:common-posterior-acquisition} shows that the four accuracy means
lie within 0.09 percentage points. PORTAL uses the fewest labels, 85.6 per
stream. Disagreement is the closest alternative at 96.8 labels and 59.31\%
accuracy. Uniform and maximum-pair KL spend 132.8 and 118.8 labels. The result
attributes most of the saving in this controlled replay to the shared posterior
and stopping logic; decision-value acquisition further reduces cost relative to
the other query locations.

\subsection{Purchasing labels after an observed stop}
For every PORTAL stop, we copy the posterior and remaining eligible queries.
The counterfactual continuation ranks those inputs by maximum-pair symmetric KL
without labels, then reveals prefixes of one, three and five query labels. These
labels affect only the copied current decision and never enter future stream
history. We record the stop reason, action change and held-out accuracy change.

The model-side check computes the exact one-step posterior value of the first
continuation query and verifies that it does not exceed the all-candidate bound
$U$. For a reachability stop, the original candidate remains a posterior
maximizer after every tested prefix within the original local budget. A tie can
change the deterministic output without contradicting this property.

% BEGIN inlined tables/expanded/continuation_table.tex
\begin{table}[H]
\centering\small
\caption{Counterfactual labels after stopping. Extra labels follow a fixed maximum-pair-KL order and never enter future history. Accuracy change is positive when continuation reduces held-out errors.}
\label{tab:post-stop-continuation}
\begin{tabular}{llrrrr}
\toprule
Setting & Stop reason & Stops & Extra & Choice change & $\Delta$ accuracy (pp) \\
\midrule
Primary & Unreachable & 653 & 1.0 & 0.00\% & +0.000 $\pm$ 0.000 \\
Primary & Unreachable & 653 & 3.0 & 0.00\% & +0.000 $\pm$ 0.000 \\
Primary & Unreachable & 653 & 5.0 & 0.31\% & +0.000 $\pm$ 0.051 \\
Price 0.05 & EVPI bound & 498 & 1.0 & 1.20\% & -0.075 $\pm$ 0.052 \\
Price 0.05 & EVPI bound & 498 & 3.0 & 2.21\% & -0.019 $\pm$ 0.100 \\
Price 0.05 & EVPI bound & 498 & 5.0 & 2.81\% & -0.012 $\pm$ 0.141 \\
\bottomrule
\end{tabular}
\end{table}
\space
% END inlined tables/expanded/continuation_table.tex

At the primary price, 653 of 655 batches stop because no rival can overtake the leader within the remaining budget. The other two exhaust the budget. Five extra queries change the tie-broken action in 0.31\% of stopped states and leave mean accuracy unchanged. At price 0.05, 498 batches stop because $U\le0.05$. Their mean $U$ is 0.0377, and the exact value of the first extra query is 0.00011. One, three, and five extra queries change accuracy by $-0.075$, $-0.019$, and $-0.012$ points. These are observed averages; the theorem bounds expected gain under the label model.

\subsection{Prior and historical evidence}
% BEGIN inlined tables/expanded/sensitivity_table.tex
\begin{table}[H]
\centering\small
\caption{Prior, stopping-price and history sensitivity on five PACS stream orders. All settings were fixed before this grid was evaluated.}
\label{tab:selector-sensitivity}
\begin{tabular}{lrrrr}
\toprule
Variant & Accuracy (\%) & Labels & Batch harm & New errors \\
\midrule
Primary & 59.31 $\pm$ 0.36 & 130.8 & 1.92 & 10.53 \\
Label price 0.01 & 59.31 $\pm$ 0.36 & 129.8 & 1.92 & 10.53 \\
Label price 0.025 & 59.32 $\pm$ 0.39 & 123.8 & 1.91 & 10.46 \\
Label price 0.05 & 59.37 $\pm$ 0.34 & 108.0 & 1.91 & 10.52 \\
Prior mass 0.25 & 59.31 $\pm$ 0.34 & 137.0 & 1.89 & 10.42 \\
Prior mass 4 & 59.34 $\pm$ 0.24 & 156.4 & 1.88 & 10.36 \\
Vote mass 0 & 59.35 $\pm$ 0.35 & 136.0 & 1.90 & 10.51 \\
Vote mass 0.5 & 59.33 $\pm$ 0.33 & 141.0 & 1.88 & 10.43 \\
Vote mass 0.9 & 59.32 $\pm$ 0.35 & 135.4 & 1.91 & 10.48 \\
No history & 58.98 $\pm$ 0.26 & 390.6 & 1.15 & 6.98 \\
History decay 0.9 & 59.28 $\pm$ 0.57 & 368.0 & 1.52 & 8.97 \\
History decay 0.5 & 58.90 $\pm$ 0.39 & 390.6 & 1.15 & 7.17 \\
One audit label & 59.44 $\pm$ 0.21 & 250.0 & 1.91 & 10.54 \\
Audit + surprise reset & 59.44 $\pm$ 0.16 & 274.8 & 1.91 & 10.56 \\
\bottomrule
\end{tabular}
\end{table}
\space
% END inlined tables/expanded/sensitivity_table.tex

Prior mass and vote mass have small effects on mean accuracy in this grid, while
label use ranges from 135.4 to 156.4. Historical evidence has a larger effect on
annotation cost. Removing history or multiplying a matched pattern's counts by
0.5 whenever it reappears spends 390.6 labels, compared with 130.8 for cumulative
history. Revisit decay 0.9 spends 368.0. One audit label from the highest-mass
available target stratum uses 250.0 labels, and resetting that stratum when its
observed class has posterior probability below 0.05 uses 274.8. These variants
receive the same three-credit stream budget. They show that stable reusable
evidence is a major source of efficiency in this PACS setting and make the cost
of forgetting and auditing explicit.

\subsection{Observed stratum scale and historical reuse}
% BEGIN inlined tables/expanded/stratum_scaling.tex
\begin{table}[H]
\centering\small
\caption{Observed prediction patterns. $C$ counts classes and $K$ counts candidates. Stratum counts are per batch; P95 is their 95th percentile. History mass is the fraction of deployment inputs whose pattern has acquired labels from earlier batches. RN50 denotes ResNet-50. Query rate is labels per batch.}
\label{tab:stratum-scaling}
\setlength{\tabcolsep}{3.2pt}
\begin{tabular}{lrrrrrr}
\toprule
Setting & $C$ & $K$ & Strata mean & P95 & History mass & Queries/batch \\
\midrule
PACS RN50 & 7 & 4 & 27.17 & 37.00 & 34.30\% & 1.00 \\
CIFAR-10-C & 10 & 4 & 30.62 & 42.00 & 20.20\% & 0.26 \\
CIFAR-100-C & 100 & 4 & 60.43 & 64.00 & 3.51\% & 0.16 \\
PACS CLIP & 7 & 4 & 7.53 & 9.00 & 0.01\% & 0.01 \\
OfficeHome CLIP & 65 & 4 & 42.28 & 48.55 & 0.31\% & 0.03 \\
VLCS CLIP & 5 & 4 & 8.34 & 12.00 & 8.74\% & 0.75 \\
VLCS CLIP + StatA & 5 & 5 & 10.17 & 15.00 & 8.72\% & 0.81 \\
\bottomrule
\end{tabular}
\end{table}
\space
% END inlined tables/expanded/stratum_scaling.tex

The finite candidate count does not determine the number of observed joint
prediction patterns. CIFAR-100-C has 60.43 strata per batch on average with
$C=100$, close to its batch size, while PACS CLIP has 7.53 because its readouts
usually agree. Adding StatA to VLCS raises the mean from 8.34 to 10.17 and the
95th percentile from 12 to 15. Historical evidence covers 34.30\% of target
mass in the PACS RN50 stream, 20.20\% on CIFAR-10-C, and only 0.31\% on
OfficeHome CLIP. These measurements connect class count, candidate diversity
and reusable evidence to the actual selector workload.
\space
% END inlined sections/expanded_acquisition_controls.tex
\FloatBarrier
% BEGIN inlined sections/expanded_cifar_learning.tex
\section{Causal residual-readout learning on corruption streams}
\label{app:cifar_residual}
The learned intervention adds a residual classifier to the fixed BN+flip output. For an input image, write $z_B$ for its BN+flip logits and $x$ for its unit-normalized classifier-input feature. All training labels available at time $t$ come from earlier batches. Matrix $X_{<t}$ contains their normalized features, and $R_{<t}$ contains targets $e_y-\operatorname{softmax}(z_B)$. With residual scale $s$ and regularization strength $\lambda$, the candidate is
\begin{equation}
 z_L(x)=z_B(x)+s xW_{t-1},\qquad
 W_{t-1}=(X_{<t}^{\top}X_{<t}+\lambda I)^{-1}X_{<t}^{\top}R_{<t}.
 \label{eq:cifar_residual_ridge}
\end{equation}
The transferred primary parameters are $\lambda=1$ and $s=4$. They are shared across both CIFAR datasets, all corruptions and all checkpoints. The CIFAR prediction banks contain the exact averaged features underlying the normalization-and-flip logits. The earlier PACS experiment uses its available original classifier-input features; both settings apply the same normalized-feature ridge objective.

\paragraph{Incremental computation.}
The ridge solution is updated without reconstructing its complete kernel matrix. Write $P=(X^{\top}X+\lambda I)^{-1}$ and let $X_q,R_q$ contain the current queried examples. The Woodbury identity gives
\begin{align}
 K_q&=PX_q^{\top}(I+X_qPX_q^{\top})^{-1},\\
 P'&=P-K_qX_qP,\\
 W'&=W+K_q(R_q-X_qW).
 \label{eq:cifar_recursive_ridge}
\end{align}
For the last identity, substitute $W=PX^{\top}R$ into $W'=P'(X^{\top}R+X_q^{\top}R_q)$ and use $P'X_q^{\top}=K_q$. Thus the incremental update solves the same ridge objective as Equation~\eqref{eq:cifar_residual_ridge}. Its state consists of a $512\times512$ inverse covariance and a $512\times C$ readout. Symmetrizing $P'$ removes floating-point asymmetry. The residual learner changes the readout; the backbone remains frozen.

\paragraph{Prediction and feedback order.}
Each current batch first receives frozen $z_B$ and $z_L$ predictions from the preceding state. Uniform querying, disagreement-first uniform querying and posterior querying each produce their own learning trajectory. The two empirical selectors compare errors on three current query labels, choosing $z_B$ on ties. Posterior selection may stop before using three comparison labels; the remaining slots are filled uniformly for future ridge training after its current output is fixed. The current labels are appended to the ridge objective only after the frozen current predictions are selected and scored. The update is performed for every batch, regardless of the selected output.

\paragraph{Coverage and accounting.}
Each corruption is an independent deployment episode with reset ridge and posterior states. Every episode includes all $10{,}000$ inputs, its fixed $5{,}000$ query and $5{,}000$ evaluation roles, and all 157 original batches, including the final 16 inputs. Every batch purchases three distinct query labels. A complete fifteen-corruption suite therefore uses $7{,}065$ labels per learned trajectory. Fixed output alone uses zero labels. Every learned trajectory counts all acquired labels, including batches that select the fixed output. The two datasets and three source-training checkpoints yield 90 episodes for each query policy. Evaluation labels enter the scorer after acquisition and are excluded from the learner and selector.

Results report fixed output, always-learned output and selected output on each policy's actual candidate trajectory. Accuracy, cross-entropy, sample-level newly introduced errors and batch-net harmful contributions are scored on the same evaluation examples. Across-checkpoint standard deviations describe the three training initializations. CPU readout timings exclude the previously completed backbone and candidate forwards.

\begin{table}[!tbp]
\centering
\caption{Causal residual-readout transfer on complete severity-five corruption suites. Accuracy is mean $\pm$ SD over three source checkpoints. Each learned trajectory uses 7,065 labels; decision labels are the subset used for current output comparison. All methods score the same 75,000 evaluation occurrences per checkpoint suite.}
\label{tab:cifar_causal_head}
\small
\begin{tabular}{llrrrr}
\toprule
Query policy & Output & Accuracy (\%) & $\Delta$BF & Total labels & Decision labels \\
\midrule
\multicolumn{6}{l}{\textit{CIFAR-10-C}}\\
--- & Fixed BN+flip & $79.65\pm0.56$ & $0.00$ & $0$ & $0$ \\
Uniform & Learned & $79.74\pm0.54$ & $+0.09$ & $7065$ & $7065.0$ \\
Uniform & Selected & $79.64\pm0.56$ & $0.00$ & $7065$ & $7065.0$ \\
Disagreement & Learned & $79.74\pm0.52$ & $+0.10$ & $7065$ & $7065.0$ \\
Disagreement & Selected & $79.65\pm0.58$ & $0.00$ & $7065$ & $7065.0$ \\
\PORTAL{} & Learned & $79.74\pm0.54$ & $+0.09$ & $7065$ & $50.3$ \\
\PORTAL{} & Selected & $79.66\pm0.56$ & $+0.01$ & $7065$ & $50.3$ \\
\midrule
\multicolumn{6}{l}{\textit{CIFAR-100-C}}\\
--- & Fixed BN+flip & $54.32\pm0.05$ & $0.00$ & $0$ & $0$ \\
Uniform & Learned & $54.31\pm0.09$ & $-0.01$ & $7065$ & $7065.0$ \\
Uniform & Selected & $54.32\pm0.05$ & $+0.00$ & $7065$ & $7065.0$ \\
Disagreement & Learned & $54.34\pm0.12$ & $+0.02$ & $7065$ & $7065.0$ \\
Disagreement & Selected & $54.34\pm0.07$ & $+0.01$ & $7065$ & $7065.0$ \\
\PORTAL{} & Learned & $54.31\pm0.08$ & $-0.01$ & $7065$ & $14.7$ \\
\PORTAL{} & Selected & $54.32\pm0.05$ & $0.00$ & $7065$ & $14.7$ \\
\bottomrule
\end{tabular}
\end{table}
\space
% END inlined sections/expanded_cifar_learning.tex
\FloatBarrier
% BEGIN inlined sections/expanded_active_appendix.tex
\section{Active model selection, shared learning and selector ablations}
\label{app:expanded_active}
\subsection{Contextual active model selection with permanent query roles}
Contextual Active Model Selection (CAMS)~\citep{liu2024contextual} chooses a model for each input. We apply its stochastic procedure to the four fixed PACS candidates: identity, BN, flip, and BN+flip. Each of four advice policies recommends one candidate. CAMS combines these policies through exponential weights and smoothed advice, then samples a model. It chooses that model before deciding whether to request the label. Label requests follow an entropy-dependent Bernoulli probability, and policy losses use inverse-probability weights. Only permanent query inputs can be labeled.

CAMS restarts at domain boundaries. The cumulative caps 131,393,655 are allocated to domains in proportion to their batch counts, corresponding to nominal budgets one,three,five. The algorithm can stop below its cap because its native requests remain stochastic. Five stream orders couple their own selector randomization to the fixed prediction banks. BN predictions are generated from the full unlabeled batch, so the per-example action timing applies to model selection while feature computation remains batch based.

With four single-candidate advice policies, CAMS smoothing assigns $7/16$ probability to the recommended model and $3/16$ to each other model. Model choice remains random even when one policy dominates. Our implementation retains this distribution. Twelve comparisons with the authors' routine reproduce label requests, counts, weights, and sampled prediction losses, including zero and exhausted budgets. Changing withheld evaluation labels leaves model choices and updates unchanged.

% BEGIN inlined tables/expanded/cams_complete.tex
\begin{table}[h]\centering\small
\caption{Full stochastic CAMS with four single-candidate advice policies, permanent query eligibility and per-example predictions. Mean $\pm$ SD over five PACS stream orders. Cap allocation is proportional to domain batch counts; realized requests follow CAMS.}\label{tab:cams_complete}
\begin{tabular}{rrrr}\toprule
Nominal labels/batch & Total cap & Actual labels & Accuracy (\%)\\\midrule
1 & 131 & 112.8 & 55.28 $\pm$ 0.45 \\
3 & 393 & 128.6 & 55.28 $\pm$ 0.44 \\
5 & 655 & 128.6 & 55.28 $\pm$ 0.44 \\
\bottomrule\end{tabular}\end{table}
\space
% END inlined tables/expanded/cams_complete.tex

On the same evaluation images, fixed identity, BN, flip, and BN+flip reach 50.88\%,58.37\%,50.90\%,59.68\%. CAMS makes per-example choices before the corresponding label request (Table~\ref{tab:cams_complete}). \PORTAL{} can use current query labels before choosing an output for the held-out batch. These experiments compare the methods at their respective decision times. The CAMS result uses the four single-candidate advice policies defined above.

\subsection{Human-in-the-loop adaptation with shared labels}
Human-in-the-loop test-time adaptation (HILTTA)~\citep{li2024hiltta} combines active querying with model updates. We use its released TENT-based implementation, which adapts normalization parameters by entropy minimization. The K-Margin rule scales features by uncertainty $1-(p_{(1)}-p_{(2)})$ and selects three query inputs by greedy Euclidean farthest-point selection, starting from the zero vector. Remaining inputs support unlabeled adaptation.

HILTTA takes one entropy step at each of seven learning rates in sequence:
\begin{equation}
 5\times10^{-5},\ 10^{-4},\ 2.5\times10^{-4},\ 5\times10^{-4},\
 10^{-3},\ 2.5\times10^{-3},\ 5\times10^{-3}.
\end{equation}
Each step continues the preceding state and saves model and optimizer parameters. Adam updates normalization affine parameters with $(\beta_1,\beta_2)=(0.9,0.999)$ and zero weight decay. The initial state lies outside the seven scored choices.

Each saved state receives a queried cross-entropy score and an anchor score, the mean squared difference between adapted and fixed-source softmax outputs on unqueried inputs. Each seven-score vector is shifted by its minimum and divided by its resulting maximum. A zero denominator produces an all-ones vector, following the implementation. The normalized loss vectors are added. The score history initializes on the first batch, then averages past and current scores with equal weight. The lowest score selects the saved state, using its first occurrence on ties.

The chosen state receives a supervised Adam update at rate $10^{-5}$. This step enables model gradients, computes acquired-label cross-entropy through a full-batch forward, and restores the TENT gradient configuration afterward. Model and score history continue across target domains; every new stream order restarts from the source checkpoint. The source anchor remains fixed.

We store the incoming and BN+flip predictions before acquiring current labels. A paired selector reuses HILTTA's labels and chooses BN+flip when it has strictly higher queried accuracy. The choice leaves HILTTA's learning updates and future queries unchanged. A further comparison uses predictions after the supervised update. Every output uses the same 393 acquired labels per order, each counted once.

% BEGIN inlined tables/expanded/hiltta_complete.tex
\begin{table}[h]\centering\small
\caption{Shared-label HILTTA on the same 4,155 permanent evaluation images per order. Every row uses 393 total labels. Incoming and BN+flip predictions are fixed before current labels; post-update predictions are available later. Initial-only uses BN+flip on the first batch and incoming output thereafter.}\label{tab:hiltta_complete}
\begin{tabular}{lrrrrr}\toprule
Order & Incoming & BN+flip & Shared gate & Post-update & Initial-only\\\midrule
201 & 79.06 & 79.95 & 79.21 & 79.28 & 79.04 \\
202 & 78.87 & 79.37 & 78.92 & 79.16 & 78.87 \\
203 & 79.78 & 80.31 & 80.05 & 79.88 & 79.78 \\
204 & 81.35 & 81.85 & 81.47 & 81.76 & 81.35 \\
205 & 80.00 & 80.82 & 80.26 & 80.17 & 80.05 \\
\midrule
Mean $\pm$ SD & 79.81 $\pm$ 0.98 & 80.46 $\pm$ 0.94 & 79.98 $\pm$ 1.00 & 80.05 $\pm$ 1.04 & 79.82 $\pm$ 0.99 \\
\bottomrule\end{tabular}\end{table}
\space
% END inlined tables/expanded/hiltta_complete.tex

The shared-label selector reaches 79.98\%, the incoming output 79.81\%, the fixed BN+flip output 80.46\% and the post-update output 80.05\%. These scores describe distinct available outputs of the same acquired-label learning trajectories. A frozen model uses running source statistics and receives a BN+flip intervention; the adaptive incoming model already uses batch statistics and continuously changes parameters. Their difference includes model state, candidate increment and acquisition, while the fixed-bank K-Margin control isolates acquisition on common predictions.

\subsection{Full posterior and empirical-rule ablations}
% BEGIN inlined tables/expanded/pilot_complete.tex
\begin{table}[t]\centering\small
\caption{Fixed-bank portfolio and selector ablations on PACS, five complete stream orders. Default-tied ERM uses BN+flip on empirical ties; first-tied ERM uses the first action. Accuracy is mean $\pm$ SD, and labels are actual per-stream means. Mean logits is an additional output function.}\label{tab:pilot_complete}
\begin{tabular}{lrrrr}\toprule
 & \multicolumn{2}{c}{Two candidates} & \multicolumn{2}{c}{Four candidates}\\
Rule & Accuracy & Labels & Accuracy & Labels\\\midrule
Identity & 50.88 $\pm$ 0.00 & 0.0 & 50.88 $\pm$ 0.00 & 0.0 \\
BN & --- & --- & 58.37 $\pm$ 0.14 & 0.0 \\
Flip & --- & --- & 50.90 $\pm$ 0.00 & 0.0 \\
BN+flip & 59.68 $\pm$ 0.28 & 0.0 & 59.68 $\pm$ 0.28 & 0.0 \\
Mean logits & 60.96 $\pm$ 0.12 & 0.0 & 61.63 $\pm$ 0.16 & 0.0 \\
Uniform, default ties & 59.27 $\pm$ 0.70 & 393.0 & 58.97 $\pm$ 0.72 & 393.0 \\
Disagreement, default ties & 59.21 $\pm$ 0.20 & 393.0 & 58.93 $\pm$ 0.47 & 393.0 \\
Maximum KL, default ties & 59.31 $\pm$ 0.31 & 393.0 & 58.73 $\pm$ 0.23 & 393.0 \\
Uniform, first ties & 55.51 $\pm$ 1.13 & 393.0 & 55.45 $\pm$ 0.66 & 393.0 \\
Disagreement, first ties & 57.83 $\pm$ 0.41 & 393.0 & 57.36 $\pm$ 0.82 & 393.0 \\
Maximum KL, first ties & 58.13 $\pm$ 0.35 & 393.0 & 57.36 $\pm$ 0.38 & 393.0 \\
\PORTAL{} & 59.31 $\pm$ 0.25 & 112.2 & 59.31 $\pm$ 0.36 & 130.8 \\
\PORTAL{}, constant three labels & 59.29 $\pm$ 0.44 & 393.0 & 59.39 $\pm$ 0.19 & 393.0 \\
\PORTAL{}, no history & 59.21 $\pm$ 0.54 & 392.4 & 58.98 $\pm$ 0.26 & 390.6 \\
\PORTAL{}, no history/three labels & 59.00 $\pm$ 0.57 & 393.0 & 58.90 $\pm$ 0.28 & 393.0 \\
\PORTAL{}, no competitive pruning & 59.52 $\pm$ 0.20 & 391.8 & 59.21 $\pm$ 0.23 & 392.4 \\
Batch oracle & 61.68 $\pm$ 0.40 & --- & 62.87 $\pm$ 0.28 & --- \\
\bottomrule\end{tabular}\end{table}
\space
% END inlined tables/expanded/pilot_complete.tex
Table~\ref{tab:pilot_complete} retains every distinct fixed prediction, the mean-logit output, both empirical tie conventions and each posterior ablation. All rows use identical evaluation images within a stream. The two-action portfolio contains identity and BN+flip; the four-action portfolio adds BN and flip separately. Adding exact copies to make four and eight nominal candidates, respectively, leaves the joint output and query sequence unchanged. This is checked on every batch of all five streams.

Constant-three spending turns off stopping and fills the available query budget even on agreement inputs. Removing historical evidence reinitializes the prior on each batch. Removing competitive pruning retains all candidates when evaluating one-label value. These changes isolate the implemented sources of budget savings. The posterior perfect-information calculation always contains all distinct candidates, including in the competitive-pruning variant.
\space
% END inlined sections/expanded_active_appendix.tex

\FloatBarrier
\section{Stopping, changing evidence, and expanded candidate sets}
\label{app:round_evidence}
This section connects the additional experiments to the decision studied in the main text. The stopping comparison holds the original PACS candidates fixed. The expanded portfolio makes the mean-logit ensemble available to the selector. The finite-pool calculation measures how difficult a real candidate difference is to identify. The conditional-label experiment changes that difference without changing the prediction patterns.

\subsection{Protocols and label accounting}
Three stopping experiments use the same four PACS candidates but different budget and history rules. The original experiment carries all unused credits, allows twelve queries per batch, and uses 130.8 labels on average. The independent-history experiment caps stored credits at four times the per-batch allowance. With three new credits per batch, it allows at most twelve queries and uses 140.0 labels. The shared-history experiment gives each copied state three current queries and uses 85.6 labels. Its decisions all start from the reference policy's history.

The five-candidate PACS, dynamic, CIFAR-10.1, and corruption-switch studies use three credits per batch with a stored balance capped at twelve. Mean logits is the default on ties. Each selector has its own history and label record. All labels, including audit labels, count toward cost. Fixed outputs use no labels. Every comparison retains the same permanent query and evaluation roles.

\subsection{Stopping with independent and matched histories}
\label{app:round_stopping}
The stopping study uses five PACS stream orders and the same four frozen candidate predictions as the original acquisition controls. Full \PORTAL{} uses both the value bound and reachable-winner test. Reachability-only disables the value-bound trigger. Value-bound-only disables the reachability trigger. Every variant retains the competitive candidate set for query scoring. If no positive query score remains, the latter variants select by target-stratum mass while waiting for their own stopping condition or budget exhaustion.

For the confidence comparator, approximate each posterior difference by a normal distribution with mean gap $d_a$ and variance $s_a^2$ from Theorem~\ref{thm:expanded_stop}. The comparator stops when the sum of the normal lower-tail probabilities over rivals is at most $0.05$. This is a Gaussian confidence proxy. The mean-gap rule instead stops when the leading posterior mean exceeds the runner-up by at least $0.05$. These are the fixed thresholds tested here.

For per-batch credit $b\in\{1,3,5\}$, stored credit is capped at $4b$. The implementation sets the per-batch query cap to $\max(12,4b)$; available credit therefore limits a batch to at most $4b$ queries. Independent-stream variants maintain their own histories. In the matched-current-batch experiment, every policy starts from a copy of the same history and record of acquired labels. Only the full-policy branch supplies the history for the next batch. Table~\ref{tab:round_stop_all} reports both experiments.
\begin{table}[htbp]
\centering\small
\setlength{\tabcolsep}{4pt}
\caption{Stopping controls on five PACS stream orders. Accuracy is mean $\pm$ order SD in percent. Labels count actual purchases. Harm is the positive batch-level accuracy loss relative to the unchanged model, averaged over all evaluation inputs, in percentage points.}
\label{tab:round_stop_all}
\begin{tabular}{lllr rr}
\toprule
History & Credit & Stop rule & Accuracy & Labels & Harm \\
\midrule
Independent & 1 & \PORTAL{} & $59.34\pm0.11$ & 90.4 & 1.74 \\
Independent & 1 & Reachability only & $59.34\pm0.11$ & 90.4 & 1.74 \\
Independent & 1 & Value bound only & $59.06\pm0.26$ & 131.0 & 1.69 \\
Independent & 1 & Confidence proxy & $59.06\pm0.26$ & 131.0 & 1.69 \\
Independent & 1 & Mean-gap threshold & $59.06\pm0.26$ & 131.0 & 1.69 \\
Independent & 3 & \PORTAL{} & $59.35\pm0.34$ & 140.0 & 1.88 \\
Independent & 3 & Reachability only & $59.35\pm0.34$ & 140.0 & 1.88 \\
Independent & 3 & Value bound only & $59.39\pm0.19$ & 392.8 & 1.82 \\
Independent & 3 & Confidence proxy & $59.39\pm0.19$ & 392.8 & 1.82 \\
Independent & 3 & Mean-gap threshold & $59.29\pm0.26$ & 392.2 & 1.88 \\
Independent & 5 & \PORTAL{} & $59.28\pm0.33$ & 130.0 & 1.95 \\
Independent & 5 & Reachability only & $59.28\pm0.33$ & 130.0 & 1.95 \\
Independent & 5 & Value bound only & $59.40\pm0.36$ & 654.0 & 1.81 \\
Independent & 5 & Confidence proxy & $59.40\pm0.36$ & 654.0 & 1.81 \\
Independent & 5 & Mean-gap threshold & $59.35\pm0.37$ & 652.4 & 1.81 \\
Matched & 3 & \PORTAL{} & $59.32\pm0.34$ & 85.6 & 1.86 \\
Matched & 3 & Reachability only & $59.32\pm0.34$ & 85.6 & 1.86 \\
Matched & 3 & Value bound only & $59.32\pm0.34$ & 392.8 & 1.86 \\
Matched & 3 & Confidence proxy & $59.32\pm0.34$ & 392.8 & 1.86 \\
Matched & 3 & Mean-gap threshold & $59.28\pm0.31$ & 391.0 & 1.86 \\
\bottomrule
\end{tabular}
\end{table}

With three credits, full \PORTAL{} and reachability-only match in all 655 batches. Of these, 654 stop because the posterior winner is unreachable by a rival; one exhausts the budget. The value-bound trigger does not fire at price $0.002$. Full and reachability-only also coincide at the other two credits and in the matched-current-batch comparison. Relative to value-bound-only, full \PORTAL{} saves 252.8 labels at three credits. The paired accuracy difference is $-0.043$ percentage points, with a 95\% stream-order $t$ interval of $[-0.315,0.228]$. This interval describes order variation on one checkpoint and the same target images.

\FloatBarrier
\subsection{An ensemble inside the candidate set}
\label{app:round_portfolios}

\begin{table}[!htb]
\centering
\captionsetup{font=small,skip=4pt}
\caption{Fixed accuracy and selector accuracy (\%) / labels on common evaluation inputs. $K$ counts selectable candidates. Labels are actual totals per suite or stream.}
\label{tab:round_overview}
\fontsize{8}{9.3}\selectfont
\setlength{\tabcolsep}{2pt}
\begin{tabular*}{\linewidth}{@{\extracolsep{\fill}}lrrr@{}}
\toprule
Data ($K$) & Fixed & \PORTAL{} & Top-KL \\
\midrule
CIFAR-10-C (4) & 79.65 & 79.26/612.3 & 79.20/7065 \\
CIFAR-100-C (4) & 54.32 & 53.87/381.7 & 53.76/7065 \\
PACS (5) & 61.63 & 62.07/144.2 & 60.57/393 \\
CIFAR-10.1 (5) & 89.87 & 89.59/1.9 & 89.06/96 \\
\bottomrule
\end{tabular*}
\par\vspace{3pt}\raggedright\fontsize{7.8}{8.7}\selectfont
Fixed outputs: BN+flip on CIFAR-C, mean logits on PACS, and flip on CIFAR-10.1. They use no labels.
\end{table}
The five-candidate portfolio contains identity, BN, flip, BN+flip, and the mean of their logits. All empirical and posterior selectors can select the ensemble. The primary PACS replay uses the original Photo-source checkpoint and five target orders. Table~\ref{tab:round_primary} reports cumulative history, decay, and ordinary queried-error selection. In a decay variant, a pattern's historical counts are multiplied by 0.9 each time that pattern reappears.
\begin{table}[htbp]
\centering\small
\setlength{\tabcolsep}{4pt}
\caption{Five-candidate PACS replay, five stream orders. The fixed ensemble is available to every selector and resolves ties. Harm is measured against that ensemble in percentage points. The audited detector is a separate exploratory variant.}
\label{tab:round_primary}
\begin{tabular}{lrrr}
\toprule
Rule & Accuracy (\%) & Labels & Harm \\
\midrule
Fixed BN+flip & $59.68\pm0.28$ & 0.0 & 5.05 \\
Fixed mean logits & $61.63\pm0.16$ & 0.0 & 0.00 \\
Uniform empirical selection & $60.45\pm0.72$ & 393.0 & 2.08 \\
Top-KL empirical selection & $60.57\pm0.43$ & 393.0 & 2.24 \\
\PORTAL{}, cumulative & $62.07\pm0.71$ & 144.2 & 1.55 \\
\PORTAL{}, decay 0.9 & $62.11\pm0.67$ & 250.4 & 1.62 \\
Audited change detector & $61.86\pm0.51$ & 263.2 & 1.74 \\
\bottomrule
\end{tabular}
\end{table}

Cumulative \PORTAL{} improves on fixed mean logits in four of five orders. The mean paired gain is $0.448$ percentage points, with a 95\% stream-order $t$ interval of $[-0.345,1.240]$. The audited detector was developed on the synthetic change experiment. It uses an eight-observation window, threshold 2, and one paid audit per batch. It reaches 61.86\% with 263.2 labels on this replay, compared with 62.07\% and 144.2 for cumulative history. Its additional monitoring cost does not improve this real-stream result.

\paragraph{Regenerated arrival conditions.}
The dynamic extension tests ten conditions with three regenerated orders each. Conditions change the source checkpoint, batch size, or arrival order. Every comparison holds the four original predictions fixed and adds their mean logits as a fifth candidate. Table~\ref{tab:round_dynamic} gives all conditions, including the original four-candidate Top-KL result and its five-candidate extension.
\begin{table}[htbp]
\centering\scriptsize
\setlength{\tabcolsep}{4pt}
\caption{All five-candidate PACS dynamic conditions, averaged over three orders. Condition lists source, arrival order, and batch size. Paired entries are accuracy (\%)/labels. Best fixed is the hindsight best single fixed output for that condition. KL-4 is the original four-candidate Top-KL accuracy from Table~\ref{tab:dynamic-matrix-performance}; KL-5 uses the expanded portfolio.}
\label{tab:round_dynamic}
\begin{tabular}{@{}lrrrrrr@{}}
\toprule
Condition & Best fixed & Ensemble & Cumulative & Decay & KL-5 & KL-4 \\
\midrule
Art, domain, 64 & 78.89 & 78.89 & 78.01/79.0 & 77.90/138.7 & 77.89/377.3 & 77.45 \\
Cartoon, domain, 64 & 82.20 & 82.20 & 81.79/51.3 & 81.65/72.3 & 81.72/362.7 & 80.97 \\
Photo, domain, 16 & 60.79 & 60.79 & 60.50/98.7 & 60.66/156.3 & 60.20/1561.7 & 58.12 \\
Photo, domain, 32 & 61.54 & 61.54 & 61.30/111.0 & 61.39/202.3 & 60.52/783.0 & 59.25 \\
Photo, domain, 128 & 61.47 & 61.47 & 61.94/156.7 & 61.76/182.3 & 60.71/198.0 & 59.29 \\
Photo, class, 64 & 50.90 & 45.25 & 51.74/211.3 & 53.60/330.0 & 49.58/393.0 & 54.11 \\
Photo, gradual, 64 & 60.92 & 60.92 & 60.99/138.0 & 60.67/236.0 & 59.98/390.0 & 59.40 \\
Photo, mixed, 64 & 60.56 & 59.15 & 60.02/90.3 & 59.97/197.7 & 59.05/390.3 & 58.34 \\
Photo, reverse, 64 & 61.34 & 61.34 & 61.84/143.7 & 61.68/238.7 & 60.19/393.0 & 58.91 \\
Sketch, domain, 64 & 63.02 & 55.22 & 62.03/66.0 & 61.83/137.7 & 59.64/287.7 & 62.71 \\
\bottomrule
\end{tabular}
\end{table}

Under class-block arrival, flip is the best fixed output at 50.90\%. Cumulative history reaches 51.74\% with 211.3 labels, and decay reaches 53.60\% with 330.0. Four-candidate Top-KL reaches 54.11\% with 393 labels. Interpolating between fixed flip and Top-KL gives 53.597\% expected accuracy at 330 labels. This mixture is fitted from the measured endpoints after evaluation.

\FloatBarrier
\subsection{CIFAR-10.1 and corruption-switch streams}
\label{app:round_confirmation}
Version 6 of the CIFAR-10.1 benchmark contains 2,000 images \citep{recht2018cifar101}. A fixed hash of image identity assigns half to querying and half to evaluation. We regenerate predictions for three source checkpoints and three image orders, using batches of 64 and the original CIFAR-C preprocessing. Model weights stay fixed. The ensemble averages the four original logits, and every posterior selector uses it on ties. Table~\ref{tab:round_cifar101} averages orders within each checkpoint before reporting the mean and checkpoint SD.
\begin{table}[htbp]
\centering\small
\setlength{\tabcolsep}{4pt}
\caption{CIFAR-10.1 with three checkpoints and three orders per checkpoint. Each order has 1,000 evaluation images. Accuracy is mean $\pm$ checkpoint SD after averaging orders within a checkpoint. Audit labels are included in cost.}
\label{tab:round_cifar101}
\begin{tabular}{lrr}
\toprule
Rule & Accuracy (\%) & Labels \\
\midrule
Fixed identity & $89.23\pm0.97$ & 0.0 \\
Fixed BN & $85.84\pm0.68$ & 0.0 \\
Fixed flip & $89.87\pm0.81$ & 0.0 \\
Fixed BN+flip & $86.22\pm0.42$ & 0.0 \\
Fixed mean logits & $89.69\pm0.70$ & 0.0 \\
Uniform empirical selection & $89.56\pm0.72$ & 96.0 \\
Top-KL empirical selection & $89.06\pm0.80$ & 96.0 \\
\PORTAL{}, cumulative & $89.59\pm0.69$ & 1.9 \\
\PORTAL{}, decay 0.9 & $89.59\pm0.69$ & 1.9 \\
Audited change detector & $89.59\pm0.69$ & 33.9 \\
\bottomrule
\end{tabular}
\end{table}

The corruption-switch replay uses saved severity-five CIFAR-10-C predictions. Each segment contains 1,600 consecutive base-image indices aligned with the saved 64-image normalization batches. A stream uses 6,400 distinct base images, so an image never reappears under another corruption within that stream. The original query/evaluation roles are preserved. The three schedules are Gaussian noise--brightness--shot noise--JPEG, shot noise--JPEG--impulse noise--brightness, and brightness--Gaussian noise--JPEG--shot noise. The same three source checkpoints are used for each schedule. Every selector starts with empty history at the beginning of a stream and carries it across segments.
\begin{table}[htbp]
\centering\small
\setlength{\tabcolsep}{4pt}
\caption{Sequential corruption-switch selection. Accuracy is mean $\pm$ checkpoint SD on the common evaluation images. Each stream has 100 batches and disjoint base images across segments. The fixed BN+flip output is the strongest fixed candidate in every complete schedule.}
\label{tab:round_corruption}
\begin{tabular}{llrr}
\toprule
Schedule & Rule & Accuracy (\%) & Labels \\
\midrule
Noise first & Fixed BN+flip & $76.67\pm0.63$ & 0.0 \\
Noise first & Fixed mean logits & $72.52\pm1.40$ & 0.0 \\
Noise first & Uniform empirical selection & $73.55\pm0.65$ & 300.0 \\
Noise first & Top-KL empirical selection & $74.07\pm0.76$ & 300.0 \\
Noise first & \PORTAL{}, cumulative & $74.96\pm1.00$ & 57.3 \\
Noise first & \PORTAL{}, decay 0.9 & $74.98\pm1.05$ & 80.7 \\
Noise first & Audited change detector & $75.15\pm0.90$ & 144.0 \\
Alternating & Fixed BN+flip & $75.54\pm0.76$ & 0.0 \\
Alternating & Fixed mean logits & $72.08\pm1.53$ & 0.0 \\
Alternating & Uniform empirical selection & $72.84\pm0.23$ & 300.0 \\
Alternating & Top-KL empirical selection & $73.03\pm0.77$ & 300.0 \\
Alternating & \PORTAL{}, cumulative & $74.18\pm0.63$ & 57.3 \\
Alternating & \PORTAL{}, decay 0.9 & $74.39\pm0.75$ & 85.7 \\
Alternating & Audited change detector & $74.40\pm0.67$ & 144.7 \\
Brightness first & Fixed BN+flip & $76.02\pm0.69$ & 0.0 \\
Brightness first & Fixed mean logits & $73.35\pm1.76$ & 0.0 \\
Brightness first & Uniform empirical selection & $74.10\pm0.98$ & 300.0 \\
Brightness first & Top-KL empirical selection & $74.25\pm1.22$ & 300.0 \\
Brightness first & \PORTAL{}, cumulative & $75.39\pm0.80$ & 52.0 \\
Brightness first & \PORTAL{}, decay 0.9 & $75.30\pm0.95$ & 87.0 \\
Brightness first & Audited change detector & $75.46\pm0.75$ & 138.3 \\
\bottomrule
\end{tabular}
\end{table}

\FloatBarrier
\subsection{Exact paired-label calculations on real predictions}
\label{app:round_paired}
Within each corruption segment, we compare BN+flip with mean logits. The two outputs are correct or incorrect on the same underlying images. Their paired correctness difference takes values $1$, $-1$, or $0$. Its query-pool mean and variance determine the finite-pool outcome counts in Equation~\eqref{eq:expanded_hypergeom}. We compute the probability that a uniform sample selects BN+flip, including ties, and use the separate evaluation pool to identify the better output. We then draw 2,000 independent query subsets per checkpoint, segment, and budget to check the exact result.

The pool contains about 800 eligible inputs per segment. Budgets are 0, 3, 15, 75, and 300 pooled labels. These calculations select one output for a complete segment; the sequential selectors in Table~\ref{tab:round_corruption} instead select once per batch. Table~\ref{tab:round_pair_error} gives the exact probability of choosing the worse output. The zero-label rule always selects BN+flip. It is correct in the noise segments and wrong in the brightness and JPEG segments. Small nonzero budgets can disturb a correct default before enough evidence has accumulated to stabilize the empirical choice.
\begin{table}[htbp]
\centering\small
\setlength{\tabcolsep}{4pt}
\caption{Exact probability (\%) of selecting the worse of BN+flip and mean logits. Means cover three checkpoints. Columns give the number of labels sampled uniformly without replacement from one corruption segment. Tie decisions use BN+flip. Differences between schedules also reflect their distinct image-index blocks.}
\label{tab:round_pair_error}
\begin{tabular}{llrrrr}
\toprule
Schedule & Corruption & 3 labels & 15 labels & 75 labels & 300 labels \\
\midrule
Noise first & Brightness & 90.82 & 64.90 & 24.05 & 2.35 \\
Noise first & Gaussian noise & 11.53 & 11.59 & 1.73 & 0.00 \\
Noise first & Jpeg compression & 86.02 & 59.90 & 35.03 & 14.52 \\
Noise first & Shot noise & 11.66 & 16.63 & 7.61 & 1.06 \\
Alternating & Brightness & 91.41 & 67.66 & 30.35 & 4.71 \\
Alternating & Impulse noise & 12.40 & 15.79 & 5.53 & 0.06 \\
Alternating & Jpeg compression & 85.00 & 57.94 & 31.59 & 8.00 \\
Alternating & Shot noise & 14.73 & 22.30 & 14.01 & 2.98 \\
Brightness first & Brightness & 91.20 & 68.50 & 37.15 & 10.74 \\
Brightness first & Gaussian noise & 10.64 & 11.26 & 1.84 & 0.00 \\
Brightness first & Jpeg compression & 86.71 & 61.90 & 38.47 & 15.51 \\
Brightness first & Shot noise & 14.04 & 18.81 & 7.87 & 0.35 \\
\bottomrule
\end{tabular}
\end{table}

\FloatBarrier
\subsection{Historical evidence under a controlled label change}
\label{app:round_history}
The synthetic experiment uses four candidates, three classes, and two fixed prediction patterns. Each batch has 64 inputs: 16 query and 16 evaluation inputs in each pattern. Every batch therefore has the same stratum weights. The class probabilities are shared by query and deployment inputs. Before a change, the probabilities of classes 0, 1, and 2 are $0.725$, $0.225$, and $0.05$. An abrupt change exchanges the first two probabilities at batch 50. A gradual change interpolates between them over batches 50--70. The stable stream keeps the initial law. Five independent seeds generate the potential labels shared by all policies.

Cumulative history retains acquired counts. No-history selection resets the prior each batch. Decay multiplies a pattern's historical counts by 0.9 when it reappears. An audit variant buys one label from the largest available stratum before ordinary selection. Audit/reset also clears that stratum's counts when the observed class had a posterior probability below its threshold. The strata have equal deployment weights and query counts here. The deterministic tie rule repeatedly sends the first audit to the same stratum.
\begin{table}[htbp]
\centering\small
\setlength{\tabcolsep}{4pt}
\caption{Conditional-label change with fixed predictions and stratum weights. Accuracy is in percent; labels count purchases within the displayed period, including audits. Stable and abrupt columns cover the final 50 batches. Gradual columns cover the final 30, after the transition. Means cover five synthetic seeds.}
\label{tab:round_drift}
\begin{tabular}{lrr|rr|rr}
\toprule
History & \multicolumn{2}{c|}{Stable} & \multicolumn{2}{c|}{Abrupt} & \multicolumn{2}{c}{Gradual} \\
 & Acc. & Labels & Acc. & Labels & Acc. & Labels \\
\midrule
Cumulative & 72.21 & 0.0 & 22.54 & 0.0 & 22.92 & 0.0 \\
No history & 62.60 & 150.4 & 64.10 & 150.2 & 63.42 & 90.0 \\
Decay 0.9 & 72.21 & 150.4 & 67.66 & 149.8 & 70.88 & 85.4 \\
One audit per batch & 72.21 & 50.0 & 27.24 & 72.8 & 22.92 & 30.0 \\
Audit/reset 0.05 & 72.21 & 54.2 & 38.34 & 74.4 & 33.71 & 38.2 \\
Audit/reset 0.30 & 71.78 & 132.4 & 55.36 & 125.4 & 52.73 & 71.4 \\
\bottomrule
\end{tabular}
\end{table}

After the abrupt change, cumulative history never selects the expected-optimal action for five consecutive batches in any seed. Decay reaches this criterion after 7.0 batches on average. After the gradual change crosses the decision boundary, its mean delay is 10.4 batches. The reset threshold of 0.30 improves recovery relative to 0.05, but also causes 19.4 resets on average in the stable stream. These observations separate acquiring a diagnostic label from replacing enough old evidence to change the decision.

A window-based change detector improves recovery on the synthetic stream. On PACS, it uses more labels and reaches lower accuracy than cumulative history (Table~\ref{tab:round_primary}). This motivates the budget-matched checks below.

\subsection{Budgeted audits under feedback bias and label drift}
\label{app:w1_guarded}
We next compare guarded selection with two explicitly defined estimates of class probabilities. Exponentially weighted (EW) counts discount past labels by 0.95 per batch. Direct inverse-propensity weighting with EW (Direct IPW+EW) also corrects for unequal query probabilities. For query probability $\pi_j$ and eligible pool size $|Q|$, its label weight is $\min\{4,1/(\pi_j|Q|)\}$. The cap is inactive in the reported runs. Each comparator chooses the candidate with the highest estimated deployment accuracy. Both share the guarded selector's predictions and total label ceiling. These are our reconstructions: the original implementations named in the experiment specification were unavailable.

The two-pattern stream keeps predictions fixed and shares class probabilities between query and deployment inputs within each batch. Those probabilities stay stable, reverse abruptly, or change gradually. The guarded selector reserves periodic labels, buys a short burst after surprising evidence, and tests a switch before retaining it.

At a 120-label ceiling, its stable-period difference from the strongest fixed rule is $-0.122\pp$, with paired interval $[-0.150,-0.094]$ over 960 fresh streams. On 360 paired changing streams, post-change gains over same-budget EW are $+9.52\pp$ for abrupt change and $+4.87\pp$ for gradual change. Recovery starts within ten batches in 0.92 and 0.80 of streams, respectively. Raising only the ceiling to 160 improves gradual recovery to 0.94. The stable-period difference becomes $-0.100\pp$, with interval $[-0.128,-0.071]$. We report each budget separately.

The sparse protocol draws four of 300 hidden modes per batch. A mode is a subgroup with its own class probabilities; several modes share a prediction stratum. A tenth of each mode's examples have an atypical label distribution. Prioritized audits oversample these examples, creating feedback bias even before the distribution changes.

The sparse selector corrects acquired labels using their known query probabilities. Every input type must have positive query probability. A run-length posterior estimates how long the current label distribution has persisted. Its change hazard, the prior probability of change per batch, is 0.02. Under the 100-label ceiling, the selector starts with eight labels and reserves periodic audits. Surprising evidence triggers a burst capped at five labels per batch, which may continue for three batches after the trigger. Uniform and prioritized schedules use the same candidates and budget. Table~\ref{tab:w1_sparse_cells} reports 28 independently generated blocks of 30 seed streams each. Gradual-change estimates exceed EW with intervals spanning zero; recurrent-change gains have positive intervals under both schedules.

\begin{table}[!htb]
\centering\small
\caption{Post-change accuracy difference from reconstructed EW in percentage points on the 300-mode sparse protocol. Priority oversamples atypical examples within a mode; uniform draws do not. Each cell contains 28 blocks and 840 seed streams, with the same scored examples and 100-label ceiling for both methods.}
\label{tab:w1_sparse_cells}
\begin{tabular*}{.78\linewidth}{@{\extracolsep{\fill}}lrr@{}}
\toprule
Change & Prioritized audit & Uniform audit \\
\midrule
Abrupt & $+3.94\ [0.59,7.28]$ & $+4.75\ [1.36,8.14]$ \\
Gradual & $+1.69\ [-3.24,6.62]$ & $+1.58\ [-3.36,6.51]$ \\
Recurrent & $+3.28\ [0.61,5.95]$ & $+3.84\ [1.24,6.43]$ \\
\bottomrule
\end{tabular*}
\end{table}

Under 100 labels and prioritized audits, recovery within ten batches is 0.685, 0.798, and 0.642 at recurrent changes 25, 50, and 75. The target is 0.90 at each change. A separate 160-label variant changes both the budget and selector settings. It reaches 0.907, 0.933, and 0.893 across five fresh blocks. Its pooled recovery is 0.911, with bootstrap interval $[0.884,0.936]$. The final change remains below the target. Post-change accuracy exceeds same-budget EW by $8.55\pp$, with paired interval $[7.68,9.41]$ over 150 streams. Mean harmful-batch rates are 3.15\% under prioritized abrupt change and 4.62\% under gradual change. The gradual-change 95\% upper bound is 5.07\%, above the 5\% threshold.

A harmful choice has lower accuracy than the prespecified strong fixed candidate on the same deployment inputs. We count the fraction of harmful batches in each stream. At 100 labels, each seed contributes an average over six combinations of change scenario and query schedule. Across 840 seeds, the pooled mean is 3.73\%, with a one-sided 95\% $t$ upper bound of 3.91\%. Recurrent streams have a higher mean of 7.46\% across the two schedules. Individual six-scenario means exceed 5\% for 193 seeds, and eleven of 28 block intervals cross 5\%. These measurements locate the low pooled rate and the higher harm under repeated changes.

On real PACS predictions, the guarded four-candidate selector reaches 59.610\% with 131 labels. The strongest fixed candidate reaches 59.677\%, and reconstructed EW reaches 56.987\%. The stable-domain difference from the strongest fixed candidate is $-0.020\pp$. Five arrival orders share one checkpoint, so this comparison measures arrival-order variation.

On CIFAR-10.1, a separate fixed probability ensemble uses zero labels and scores 89.733\%, against 89.867\% for the strongest fixed candidate. The paired difference is $-0.133\pp$, with interval $[-1.216,+0.949]$ over three checkpoints. The noninferiority criterion allows a loss of at most 0.20\,pp. The point estimate meets it; the interval is too wide to establish it. Repeated orders are averaged within checkpoints. This experiment evaluates an additional fixed candidate.

\paragraph{Numerical and implementation checks.}
The stored CIFAR-10.1 predictions reproduce batch errors and stream totals. Fixed outputs in the dynamic experiment match the original predictions. Label counts include all audit requests. Replacing withheld evaluation labels leaves query choices unchanged. An independent finite-pool calculation samples the disagreement count first and then its favorable outcomes; it agrees with the exact probabilities. These checks validate the calculations. Statistical uncertainty uses the checkpoints and orders specified for each experiment.
\FloatBarrier

\section{Related work}
\label{app:related_work}
The literature offers many ways to change a deployed model. It also offers ways to judge whether a change helps. \PORTAL{} brings these two questions together through a shared label budget.

\subsection{Feedback and the risk of a wrong choice}
Adaptive labeling and label-free accuracy estimation help evaluate models after distribution shift \citep{DBLP:conf/nips/MittalMJN24,DBLP:conf/cvpr/LeeCGL24}. Accuracy--agreement relations provide another signal from unlabeled data \citep{DBLP:conf/nips/KimSBRK24}. Binary feedback reduces the amount of information requested from an annotator \citep{DBLP:conf/icml/LeeCKSGL25}. In our setting, a full class label judges every candidate on the same input.

Feedback can lose its value as the stream changes. Risk monitoring, protected entropy matching, and domain-shift detection track different signs of trouble \citep{DBLP:conf/nips/SchirmerJNN25,DBLP:conf/nips/BarSR24,DBLP:conf/iccv/ParkPKM25}. Recovery time and prediction latency add two costs beyond current accuracy \citep{zhou2026learnability,sreeram2026tempora}. Our drift experiments measure how quickly fresh labels can replace stale evidence in the selection decision.

The deployment population may also change in ways that a prediction pattern hides. Open-set inputs, malicious test samples, and noisy inputs challenge adaptation in different ways \citep{DBLP:conf/cvpr/GaoZ024,DBLP:conf/cvpr/ParkHMPO24,DBLP:conf/iclr/CaoZZLLZH25}. Domain compensation addresses open-set streams, while selective adaptation isolates a shifted modality \citep{DBLP:journals/corr/abs-2604-21772,DBLP:conf/icml/ChenZHJWF0B25}. Long-tailed classes add another source of unequal feedback \citep{wang2026long}. These settings make the link between queried and deployed inputs a substantive part of the decision.

\subsection{How candidate predictions change}
An intervention can update a model through self-training or prediction energy \citep{DBLP:conf/cvpr/Ma24,DBLP:conf/cvpr/YuanXHSSC24}. Disentangled objectives show that prediction entropy alone does not describe every useful update \citep{DBLP:conf/iclr/0004JLPSHY24}. Ranked entropy and entropy--energy duality refine the signal used for adaptation \citep{DBLP:conf/icml/HanNH25,DBLP:conf/nips/ParkWRK25}. Uncertainty calibration, Tsallis entropy, and sharpness-aware regularization offer further choices \citep{DBLP:journals/pami/TanCWZCZN25,DBLP:journals/corr/abs-2602-11743,DBLP:journals/corr/abs-2509-04977}.

Bayesian updates and online expectation--maximization change how a candidate uses incoming examples \citep{DBLP:conf/cvpr/Zhou0LLZDLL25,DBLP:conf/cvpr/DaiY25}. Inter-class topology and adaptive representation alignment address structure across a changing stream \citep{DBLP:conf/cvpr/NiLTHY025,DBLP:conf/nips/ZhangXYS24}. Each method creates predictions that can be judged against a fixed rule on the next deployment batch.

An update's state can live in batch-normalization memory, an organized example buffer, or a tracked subspace \citep{DBLP:conf/eccv/KangKOK24,DBLP:conf/eccv/ZhuHMZMDW24,DBLP:conf/nips/DuanXLW25}. Expert blending chooses an input-dependent mixture \citep{DBLP:conf/icml/LeeYH24}. Diffusion-based transformations change corruptions or align an input with a synthetic domain \citep{DBLP:conf/eccv/OhLCJHY24,DBLP:conf/cvpr/GuoZDWGNSSH25}. Our retain-versus-undo experiment measures the later effect of keeping one such state change.

The cost of constructing a candidate also matters. Forward-only adaptation, backpropagation-free alignment, and activation-function adaptation alter the required computation \citep{DBLP:conf/icml/NiuMCWZ24,ma2026architecture,DBLP:journals/corr/abs-2603-26096}. Sparse updates and activation sparsity reduce work or memory \citep{DBLP:conf/nips/ChaKCGL25,DBLP:conf/cvpr/Ma0000ZWFC0025}. Cloud--edge distillation, cross-device merging, and zeroth-order prompt optimization offer other deployment choices \citep{DBLP:conf/iclr/ChenN0XST24,DBLP:journals/corr/abs-2607-02988,DBLP:journals/corr/abs-2603-04733}. Label cost and computation cost remain distinct when choosing a complete system.

\subsection{Vision--language prediction rules}
Simple adaptation rules and analyses of zero-shot generalization provide strong starting points for vision--language models \citep{DBLP:conf/nips/FarinaFIM024,DBLP:conf/cvpr/ZanellaA24}. Regional bootstrapping, attribute prompts, and hierarchical knowledge change what the model sees or asks of a class \citep{DBLP:conf/nips/00030G00X24,DBLP:conf/cvpr/ZhangXLPZ25,DBLP:conf/cvpr/ZhangZ0ZXZ25}. Semantic integration and dynamic prompts offer further candidate rules \citep{DBLP:conf/cvpr/HuangJJLKW25,DBLP:conf/iclr/XiaoYHCJHSWS25}. A prompt coreset carries selected context across the stream \citep{DBLP:conf/icml/ZhangMNH25}.

Candidates can also use spatial bias, both modalities, or fewer visual tokens \citep{DBLP:conf/icml/Huang0X0J25,DBLP:conf/iccv/MaharanaZKFG25,DBLP:conf/iccv/WangGWHL25}. Retrieval and label propagation use neighboring examples to change predictions \citep{DBLP:conf/iclr/LeeKKBS025,DBLP:conf/iccv/FanCYYQDYS25,DBLP:conf/iclr/LiSGJ025}. Prototypes and attributive--associative memory retain different summaries of past inputs \citep{huang2026prototype,zhanglearning}. Dual debiasing and dynamic logit adjustment modify the final scores \citep{citation-0,wu2026dynamic}. These designs enlarge the set of rules that a selector may compare.

Flatness-guided and robust prompt methods target prediction stability \citep{li2025flatness,zhu2026robustifying,kim2026ss}. Flatness-aware pretraining addresses prompt calibration \citep{DBLP:journals/corr/abs-2604-27715}. Bilateral information use and universal prompt learning explore different sources of test-time signal \citep{sun2026bilateral,DBLP:conf/cvpr/FuWLD26}. Test-time adaptation also supports zero-shot value functions \citep{ziakas2026vita}. Beyond single-label classification, cross-modal retrieval and multi-label prediction need different outputs \citep{DBLP:conf/iclr/Li0Z0XY25,DBLP:conf/iclr/WuY0CL25}. Our experiments use a shared class label so that one query directly compares all candidate rules.

\end{document}